%% file: main.tex
\documentclass[11pt] {article}
\usepackage[round]{natbib}
\usepackage{amssymb}
\usepackage{amsmath}
\usepackage{bm}
\usepackage{fullpage}
\usepackage{amsthm}
\usepackage{graphicx}
\usepackage{xcolor}
\usepackage{enumitem}
 \usepackage{algorithm} 
\usepackage{algpseudocode}
\usepackage{subcaption}
\usepackage{authblk}
\usepackage{comment}

\algnewcommand\algorithmicinput{\textbf{INPUT:}}
\algnewcommand\INPUT{\item[\algorithmicinput]}  
\algnewcommand\algorithmicoutput{\textbf{OUTPUT:}}
\algnewcommand\OUTPUT{\item[\algorithmicoutput]}
\usepackage{hyperref}[]
\hypersetup{
    colorlinks=true,
    linkcolor=blue,
    filecolor=magenta,      
    urlcolor=blue,
      citecolor=blue,
}
\usepackage{cleveref}
 \usepackage{xcolor}

 \DeclareMathOperator*{\argmin}{arg\,min}

\newcommand{\R}{r}
\renewcommand{\r}{\R}
\newcommand{\ri}{{\rho}}
\renewcommand{\l}{\ell}
\newcommand{\li}{{\eta}}
\newcommand{\m}{m}
\newcommand{\mi}{{\mu}}
\newcommand{\ki}{k}

\newcommand{\bvec}[1]{\bm{#1}}
\newcommand{\timesS}{{\times_y \mathcal P_{\mathcal S}}}

\DeclareMathOperator*{\rank}{Rank}
\DeclareMathOperator*{\range}{Range}
\DeclareMathOperator*{\Tr}{Tr}
\DeclareMathOperator*{\KDE}{KDE}
\DeclareMathOperator*{\VRS}{VRS}

 \DeclareMathOperator*{\Span}{Span}
 
\DeclareMathOperator*{\op}{op}

 \DeclareMathOperator*{\OP}{O_{\mathbb P}}
 \DeclareMathOperator*{\bigO}{O}
 \DeclareMathOperator*{\smallo}{o}

 \DeclareMathOperator{\var}{Var}
 \newcommand{\xiapp}{ \xi_{ (r_1,\ldots, r_d)}^*}
\providecommand{\keywords}[1]{\textbf{\textit{Keywords---}} #1}

 \newcommand{\p}{\mathbb P}

  \newcommand{\h}{{  \mathcal H } }
 \newcommand{\C}{{  \mathcal C} }

   \newcommand{\sigmak}{  \lambda^{\mathbb K}   }
   \newcommand{\phik}{\phi^{\mathbb K} }

    \newcommand{\BZ}{{\bm Z} }

    \newcommand{\K}{ \mathbb K  }
    \newcommand{\timesML}{ \times_x  \mathcal{P}_{\mathcal M} \times_y \mathcal{P}_{\mathcal S}  }
         
   \newcommand{\timesL}{    \times_y \mathcal P_\mathcal L  }
         \newcommand{\M}{ \mathcal M  }
     
   \newcommand{\timesMD}{{ \times_1 \mathcal P_  \mathcal M \times \ldots \times_d \mathcal P _ \mathcal M} }

   \newcommand{\lt}{{  {\bf L}_2 } }

    \newcommand{\E}{ { \mathbb{E} } }

\newtheorem{assumption}{Assumption}
\newtheorem{lemma}{Lemma}

\newtheorem{corollary}{Corollary}
\newtheorem{definition}{Definition}
\newtheorem{remark}{Remark}
\newtheorem{theorem}{Theorem}
  \newtheorem{example}{Example}

\title{Multivariate Density Estimation via Variance-Reduced Sketching}

\author[1]{Yifan Peng}
\author[2]{Yuehaw Khoo\thanks{Corresponding author: \texttt{yuehaw.khoo@uchicago.edu}}}
\author[3]{Daren Wang}

\affil[1]{Committee on Computational and Applied Mathematics, University of Chicago}
\affil[2]{Department of Statistics, University of Chicago}
\affil[3]{Department of Applied and Computational Mathematics and Statistics, University of Notre Dame}

\date{}  
\begin{document}
\maketitle

\begin{abstract}

Multivariate density estimation is of great interest in various scientific and engineering disciplines.  In this work, we introduce a new framework  called Variance-Reduced Sketching (VRS), specifically designed to estimate multivariate  density functions with a reduced curse of dimensionality.
Our VRS framework conceptualizes multivariate functions as infinite-size matrices/tensors, and facilitates a new sketching technique motivated by the numerical linear algebra  literature  to reduce the variance in density estimation problems. We demonstrate the robust numerical performance of VRS through a series of simulated experiments and real-world data applications. Notably, VRS shows remarkable improvement over   existing neural network density estimators and classical kernel methods in numerous distribution models. Additionally, we offer theoretical justifications for VRS to support  its ability to deliver   density estimation  with a reduced curse of dimensionality. 
 
\end{abstract}
\keywords{Density estimation, matrix and tensor sketching, range estimation,   curse of dimensionality reduction 
}

\section{Introduction}
\label{section:introduction}
\input{introduction}

\section{Density range estimation by sketching}
\label{section:general range estimation}

\input{range}

\section{Density estimation by sketching}
\label{section:tensor estimation}

\input{tensor}

 \section{Simulations and real data examples}
 \label{sec: numerical experiments}

\input{numerical}

  \bibliographystyle{plainnat}
\bibliography{citations}  
 
\appendix 
\input{appendix}

\end{document}

%% file: introduction.tex
Multivariate  density estimation has extensive applications across diverse fields, including  genetics \citep{alquicira2021nebulosa,tian2024local}, neuroscience \citep{grieves2023estimating,torao2021up}, ecology \citep{dudik2007maximum,torres2012can}, and epidemiology \citep{shi2021extended,yang2021vision}.  In this context, the goal is to estimate the underlying   density function based on a collection of  independently and identically distributed data. 
 Classical density estimation methods such as histograms (\cite{scott1979optimal} and \cite{scott1985averaged}) and kernel density estimators (\cite{parzen1962estimation} and \cite{davis2011remarks}), known for their numerical robustness and statistical stability in lower-dimensional settings, often
suffer from the curse of dimensionality even in moderately higher-dimensional spaces. Mixture models (\cite{dempster1977maximum} and \cite{escobar1995bayesian}) offer a potential solution for higher-dimensional problems, but these methods lack the flexibility to extend to nonparametric settings. Additionally, adaptive methods (\cite{liu2007sparse} and \cite{liu2023convergence}) have been developed to address higher-dimensional challenges.   Recently, deep generative modeling has emerged as a popular technique to approximate high-dimensional densities from given samples, especially in the context of images. This category includes generative adversarial networks (GAN) (\cite{goodfellow2020generative,liu2021density}), normalizing flows (\cite{dinh2014nice,rezende2015variational,dinh2016density}), and autoregressive models (\cite{germain2015made,uria2016neural,papamakarios2017masked,huang2018neural}).
Despite their remarkable performance, particularly with high-resolution images, statistical guarantees   for these neural network methods continue  to be a challenge.   In this paper, we aim to develop a new framework specifically designed to estimate  multivariate nonparametric density functions. Within this framework, we conceptualize multivariate density functions as matrices or tensors and extend matrix/tensor approximations algorithm to this nonparametric setting, in order to reduce the curse of dimensionality in higher dimensions.

\subsection{Variance-Reduced Sketching}

Suppose the  data $\{ \BZ_i\}_{i=1}^N \subset \mathbb R^d$  are independently sampled from an unknown density $A^* \in \lt(\mathbb R^d )$. 
 In this work,  we  develop  a new matrix/tensor-based sketching framework to estimate    $A^*$, which  we  refer to as  \textit{Variance-Reduced Sketching} (VRS). In VRS, we view density functions as order $d$ tensors in the infinite-dimensional function space 
 $$\lt(\mathbb R^d ) =\lt(\mathbb R )\otimes \cdots \otimes \lt(\mathbb R ). $$
In the setting of estimating an $\alpha$-times differentiable density function in $d$ dimensions, 
the VRS estimator   achieves error rates 
\begin{align} 
\label{eq:intro rates} 
\E ( \| A^* - \widehat A_{\VRS}\|_{\lt} )  = \bigO \bigg( \frac{\sqrt{\prod_{j=1}^d r_j} }{N^{ \alpha/(2\alpha+1)}}   +\xiapp\bigg). 
\end{align}
  Here, $\{ r_j\}_{j=1}^d$ are the user-specified ranks for the   estimator $\widehat A_{\VRS}$, which is an order $d$ tensor in the function space; and  $\xiapp$    represents   the best possible error one could achieve  using any  rank-\((r_1,\ldots, r_d)\) tensor to approximate    \(A^*\)   in $\lt(\mathbb R^d ) $ :
  \begin{align*}  \xiapp = \min_{    B \text{  has   ranks  }    (r_1,\ldots, r_d) }
  \{ \| A^* - B \|_{\lt  }  \}.
\end{align*}
The precise definition of $\xiapp$ can be found at \eqref{eq:definition of low rank bias main}.
In contrast,   the error rate of the classical kernel density estimator (KDE)  satisfies (see e.g.\cite{wasserman2006all})
\begin{align} 
\label{eq:intro KDE rates} 
\E ( \| A^* - \widehat A_{\KDE}\|_{\lt} )   = \bigO \bigg(\frac{1}{N^{ \alpha/(2\alpha+d)}} \bigg),
\end{align}

In \Cref{section:examples of low rank}, we show that functions from many well-known  models such as the additive model\footnote{$f:\mathbb R^d\to \mathbb R $ is additive if $f(z_1,\ldots, z_d) =f_1(z_1) +f_2(z_2)+ \ldots + f_d(z_d)$.} and the mean-field model\footnote{$p:\mathbb R^d\to \mathbb R $ is mean-field if  $p(z_1,\ldots, z_d)= p_1(z_1) \cdot p_2(z_2) \cdots p_d(z_d)$.}  are      low-rank tensors    with  bounded ranks   $(r_1,\ldots, r_d)$ and  approximation error $\xiapp=0$.
In such  settings, the error rate in \eqref{eq:intro rates} reduces to $ \bigO \big( \frac{1}{N^{2\alpha/(2\alpha+1)}} \big) $, which corresponds to the nonparametric error rate in one dimension, significantly better than \eqref{eq:intro KDE rates}.  

To complement our theoretical findings,  we present a simulation study in \Cref{fig:intro} that compares the numerical performance  of    deep neural network estimators, classical kernel density estimators, and VRS. Additional empirical studies are provided in \Cref{sec: numerical experiments}. Extensive numerical evidence suggests that VRS significantly outperforms various deep neural network estimators and KDE by a considerable margin.

 \begin{figure} [h]
 
  \centering\includegraphics[width=10cm,height=6cm]{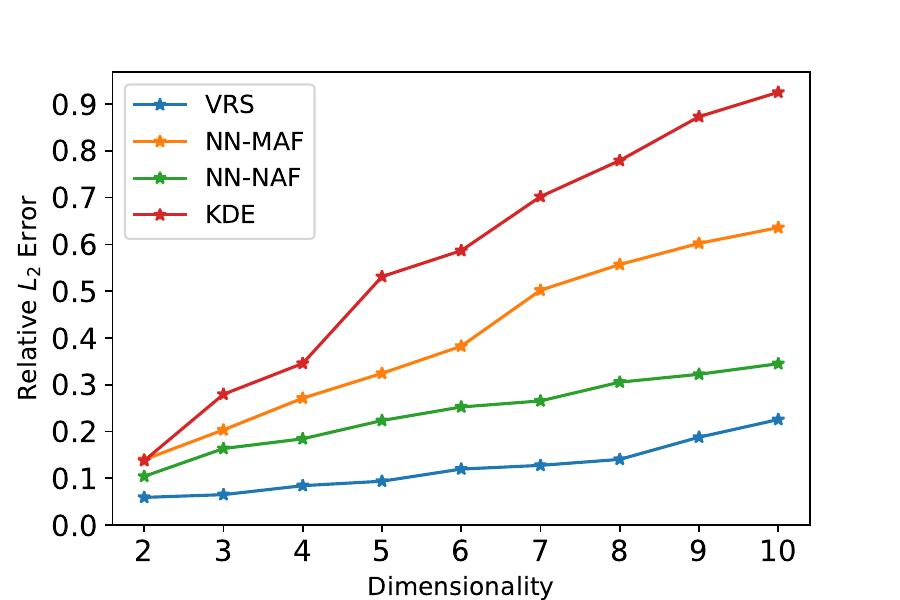}
 
     \caption{Density estimation   with  data sampled from the Ginzburg-Landau density  in \eqref{eq:GL density}. The $x$-axis represents dimensionality, varying from $2$ to $10$. NN-MAF corresponds to the Masked Autoregressive Flow  method \citep{papamakarios2017masked} and NN-NAF corresponds to the Neural Autoregressive Flows  method \citep{huang2018neural}  The performance of different estimators is evaluated using $\lt$-errors. Additional details are provided in  {\bf Simulation $\mathbf{III}$} of \Cref{sec: numerical experiments}.    }
    \label{fig:intro}
\end{figure}

The promising performance of Variance-Reduced Sketching (VRS) comes from its ability to reduce the problem of multivariate density estimation to the problem of estimating low-rank functional matrices/tensors using their moments.

In the finite-dimensional low-rank matrix estimation setting, estimating the matrix is equivalent to estimating its range, which is determined by its singular vectors. This intuition extends to  the density estimation    setting. As a concrete example, suppose \( A^*(x,y) : \mathbb R^2 \to \mathbb R^+ \) is a two-dimensional density function. It can be viewed as an infinite-dimensional matrix in $\lt(\mathbb R )\otimes \lt(\mathbb R )$. If, in addition, \( A^* \) is low-rank, then estimating the matrix \( A^* \) is equivalent to estimating its function range. Here, the function range of \( A^* \) with respect to the variable \( x \) is defined as
\[
{\range}_x(A^*) = \bigg\{ \int A^*(x, y) g(y) \, \text{d}y : g(y) \in \lt(\mathbb{R}) \bigg\}.
\]

Note that the function range of \( A^*(x, y) \) forms a linear subspace of functions in the variable \( x\in \mathbb R  \). 
The moral   from  linear algebra indicates that, by multiplying a low-rank matrix with some carefully  chosen vectors, the range can be readily obtained.  To estimate the function range, as demonstrated   \Cref{fig:intro sketch}, in VRS  we integrate   \( A^* \) with carefully chosen functions, analogous to finite-dimensional matrix-vector multiplications.   Integrating \( A^* \) against low-frequency functions results in    moments of \( A^* \) that are  highly correlated to the signals. We refer to this technique as a variance-reduced sketch for matrices in function spaces. These moments, typically of low orders, are selected to ensure that the range of the density can be accurately recovered while maintaining low variance in range estimation. 

In this way, VRS transforms a multidimensional density estimation problem into the task of estimating one-dimensional moments,   achieving the univariate estimation error. Additionally, this moment sketching strategy requires only a single pass over the samples and dimensions, resulting in a computational cost that scales linearly with both the dimensionality \( d \) and the sample size \( N \).
The moment-based range estimation technique  can be readily generalized to density estimation in arbitrary dimensions.

 We emphasize that our VRS framework is fundamentally different from the randomized sketching algorithm for finite-dimensional matrices, as randomly chosen sketching does not address the curse of dimensionality in multivariable density estimation models.  Additionally, the statistical guarantees for VRS are derived from a computationally tractable algorithm.  In contrast, deep learning methods exhibit generalization errors that depend strongly on the architecture of the neural network estimators. Moreover, the statistical analysis of these generalization errors is not necessarily aligned with the optimization errors achieved by computationally tractable optimizers.

 \begin{figure}[h]
  
  \centering\includegraphics[width=1\linewidth]{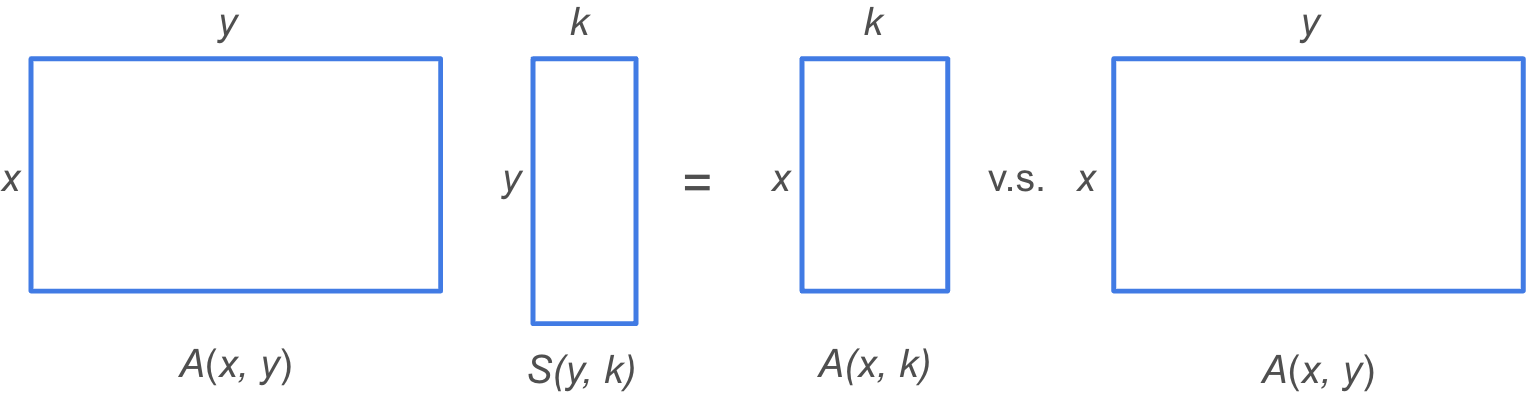}  
     \caption{ The sketched function $AS$ by VRS   retains the range in the variable $x$ of  $A (x, y)$. The complexity of estimating  the range   of $A$ using $AS$ is much lower than  the complexity of directly estimating $A$. }
    \label{fig:intro sketch}
\end{figure}

 \subsection{Related literature}
\par 
Matrix approximation algorithms, such as singular value decomposition and QR decomposition, play a crucial role in computational mathematics and statistics. A notable advancement in this area is the emergence of randomized low-rank approximation algorithms. These algorithms excel in reducing time and space complexity substantially without sacrificing too much numerical accuracy.  Seminal contributions to this area are outlined in works such as \cite{liberty2007randomized} and \cite{halko2011finding}.  Additionally, review papers like \cite{woodruff2014sketching}, \cite{drineas2016randnla}, \cite{martinsson2019randomized}, and \cite{martinsson2020randomized}  have provided  comprehensive summaries of these randomized approaches, along with their theoretical stability guarantees. Randomized low-rank approximation algorithms  typically  start by estimating  the range of a large low-rank matrix $A\in\mathbb{R}^{n\times n}$ by forming a reduced-size sketch. This is achieved  by right multiplying $A$ with a random matrix  $S\in\mathbb{R}^{n\times k}$, where $k\ll n$. The random matrix $S$ is selected to ensure that the range of $AS$ remains a close  approximation of the range of $A$, even when the column size of $AS$ is significantly reduced from $A$.  As such, the random matrix $S$ is referred to as randomized linear embedding or sketching  matrix by \cite{tropp2017practical} and \cite{nakatsukasa2021fast}. The sketching  approach reduces the cost in singular value decomposition from $\bigO(n^3)$ to $\bigO(n^2)$, where $\bigO(n^2)$ represents the complexity of matrix multiplication.

Recently, a series of studies have extended the matrix sketching technique to range estimation for high-order tensor structures, such as the Tucker structure  (\cite{che2019randomized,sun2020low,minster2020randomized}) and  the tensor train structure 
(\cite{al2023randomized,kressner2023streaming,shi2023parallel}). These studies developed specialized structures for  sketching to  reduce computational complexity while maintaining high levels of numerical accuracy in handling high-order tensors.

Previous work has also explored  randomized  sketching techniques in specific  estimation problems. For instance, \cite{mahoney2011randomized} and \cite{raskutti2016statistical} utilized  randomized  sketching to solve unconstrained least squares problems. \cite{williams2000using} , \cite{rahimi2007random}, \cite{kumar2012sampling} and \cite{tropp2017fixed} improved the Nystr{\"o}m method with randomized sketching techniques. Similarly, \cite{alaoui2015fast}, \cite{wang2017sketched}, and \cite{yang2017randomized} applied  randomized  sketching to kernel matrices in kernel ridge regression to reduce computational complexity. While these studies mark significant progress in the literature, they usually require extensive sampling or evaluation of the estimated function to employing the randomized sketching technique, in order to maintain acceptable accuracy. This step would be significantly expensive for the higher-dimensional setting.  Notably, \cite{hur2023generative} and subsequent studies \cite{tang2022generative,ren2023high,peng2023generative,chen2023combining,tang2023solving} aaddressed these issues by incorporating the variance of the data generation process into the design of sketching techniques for high-dimensional tensor estimation.  This sketching technique allows for the direct estimation of the  range of a tensor with reduced sample complexity, rather than directly estimating the full tensor.

\subsection{Organization}

Our manuscript   is organized as follows. 
In \Cref{section:general range estimation}, we  detail the procedure for implementing the range estimation through sketching and study  the corresponding error analysis. In \Cref{section:tensor estimation},   we extend our study  to density function estimation based  on the range  estimators developed in \Cref{section:general range estimation}, and provide the corresponding statistical error analysis  with a reduced curse of dimensionality.   In  \Cref{sec: numerical experiments}, we present  comprehensive numerical results to demonstrate the superior numerical performance of  our method.   Finally,  \Cref{sec:conclusion} summarizes our conclusions. 

\subsection{Notations}\label{sec:notation}

 Let  $\mathbb N =\{1,2,3,\ldots\}$ and $\mathbb R$   denote the set of all the real numbers.
  We denote  $X_n=\OP(a_n)$  if for any given $\epsilon>0$, there exists a  
$K_\epsilon>0$ such that 
$  \limsup_{n\to \infty} \p(|X_n/a_n| \ge K_\epsilon ) < \epsilon.$   
For  real numbers $\{ a_n\}_{n=1} ^\infty $, $\{ b_n\}_{n=1}^\infty $ and $\{ c_n\}_{n=1} ^\infty$, we denote 
$a_n=\bigO(b_n)$ if 
$ \lim_{n\to \infty }a_n/b_n < \infty $ and $a_n=\smallo(c_n) $ if 
$ \lim_{n\to \infty }a_n/c_n  =0$.
Let $[0,1]$ denote the unit interval in $\mathbb R$. For   $d \in \mathbb N $,  denote 
  the $d$-dimensional unit cube as $ [0,1]^d   $.

  Let $\{ f_i\}_{i=1}^n$ be a collection of elements in the Hilbert space $\h$. Then 
\begin{align}\label{eq:definition span} \Span\{ f_i\}_{i=1}^n =\{ b_1 f _1+\cdots+b_nf_n: \{b_i\}_{i=1}^n \subset \mathbb R\}. 
\end{align}
Note that $\Span\{  f_i\}_{i=1}^n $ is a  linear subspace of $\h$. 
 
For a generic measurable set $\Omega \subset \mathbb R^d$, denote
$  \lt(\Omega) = \big\{ f:\Omega \to \mathbb R : \|f\|^2_{\lt(\Omega)} = \int_{\Omega} f^2(z)\mathrm{d}z < \infty \big\}.$ 
  For any $f, g\in \lt(\Omega) $, let the inner product between $ f$ and $g$   be 
$$ \langle f, g\rangle  =\int_\Omega   f(z)g(z) \mathrm{d}z. $$ 
We say that $\{\phi_k\}_{k=1}^\infty$ is a collection of orthonormal   functions in $\lt(\Omega)$ if 
\begin{align*}
    \langle \phi_k , \phi_l\rangle = \begin{cases}
1 &\text{if } k=l;
\\
0 &\text{if } k\not = l.
\end{cases}
\end{align*}  
 We say that $\{\phi_k\}_{k=1}^\infty$ is a complete basis  of $\lt(\Omega)$ if $\{\phi_k\}_{k=1}^\infty$ is orthonormal and $\Span \{\phi_k\}_{k=1}^\infty =\lt(\Omega)$.
Let $ Z \in  \Omega $  and  $\delta_{Z }$ be the  indicator function at the point $Z$.  Define 
$$   \langle    \delta_{Z }  ,f  \rangle =  f(Z).    $$ 

In what follows,  we   briefly  introduce the notation   for Sobolev spaces. Let $\Omega \subset \mathbb R^d$ be any measurable set.  For multi-index $\beta =(\beta _1,\ldots,\beta  _d) \in \mathbb N^{d}  $, and 
$ f(z_1,\ldots, z_d) : \Omega \to \mathbb R$, define the $\beta $-derivative of $f$ as 
 $D^{\beta   }f = \partial ^{ \beta_1} _1 \cdots \partial ^{ \beta_d} _d    f    .$ 
Then 
$$ W^{\alpha  }_2(\Omega) :=\{ f\in \lt(\Omega): \  D^\beta  f \in \lt(\Omega)   \text{ for all } |\beta  | \le \alpha  \} ,$$
where $|\beta  |= \beta  _1+\cdots+ \beta _d$ and $\alpha \in \mathbb N$ represents the total order of derivatives. The Sobolev norm of $f \in W^{\alpha  }_2 $ is 
$ \|f\|_{W^{\alpha    }_2(\Omega) }^2 = \sum_{0\le |\beta| \le \alpha } \| D^\beta f\|_{\lt(\Omega)}^2  . $ 

\subsection{Background: linear algebra in function spaces}\label{sec:linear algebra in function spaces}

Linear algebra in tensorized function spaces   is the key component  to develop our Variance-Reduced Sketching (VRS) framework for nonparametric density estimation. 
 
We start with   a convenient notation for partial  integration of two functions. 
For $j=1,2,3$  let $\Omega_j \subset \mathbb R^{d_j}$  be any regular domain.  
Let $A(x, y) \in  \lt(\Omega_1\times \Omega_2 ) $ and $Q(y,z) \in  \lt(\Omega_2\times \Omega_3 ) $ be two functions. Define
\begin{align}\label{eq:integration as matrix multiplication main}  ( A\times _y Q ) (x, z)  = \int_{\Omega_2} A(x, y) Q(y,z) \mathrm{d}y.  
\end{align}
For $j\in \{1,\ldots,d \}$, let $\Omega_j  \subset \mathbb R^{d_j}$ be a measurable subset. Let  $ \mathcal S_j   $ be   arbitrary  linear subspaces of $ \lt(\Omega_j)$ where functions in $\mathcal S_j $ correspond to  the variable $z_j$. Define
$$ \mathcal S_{1}\otimes \cdots \otimes  \mathcal S_{d} = \Span\{g_1(z_1) \cdots g_d(z_d) : g_1(z_1) \in   \mathcal S_{1}, \ldots, g_d(z_d) \in   \mathcal S_{d}\}. $$
So $ \mathcal S_{1}\otimes \cdots \otimes  \mathcal S_{d} \subset \lt( \Omega_1) \otimes \cdots \otimes  \lt( \Omega_d) = \lt(\Omega_1 \times \cdots \times \Omega_d)$.
Let $ \{ \phi_{\mi} ^{(j)} (z_j)  \}_{\mi=1}^ \infty$ be a complete basis in $\lt(\Omega_j)$, and 
   define  
\begin{align*}  \mathcal M_j =   \Span \{ \phi_{\mi}^{(j)} (z_j)  \}_{\mi=1}^ m  . 
\end{align*}
So $\mathcal M _j\subset \lt(  \Omega_j) $  and  $ \mathcal M_{1} \otimes \cdots \otimes \mathcal M_d \subset \lt( \Omega_1) \otimes \cdots \otimes  \lt( \Omega_d)  $. Denote 
$$ \mathcal P_{\mathcal M_j} (z_j, z_j') = \sum_{\mi=1}^m \phi^{(j)} _\mi(z_j) \phi^{(j)} _\mi(z_j')  . $$
For any function $A \in \lt(\Omega_1 \times \cdots \times \Omega_d) $,   we define 
$A  \times_j \mathcal P _{\mathcal M_j} (z_1, \ldots, z_d)$ using  \eqref{eq:integration as matrix multiplication main} as 
$$ A  \times_1 \mathcal P _{\mathcal M_1} (z_1, \ldots, z_d) = \int_{\mathcal O} A(z_1, \ldots, z_{j-1}, z_j',z_{j+1},\ldots, z_d ) \mathcal P_{\mathcal M_j} (z_j', z_j) \mathrm{d} z'_j.$$
Therefore 
$A  \times_1 \mathcal P _{\mathcal M_1} \times \ldots \times_d  \mathcal P_{\mathcal M_d}    $  can be defined inductively.  A direct calculation shows that   
$$  A {\times_1 \mathcal P _{\mathcal M_1} \times \ldots \times_d  \mathcal P_{\mathcal M_d}}(z_1, \ldots, z_d) = \sum_{\mi_1=1}^m \cdots \sum_{\mi_d=1}^m  \langle A, \phi_{\mi_1} \otimes \cdots \otimes \phi_{\mi_d}\rangle  \phi_{\mi_1}(z_1) \cdots   \phi_{\mi_d} (z_d) . $$
Therefore   $  A {\times_1 \mathcal P _{\mathcal M_1} \times \ldots \times_d  \mathcal P_{\mathcal M_d}} $ is a function in the space $\mathcal M_{1} \otimes \cdots \otimes \mathcal M_d  $.
Given data $\{ \BZ_i\}_{i=1}^N \subset \Omega_1 \times \cdots \times \Omega_d $,  denote the corresponding  empirical measure as 
$  \widehat A    = \frac{1}{N} \sum_{i=1}^N \delta _{\BZ_i}  ,$ 
where $ \delta _{\BZ_i }$ is the indicator function   at the  point $\BZ_i$.
We can also project the empirical measure $\widehat A$ onto $\mathcal M_{1} \otimes \cdots \otimes \mathcal M_d  $ as 
\begin{align}\nonumber 
    \widehat A  {\times_1 \mathcal P _{\mathcal M_1} \times \ldots \times_d  \mathcal P_{\mathcal M_d}}  (z_1, \ldots, z_d) =  \sum_{\mi_1=1}^m \cdots \sum_{\mi_d=1}^m  \langle  \widehat A, \phi ^{(1)}_{\mi_1} \otimes \cdots \otimes \phi_{\mi_d}^{(d)}\rangle  \phi_{\mi_1}^{(1)} (z_1) \cdots   \phi_{\mi_d}^{(d)} (z_d) 
   \\ \label{eq:projection of measure}
   =  \sum_{\mi_1=1}^m \cdots \sum_{\mi_d=1}^m \bigg \{ \frac{1}{N } \sum_{i=1}^N    \phi_{\mi_1} ^{(1)} \otimes \cdots \otimes \phi_{\mi_d}^{(d)} (\BZ_i)  \bigg\}   \phi_{\mi_1}^{(1)}(z_1) \cdots   \phi_{\mi_d}^{(d)} (z_d)  .
\end{align}

%% file: range.tex
 In this section, we introduce a new technique called sketching, which allows us to estimate the range of a  density function  based on discrete samples with a reduced curse of dimensionality. The definition of function range can be found   in \eqref{eq:range definition}.  As demonstrated in \Cref{section:tensor estimation}, range estimation    plays a crucial role in the  VRS density estimation framework.

Throughout this section, we set $d_1, d_2$ to be two arbitrary positive integers.  Let $\Omega_1 \subset \mathbb R^{d_1}$ and $\Omega_2 \subset \mathbb R^{d_2}$ be two measurable sets. Let $A^*(x, y):\Omega_1\times \Omega_2 \to \mathbb R$ be the unknown density function.

\subsection{Ranges  and ranks of functions}\label{sec:function as matrix}
  Let $B(x, y) \in \lt(\Omega_1\times \Omega_2)$. We define the range space of $B(x, y)$ in the variable $x$ as 
\begin{align} \label{eq:range definition}
{ \range}_x (B) = \bigg \{ f(x) : f(x) = \int_{\Omega_2} B(x, y) g (y) \mathrm{d} y \text{ for any } g(y) \in \lt(\Omega_2) \bigg \}.
\end{align}
In VRS, we   first estimate the range of the unknown density $A^*$ with respect to each variable, and then incorporate the range estimators of all variables to estimate the density function $A^*$.

We define the dimensionality of $B$ with respect to variables $(x, y)  \in \Omega_1 \times \Omega_2$ by using the singular value decomposition of $B$ in the function space $\lt(\Omega_1 \times \Omega_2)$.
\begin{theorem}
\label{theorem:svd Hilbert main}
[Singular value decomposition in function space] 
Let $B(x, y):\Omega_1 \times \Omega_2 \to \mathbb R$ be any function such that $\|B\|_{\lt(\Omega_1 \times \Omega_2)} < \infty$. There exists a collection of strictly positive singular values $\{ \sigma_\ri (B) \}_{\ri=1}^{\R} \subset \mathbb R^+$ in decreasing  order, and two collections of orthonormal basis functions $\{ \Phi_\ri(x) \}_{\ri=1}^{\R} \subset \lt(\Omega_1)$ and $\{ \Psi_\ri(y) \}_{\ri=1}^{\R} \subset \lt(\Omega_2)$, where $\R \in \mathbb N \cup \{+\infty\}$, such that 
\begin{align}\label{eq:svd Hilbert main}
B(x, y) = \sum_{\ri=1}^{\R} \sigma_{\ri}(B) \, \Phi_\ri(x) \, \Psi_\ri(y).
\end{align}
In this case, we say that the rank  of $B(x, y)$ is $\R$.
\end{theorem} 
Suppose \eqref{eq:svd Hilbert main} holds. An equivalent definition of ${\range}_x(B)$ is 
\[
{\range}_x(B) = \Span\{\Phi_\ri(x) \}_{\ri=1}^{\R}.
\]
Therefore, $\dim({\range}_x(B))$ equals the rank of $B$. This illustrates the well-known fact that the dimensionality of the column space of a matrix agrees with its rank.

\subsection{The sketching algorithm }
Let $\{ \BZ_i \}_{i=1}^N \subset \Omega_1 \times \Omega_2 $ be the observed data sampled from $A^* (x, y)$.  
Based on the observed data, our goal in this section is to estimate ${ \range} _x (A^*)$  defined in  \eqref{eq:range definition}.
 We denote  by   $\widehat A  $     the   empirical measure  formed by $\{ \BZ_i\}_{i=1}^N \subset\mathbb R^d$ as 
\begin{align} \label{eq:density estimator example} \widehat A  (x, y)   = \frac{1}{N}\sum_{i=1}^N \delta _{  \BZ_i  }  (x, y) .
\end{align} 

 We will work with two  user-specified   linear subspaces $\mathcal M $ and $\mathcal S  $,
where $\mathcal M \subset \lt(\Omega_1)$  serves  as an estimation subspace, and 
$\mathcal S  \subset \lt (\Omega_2)$  serves    as a sketching subspace. Suppose that 
\begin{align}\label{eq:choice of linear spaces in sketching}\M = \Span\{   v_{\mi}(x) \}_{\mi=1}^{\dim(\mathcal M)} \quad \text{and}\quad  \mathcal S  = \Span\{   w_{\li}(y) \}_{\li=1}^{\dim(\mathcal S)}  , 
\end{align}
where $ \{   v_{\mi}(x) \}_{\mi=1}^{\dim(\mathcal M)}$ are   orthogonal functions  in $ \lt(\Omega_1) $ and $ \{   w_{\li}(y) \}_{\li=1}^{\dim(\mathcal S)} $   are   orthogonal functions  in $  \lt(\Omega_2)$. Specific choices of these subspaces can be found at 
\eqref{eq:projection basis for sketching 1} and \eqref{eq:projection basis for sketching 2}.
Our procedure is composed of the following three stages.\par

\begin{itemize}[leftmargin=0.4cm]
    
\item[$\bullet$]  
 
Sketching stage.  
We apply the projection operator $\mathcal P_\mathcal S$ to $\widehat A$ by computing   
\begin{equation}\label{eq:sketching stage}
\bigg\{ \int_{\Omega_2} \widehat{A}(x,y)w_{\li}(y)\mathrm{d} y \bigg \} _{\li=1} ^{ \dim(\mathcal S)} .
\end{equation}
Note  that for each $\li=1,\ldots, \dim(\mathcal S )$, $\int_{\Omega_2} \widehat{A}(x,y)w_{\li}(y)\mathrm{d} y$ is a function solely that  depends on $x$. 
 This stage   aims at reducing the curse of dimensionality associated to variable  $y$. 
     
\item[$\bullet$]  
Estimation stage.  We estimate  the   functions   
$\big\{ \int_{\Omega_2} \widehat{A}(x,y)w_{\li}(y)\mathrm{d} y  \big\}_{\li=1}^{\text{dim}(\mathcal S)}$ by utilizing the estimation space $\M$. Specifically, for each $\li=1,\ldots,\dim(\mathcal S)  $, we approximate $\int_{\Omega_2} \widehat{A}(x,y)w_{\li}(y)\mathrm{d} y  $ by 
 \begin{equation} \label{eq:column estimation}
     \widetilde  f_\li (x)=   \argmin_{f \in \M }  \bigg \|\int_{\Omega_2} \widehat{A}(x,y)w_{\li}(y)\mathrm{d} y   - f(x)\bigg\|^2_{\lt(\Omega_1)}.
    \end{equation}
\item[$\bullet$]  Orthogonalization stage. Let   
\begin{align}\label{eq:sketch and estimation} \widetilde A(x, y)= \sum_{\li=1}^{\dim(\mathcal S)} \widetilde f_\li(x)w_\li(y). 
\end{align}
Compute the leading singular functions in the variable $x$ of $\widetilde A(x, y)$  to estimate the   $\range_x(A^*)$. 

\end{itemize}

We formally summarize our procedure in \Cref{algorithm: range estimation}.

\begin{algorithm}[H]
\begin{algorithmic}[1]
	\INPUT Empirical measure $\widehat   A    = \frac{1}{N}\sum_{i=1}^N \delta _{  \BZ_i  }    $, parameter $\R \in \mathbb Z^+$, linear subspaces $\M = \Span\{   v_{\mi}(x) \}_{\mi=1}^{\dim(\mathcal M)}$  and   $\mathcal S  = \Span\{   w_{\li}(y) \}_{\li=1}^{\dim(\mathcal S)}  $.
 	 
 	\State Compute
  $   \big \{ \int_{\Omega_2} \widehat{A}(x,y)w_{\li}(y)\mathrm{d} y \big  \} _{\li=1} ^{ \dim(\mathcal S)},  $ 
  the projection of  $\widehat{A} $ onto  $\{ w_{\li}(y) \}_{\li=1}^{\dim(\mathcal S)}$. 
 	 
 	\State Compute the estimated functions $\{ \widetilde f_\li (x)\}_{\li=1}^{\dim(\mathcal S) }$ in $\mathcal M$ by \eqref{eq:column estimation}.
  
 	\State  
Compute  the leading $\R$ singular functions in the variable $x$ of   $ \widetilde A(x, y)= \sum_{\li=1}^{\dim(\mathcal S)} \widetilde f_\li(x)w_\li(y)  $ and denote them as $\{\widehat \Phi_\ri(x)\}_{\ri=1}^{\R}$.

    \OUTPUT  $\widehat {\mathcal P}_x (x, x')  = \sum_{\ri=1}^\R  \widehat \Phi_\ri   (x) \widehat \Phi_\ri(x')  $.
 \caption{Density range Estimation via Variance-Reduced Sketching} 
\label{algorithm: range estimation}
\end{algorithmic}
\end{algorithm}

   An explicit expression of $\widetilde A(x, y)$  in \eqref{eq:sketch and estimation} can be obtained   from $\widehat A(x, y)$. More precisely,
in the sketching stage,  \eqref{eq:sketching stage} is equivalent to computing 
$\widehat A\timesS $. Indeed, by \eqref{eq:projection of measure},
$$\widehat A\timesS = \sum_{\li=1}^{\dim(\mathcal S)} \bigg( \int_{\Omega_2} \widehat{A}(x,y)w_{\li}(y)\mathrm{d} y \bigg) w_{\li}(y) .$$ 
In the estimation stage, using orthogonality of $ \{   v_{\mi}(x) \}_{\mi=1}^{\dim(\mathcal M)} $,    \eqref{eq:column estimation}    can  be explicitly expressed as 
$$ \widehat f_\li(x) = \sum_{\mi=1}^{\dim(\M)}  \langle \widehat A, v_\mi \otimes w_\li\rangle  v_\mi(x)  ,$$
where by definition,  $ \langle \widehat A, v_\mi \otimes w_\li\rangle  =\int_{\Omega_2}\int_{\Omega_1}  \widehat{A} (x, y)  v _{\mi}(x) w_{\li}(y)  \mathrm{d} x \mathrm{d} y    $. Therefore, $\widetilde A(x, y)$ in \eqref{eq:sketch and estimation} can be rewritten as 
  \begin{align} \widetilde A(x, y)=  \sum_{\mi=1}^{\dim(\M)} \sum_{\li=1}^{\dim(\mathcal S)} \langle \widehat A, v_\mi \otimes w_\li\rangle   v_\mi(x) w_\li(y). 
   \end{align} 
By  \eqref{eq:projection of measure},
$  \widehat A\timesML  $
 has the exact same expression as $\widetilde A $. Therefore, we establish  the identification  
 \begin{align}\label{eq:closed form} \widetilde A =\widehat A\timesML.   
\end{align}

We discuss the specific choices of $\mathcal M$ and $\mathcal S$.
 Let  $\mathcal O$ be a  measurable subset of  $\mathbb R $. Suppose that 
\begin{align}\label{eq:domain assumption of sketching} 
\Omega_1 \subset \mathcal  O^{d_1} \quad \text{and} \quad  \Omega_2 \subset \mathcal  O^{d_2}
 \end{align}
 In   density estimation, it is sufficient to assume 
 ${\mathcal O}=[0,1]$.   Indeed, if the density    function $A^*$  has   compact support, through necessary scaling,   we can assume the support $\Omega_1\times \Omega_2$ is a subset of $\mathcal O^{d_1} \times \mathcal O^{d_2} =[0,1]^{d_1+d_2} $.
 
Let 
$  \{ \phi_{k}\}_{k=1}^\infty  \subset \lt({\mathcal O}) 
$  be a collection of    orthonormal one-dimensional basis functions generated     from either  the reproducing kernel Hilbert spaces or the Legendre polynomial  system.  A detailed discussion of  about the reproducing kernel Hilbert spaces or the Legendre polynomial  system  can be found in \Cref{section: justify approximation theory}.    
       Let  $x_j\in \mathbb R $ for  $j  \in \{ 1,\ldots,d_1\}$  and 
       let  $y_k\in \mathbb R $ for  $k  \in \{ 1,\ldots,d_2\}$.
       For $\m, \l\in \mathbb N$, let 
 \begin{align} \label{eq:projection basis for sketching 1}
 \mathcal M  &=\text{span} \bigg\{\phi_{\mi_1 }(x_1 )\phi_{\mi_2 }(x_2 )  \cdots \phi_{\mi_{d_1} }(x_{d_1} )\bigg \}_{{\mi_1,\ldots, \mi_{d_1}=1} }^{\m } \quad \text{and} \quad 
\\\label{eq:projection basis for sketching 2}
 \mathcal S  &=\text{span} \bigg\{\phi_{\li_1 }(y_1 )\phi_{\li_2 }(y_2 )  \cdots \phi_{\li_{d_2} }(y_{d_2} )\bigg \}_{{\li_1,\ldots, \li_{d_2}=1} }^{\l}.
\end{align} 
Therefore $\dim(\mathcal M) = m^{d_1}$ and $\dim(\mathcal S) = \ell^{d_2}$. 
In     the following result,   we show that  the  range estimator in \Cref{algorithm: range estimation} is consistent.

\begin{theorem}\label{lemma:general projection deviation main} 
  Suppose that the rank of $A^*$ is $\R <\infty$, and let $\sigma_\R^* = \sigma_\R(A^*(x, y))$ denote the minimal nonzero singular value\footnote{The definition of the singular values of a two-variable function can be found in \eqref{eq:svd Hilbert main}.}   of $ A^*(x, y)$.
In \Cref{algorithm: range estimation}, let $\mathcal M$ and $\mathcal S$ be defined as in \eqref{eq:projection basis for sketching 1} and \eqref{eq:projection basis for sketching 2}.
For a sufficiently large constant $C $, suppose that
\begin{align}\label{eq:eigenvalue low bound} \sigma_\R ^*   >C   \max \bigg \{   \l^{-\alpha}  , \m^{-\alpha }  ,\sqrt{ \frac{ ( \m^{d_1} + \l^{d_2}) \log(N) }{N}}     \bigg\}  .\end{align}
Let $  {\mathcal P}_x^* $ be the projection function onto   $\range_x(A^*)$, and 
let $\widehat {\mathcal P}_x $  be the output of \Cref{algorithm: range estimation}. Then 
\begin{align}\label{eq:range estimation error bound} \| \widehat {\mathcal P} _x - \mathcal P^*_x \|_{\lt(\Omega_1 \times \Omega_1) }^2  =\OP  \bigg\{  \frac{r}{ (\sigma_{\R} ^*) ^2 }  \bigg( \m^{-2\alpha }  +  \frac{ (   \m^{d_1}+    \l^{d_2}) \log(N) }{N}       \bigg)  \bigg\}   . 
\end{align}
\end{theorem}

The proof of \Cref{lemma:general projection deviation main}  can be found in \Cref{section:proof of sketching}.  
It follows that if $\m\asymp N^{1/(2\alpha+d_1)}$, $\l =C_L ( \sigma_\R ^*) ^{   -1/\alpha} $ for a  sufficiently large constant $C_L$,  and sample size 
 $N\ge C_{\sigma}\max\{ ( \sigma_\R ^*)   ^{-2 -d_1/\alpha }  ,  ( \sigma_\R ^*)  ^{-2 -d_2/\alpha }  \}   $ 
for  a sufficiently large constant $C_\sigma$,   then   \Cref{lemma:general projection deviation main} implies that   up to a logarithm factor,
 \begin{align}\label{eq:range estimation error bound 2} \| \widehat {\mathcal P} _x - \mathcal P^*_x \|_{\lt(\Omega_1 \times \Omega_1) } ^2 = \OP   \bigg\{ \frac{ \R }{ N^{ 2\alpha /(2\alpha+d_1 ) } ( \sigma_{\R}^*) ^{ 2}  } +   \frac{ 
 \R }{N (\sigma_{\R}^*)^{ 2 +d_2/\alpha } }  \bigg\}  . 
 \end{align}
In \Cref{section:examples of low rank}, we demonstrate  that both additive  functions and the mean-field functions    are       finite-rank functions    with  minimal singular values being strictly positive.   In these settings,   $\sigma_{\R}^*$ and $\R$ are both absolute constants, and      up to a logarithm factor, \eqref{eq:range estimation error bound 2} further reduces to,
  $$\| \widehat {\mathcal P} _x - \mathcal P^*_x \|_{\lt(\Omega_1 \times \Omega_1) }^2 = \OP   \bigg\{ \frac{1 }{ N^{ 2\alpha /(2\alpha+d_1 )} }  \bigg\}  , $$
  which matches the  optimal nonparametric density estimation rate in   $ d_1$ dimensions.  This indicates  that our sketching method is able to estimate  $\range_x ( A^*)  $ without the curse of dimensionality introduced by  the variable $y\in \mathbb R^{d_2}$. Utilizing  the sketching estimators,  we can further estimate the unknown function $A^* $ with a reduced curse of 
 dimensionality as detailed in \Cref{section:tensor estimation}.

%% file: tensor.tex
  In this section, we propose a new   multi-dimensional density   estimator by utilizing the range estimator outlined in \Cref{algorithm: range estimation}.  Let  $\mathcal O \subset \mathbb R $ be a  measurable  set   and  $A^*(z_1, \ldots, z_d) : {\mathcal O}^d\to \mathbb R $ be the unknown population function.  In   density estimation, it is sufficient to assume 
 ${\mathcal O}=[0,1]$.   Indeed, if  $A^*$  has   compact support, through necessary scaling,   we can assume the support is a subset of $\mathcal O^d =[0,1]^{d} $.

     In \Cref{algorithm:analysis version of general svd version 1}, we formally  summarize  our tensor-based estimator of $A^*$.
The selection of tuning parameters   in \Cref{algorithm:analysis version of general svd version 1} will be discussed  in \Cref{subsection:tuning parameters selection}.

 \begin{algorithm}[h!]
\begin{algorithmic} 
	\INPUT Data $\{\BZ_i\}_{i=1}^N$, parameters $\{ {\R}_j \}_{j=1}^d\subset  \mathbb Z^+$, estimation subspaces $\{ \mathcal M_j\}_{j=1}^d$ as in \eqref{eq:projection basis for tensor estimation 1 main} and sketching subspaces $ \{ \mathcal S_j\} _{j=1}^d $ as in \eqref{eq:projection basis for tensor estimation 2 main}. 
 \State Set empirical measures  $\widehat   A  =\frac{1}{N/2 }\sum_{i=1}^{N/2} \delta_{\BZ_i}   $, and   $\widehat   A'  =\frac{1}{N/2 }\sum_{i=N/2 +1}^{N } \delta_{\BZ_i}   $ 	 
 	 \State [{\bf  Subspaces estimation via Sketching}]
 	\For {$ j  \in \{ 1,  \ldots,   d  \}$} 
    \State Set   $y=(z_1,\ldots, z_{j-1}, z_{j+1},\ldots, z_d)$.
 
     \State  Compute 
       the leading $\R_j$ singular functions in the variable $z_j$ of  the function   $$  \bigg(    \widehat A\times_{ j} \mathcal P_{\mathcal M_j}\times _y \mathcal P_{\mathcal S_j} \bigg)  (z_j, y)  ,$$  
       \State and  denote them as $\{\widehat \Phi_{\ri }(z_j)\}_{\ri=1}^{\R_j}$.
\State   Set the projection function  $\mathcal P^{\Phi}_j (z_j, z_j')  = \sum_{\ri =1}^{\R_j }\widehat \Phi_{\ri }(z_j) \widehat \Phi_{\ri }(z_j')  $.
    
 	\EndFor 
  \vskip 0.5cm
  \State [{\bf  Sketching using  estimated subspaces}]
  
    \For {$ j  \in \{ 1,  \ldots,   d  \}$}  
 \State Set    
  $$ \widehat B_j = \widehat A\times_{ j}\mathcal P_{\mathcal M_j}\times _1   \mathcal P_1^{\Phi} \times \cdots \times_{j-1}  \mathcal P_{j-1} ^{\Phi}\times_{j+1}  \mathcal P_{j+1} ^{\Phi} \cdots \times_d  \mathcal P_{d} ^{\Phi} . $$
 
\State  Set  $y=(z_1,\ldots, z_{j-1}, z_{j+1},\ldots, z_d)$.  
\State Compute 
       the leading $\R_j$ singular functions in the variable 
         $z_j$ of     the function $    \widehat B_j(z_j, y), $   and \State denote them as $\{\widehat \Psi_{\ri }(z_j)\}_{\ri =1}^{\R_j}$.
  \State  Set the projection function  $\mathcal P^{\Psi}_j (z_j , z_j')  = \sum_{\ri =1}^{\R_j }\widehat \Psi_{\ri }(z_j) \widehat \Psi_{\ri }(z_j')  $.
   \EndFor
     \State  Compute the estimated   density function 
    $$   \widetilde  A (z_1,\ldots,z_d)
    =  \bigg(  \widehat A' \times_ { 1} \mathcal P^{\Psi}_{1}  \times \cdots  \times_ { d}  \mathcal P^{\Psi}_ {d}\bigg) (z_1,\ldots,z_d) .$$  
	\OUTPUT   The  estimated  density function $\widetilde A(z_1,\ldots,z_d).$

 \caption{Multivariable Density  Estimation via Variance-Reduced Sketching}
\label{algorithm:analysis version of general svd version 1}
\end{algorithmic}
\end{algorithm}

 We formally introduce the estimation and the sketching spaces for VRS. Let 
$  \{ \phi_{k}\}_{k=1}^\infty  \subset \lt({\mathcal O}) 
$  be a collection of  one-dimensional   orthonormal basis functions. 
  For $j  \in \{ 1,\ldots,d\}$, let $\m \in \mathbb N  ,  \l_j \in  \mathbb N$     and denote 
 \begin{align} \label{eq:projection basis for tensor estimation 1 main}
 \mathcal M_j &=\text{span} \bigg\{\phi_{\mi }(z_j )   \bigg \}_{{\mi=1} }^{\m } \quad \text{and} \quad 
\\\label{eq:projection basis for tensor estimation 2 main}
\mathcal S_j  &= \text{span}\bigg\{  \phi _{\li_1 }(z_1)\cdots  \phi _{\li _{j-1} }(z_{j-1}) \phi _{\li _{j+1} }(z_{j+1}) \cdots  \phi _{\li_{d}  }(z_{d})  \bigg\}_{\li _1, \ldots,\li _{j-1}, \li_{j+1}, \ldots, \li_{d}=1}^{\l_j}.
\end{align}

\begin{assumption} \label{assume: baised in projected space in operator norm main}
Let $\{ \phi_{\mi}   \}_{\mi=1}^\infty  $ be a collection of one-dimensional  basis functions  in $\lt(\mathcal O)$ generated from either   the reproducing kernel Hilbert spaces or the Legendre polynomials system.
\end{assumption} 
Reproducing kernel Hilbert spaces and Legendre polynomials are widely used numerical techniques in the machine learning literature.  Next, we summarize the regularity condition used to establish the consistency of VRS.

\begin{assumption} \label{assumption:regularity of density main}
    Suppose that  the data $\{ \BZ_i\}_{i=1}^N\subset \mathcal O  ^{d}$ are  independently sampled from the density $A^* : \mathcal O   ^{d} \to \mathbb R ^+ $. Suppose in addition that   $\|A^*\|_{W^\alpha_2 (\mathcal O ^d)} < \infty  $ with $\alpha\ge 1$  and that  $\|A^*\|_\infty <\infty$.  
\end{assumption} 

\Cref{assumption:regularity of density main} is widely used in the density estimation literature.   We  aim at   establishing the consistency of   VRS   even when  the underlying density function   is infinite-rank. This leads us to assume the following singular-gap condition for $A^*$.

\begin{assumption} \label{assume:spectral gap}
    For $j=1,\ldots, d$, denote  
$y_j= (z_1,\ldots, z_{j-1}, z_{j+1}, \ldots, z_d)$, and 
 $  \sigma_{j,\ri }^* = \sigma_{\ri} ( A^*(z_j, y_j ))$  for any  $ \ri \in \mathbb N.$
Here
 $ \sigma_{\ri} ( A^*(z_j, y_j ))  $  is  the $\ri$-th singular value of the two-variable function $A^* (z_j, y)$. 
 Suppose that 
\begin{align}\label{eq:main result snr} 
N/\log(N)   >C_{snr }     
   \big( \sigma _{j,\R_j} ^*   \big) ^{ \frac{-d+1}{\alpha} - 2 }    \quad \text{and} \quad  \sigma^*_{j,\R_j+1}\le \sigma^*_{j,\R_j }/N^{\alpha/(2\alpha+1)},
\end{align} 
where    $C_{snr}$ is a sufficiently large constant independent of $N$.
\end{assumption}
\Cref{assume:spectral gap} postulates that  when  we  reshape $A^*$ into the two-variable function $A^*(z_j, y_j)$, the singular gap between its  $\R_j$-th  and $(\R_j+1)$-th singular values   is sufficiently large. Note that \Cref{assume:spectral gap}  trivially holds for additive functions and mean-field functions. 

Finally, we quantify the oracle approximation error when finite-rank tensors are used to approximate densities that may be infinite-rank. Denote
\begin{align}\label{eq:definition of low rank bias main}
\xiapp =  \inf_{\substack{   \dim(\range_{z_j}(B)) \le \R_j \\ \text{ for all } j \in \{ 1,\ldots,d \} }}   \| A^* - B \|_{\lt(\mathcal O^d)}  .
\end{align}  
In other words, \(\xiapp\) represents   the best possible error one could achieve  using any  rank-\((r_1,\ldots, r_d)\) tensor approximation of \(A^*\) in \(\lt(\mathcal O^d)\). In     the following result,   we show that  the  density estimator in \Cref{algorithm:analysis version of general svd version 1} is consistent.

	\begin{theorem} \label{theorem:density bound main} 
  Suppose      \Cref{assume: baised in projected space in operator norm main},  \Cref{assumption:regularity of density main} and \Cref{assume:spectral gap} hold. 
In \Cref{algorithm:analysis version of general svd version 1}, suppose that the inputs satisfy  $\l_j=C_L (\sigma_{j ,{\R}_j}^* )^{-1/\alpha } $ for some sufficiently large constant $C_L$ and  $ \m \asymp  N^{1/(2\alpha+1)}$, and the output is $\widetilde A$.
 Then  it holds that 
\begin{align}\label{eq:density l2 bounds}   \|  \widetilde A  - A  ^* \|  _{\lt(\mathcal O  ^d)}   = \OP\bigg(              \frac{ \sqrt{ \prod_{j  =1}^{d} {{\R}}_j }}{N^{ \alpha/(2\alpha+ 1)} }  + \xiapp\bigg)  .    
\end{align}

\end{theorem}   
	 
The proof of \Cref{theorem:density bound main} can be found in \Cref{section:proof of main result}. In \Cref{section:examples of low rank}, we show that functions from many well-known  models such as the additive model  and the mean-field model  are      low-rank tensors    with  bounded ranks   $(r_1,\ldots, r_d)$ and  approximation error $\xiapp=0$. In these settings,     \eqref{eq:density l2 bounds} reduces to 
  $$\|  \widetilde A - A  ^* \|^2_{\lt({\mathcal O}^d)}    =\OP \bigg( \frac{1}{N ^{2\alpha/(2\alpha+1 ) }} \bigg) ,$$
which matches the minimax optimal rate of estimating  density functions in one dimension.

%% file: numerical.tex
In this section, we compare  the numerical performance of the proposed estimator  VRS   with  classical kernel methods and neural network estimators through  various density   models.  
  The implementation of VRS can be found at \url{https://github.com/darenwang/variance-reduced-tensor-sketching/tree/main}. 

\subsection{Implementations} 
\label{subsection:tuning parameters selection}
 As detailed in \Cref{algorithm:analysis version of general svd version 1},  our approach involves  three groups of tuning parameters: $\m$, $\{\l_j\}_{j=1}^d$, and $ \{\R _j\}_{j=1}^d$.  In all our numerical experiments, the optimal choices for $\m$ and $\{\l_j\}_{j=1}^d$ are determined  through cross-validation.    To select $\{\R_j\}_{j=1}^d$, we apply  a popular method in low-rank matrix estimation known as adaptive thresholding. Specifically, for each $j = 1, \ldots, d$, we compute $\{ \widehat \sigma_{j,k} \}_{k=1}^\infty$, the set of singular values of $\widehat A\times_{j }\mathcal{P}_{\mathcal M_j}\times_{y} \mathcal{P}_{\mathcal S_j}$. Let $ \tau$ be a prespecified parameter. We set $r_j =k -1$   if     $k$ is first index such that 
$$ \widehat \sigma_{j, k}^2  < \tau \sum_{\mu=1}^{k-1}\widehat \sigma_{j, \mu}^2.  $$
This ensures that we can adaptively remove nonzero singular values that are associated  with   the noise of the data.  
In our simulations, we set $\tau =1/50$.
Adaptive thresholding is a very popular strategy in the matrix completion literature (\cite{candes2010matrix}) and it has  been proven to be empirically robust in many scientific and engineering applications.  We use   built-in functions provided by the popular Python package scikit-learn to train kernel estimators, and scikit-learn also utilizes cross-validation for tuning parameter selection.  For neural networks, we use PyTorch to train various models and  make predictions. The implementations of our numerical studies  can be found  at this \href{https://github.com/IvanPeng0414/Nonparametric-estimation-via-variance-reduced-sketching}{link}.
\
\\

We study the numerical performance of   Variance-Reduced Sketching (VRS), kernel density estimators (KDE), and neural networks (NN)  in  various density estimation problems.    For neural network estimators, we use two popular density estimation  architectures: Masked Autoregressive Flow (MAF) (\cite{papamakarios2017masked}) and Neural Autoregressive Flows (NAF) (\cite{huang2018neural}) for comparisons. The details of implementing neural network density estimators  are  provided  in \Cref{sec:numerical appendix}. 
We measure the estimation accuracy by the relative $\lt$-error defined as 
$$ \frac{\|A^*-\widetilde  {A}\|_{\lt (\Omega) }}{\|A^*\|_{\lt(\Omega)}},$$ 
where $\widetilde   {A}$ is the density estimator produced by a given  estimator. We also compute the standard Kullback–Leibler (KL) divergence to measure the distance between two probability density functions: $$D_{KL}=\mathbb{E}_{A^*}[\log ({A^*} /\tilde{A})].$$ 
\\
$\bullet$ {\bf  Simulation $\mathbf{I}$: Four-mode  Gaussian mixture model.} We consider  a  four-mode  Gaussian mixture model in two dimensions. We  generate 20,000 data from the density
\begin{align*}
    A^*(\bvec{x}) = \sum_{i=1}^4 \frac{1}{4}\frac{\exp\left(-\frac{1}{2}(\bvec{x}-\bvec{\mu}_i)^T\bvec{\Sigma}_i^{-1}(\bvec{x}-\bvec{\mu}_i)\right)}{\sqrt{(2\pi)^2|\bvec{\Sigma}_i|}},
\end{align*}
where $\bvec{\mu}_1=\begin{pmatrix}-0.5 \\-0.5 \end{pmatrix},
\bvec{\mu}_2=\begin{pmatrix}0.5 \\0.5 \end{pmatrix},
\bvec{\mu}_3=\begin{pmatrix}-0.5 \\0.5 \end{pmatrix},
\bvec{\mu}_4=\begin{pmatrix}0.5 \\-0.5 \end{pmatrix},
$ and
$\bvec{\Sigma}_1=\bvec{\Sigma}_2=\begin{pmatrix}
0.25^2 & 0.03^2  \\
0.03^2 & 0.25^2 \\
\end{pmatrix}$,
$\bvec{\Sigma}_3 = \bvec{\Sigma}_4 = \begin{pmatrix}
0.1^2 & -0.05^2  \\
-0.05^2 & 0.1^2 \\
\end{pmatrix}
$. This setting is more difficult than {\bf  Simulation $\mathbf{I}$} due to more modes and stronger singularity in the true density.  We report  the relative $\lt$ error and KL divergence   for each method in \Cref{table:four modes Gaussian}.
\begin{table}[H]
    \centering
    \begin{tabular}{|c|c|c|c|c|}
    \hline
         & VRS & KDE & MAF & NAF \\
         \hline
      Relative $\lt$ Error & 0.0721(0.0029) & 0.3987(0.0039) & 0.2441(0.0411) & 0.4617(0.0621) \\
         \hline 
        KL Divergence & 0.0142(0.0015) & 0.1223(0.0014) & 0.0819(0.0161) & 0.2356(0.0417) \\
         \hline
    \end{tabular}
    \caption{\label{table:four modes Gaussian}Relative $\lt$ errors and KL divergences for four different methods of the two-dimensional Gaussian mixture model in {\bf Simulation IV}. The experiments are repeated  for 50 times. The average errors are reported and standard deviations are shown in the bracket.}
\end{table}
 To further visualize  the performance of the four different methods, the true density and the estimated density from each method are plotted in Figure~\ref{fig:2dgaussian-hard}. Direct comparison  in 
Figure~\ref{fig:2dgaussian-hard} demonstrates  VRS  provides a relatively better estimate for the true density.
\begin{figure}[H]
    \centering
    \includegraphics[width=1.0\linewidth]{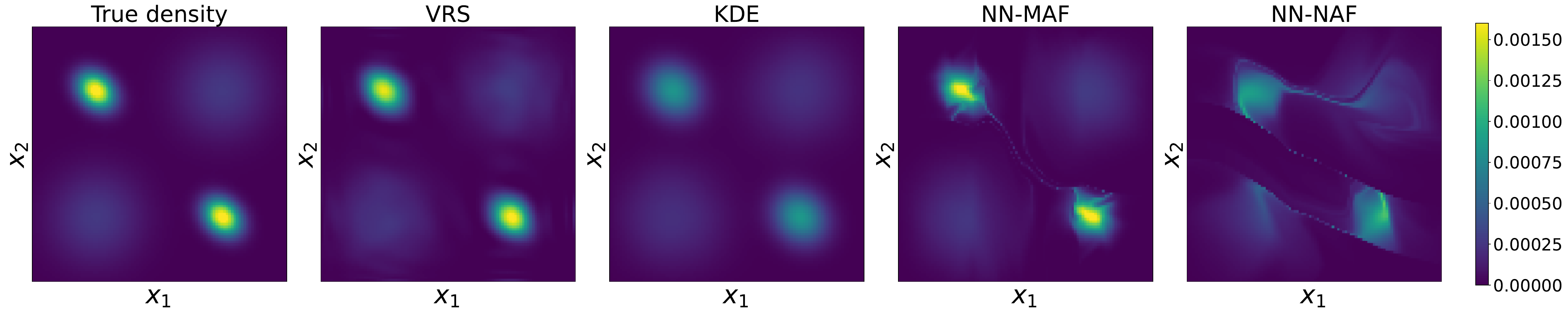}
    \caption{Density functions from the  two-dimensional Gaussian mixture model in {\bf  Simulation $\mathbf{IV}$}. From left to right are  the ground truth density, estimates from  VRS, KDE, MAF, and NAF, respectively. The values in the colorbar on the right represent   function values.}
    \label{fig:2dgaussian-hard}
\end{figure}

\
\\
$\bullet$ {\bf  Simulation $\mathbf{II}$: Two-dimensional Gaussian mixture model.}   In the first simulation, we use a  two-dimensional two-mode  Gaussian mixture model  to numerically compare    VRS, KDE and neural network estimators.  We generate 1,000 samples from the  following density function
\begin{align*}
    A^*(\bvec{x}) = 0.4\frac{\exp\left(-\frac{1}{2}(\bvec{x}-\bvec{\mu}_1)^T\bvec{\Sigma}_1^{-1}(\bvec{x}-\bvec{\mu}_1)\right)}{\sqrt{(2\pi)^2|\bvec{\Sigma}_1|}} + 0.6\frac{\exp\left(-\frac{1}{2}(\bvec{x}-\bvec{\mu}_2)^T\bvec{\Sigma}_2^{-1}(\bvec{x}-\bvec{\mu}_2)\right)}{\sqrt{(2\pi)^2|\bvec{\Sigma}_2|}},
\end{align*}
where $\bvec{\mu}_1=\begin{pmatrix}-0.35 \\-0.35 \end{pmatrix},
\bvec{\Sigma}_1=\begin{pmatrix}
0.25^2 & -0.03^2  \\
-0.03^2 & 0.25^2 \\
\end{pmatrix},
\bvec{\mu}_2 = \begin{pmatrix}0.35 \\0.35 \end{pmatrix},
\bvec{\Sigma}_2 = \begin{pmatrix}
0.35^2 & 0.1^2  \\
0.1^2 & 0.35^2 \\
\end{pmatrix}.$
   The relative $\lt$ error and KL divergence for each method are reported  in \Cref{table:two-mode Gaussian}.  As demonstrated in this example, VRS  achieves decent accuracy in the classical low-dimensional setting.

\begin{table}[H]
    \centering
    \begin{tabular}{|c|c|c|c|c|}
    \hline
         & VRS & KDE & NN-MAF & NN-NAF\\
         \hline
      relative $\lt$ error   & 0.1270(0.0054) & 0.1636(0.0068) & 0.2225(0.0135)& 0.3265(0.0103)\\
         \hline 
        KL divergence & 0.0092(0.0033) & 0.0488(0.0029) & 0.0785(0.0134) & 0.0983(0.0098)\\
         \hline
    \end{tabular}
    
    \caption{\label{table:two-mode Gaussian}Relative $\lt$ errors and KL divergences for four  different methods of the two-dimensional  Gaussian mixture model in {\bf  Simulation $\mathbf{I}$}.  The experiments are repeated  for 50 times. The average errors are reported and standard deviations are shown in the bracket.}
\end{table}
To further visualize  the performance of the four different methods, the true density and the estimated density from each method are plotted in Figure~\ref{fig:2dgaussian-simple}. Direct comparison  in 
Figure~\ref{fig:2dgaussian-simple} demonstrates that  VRS  provides a relatively better estimator for the true density.
\begin{figure}[H]
    \centering
    \includegraphics[width=1.0\linewidth]{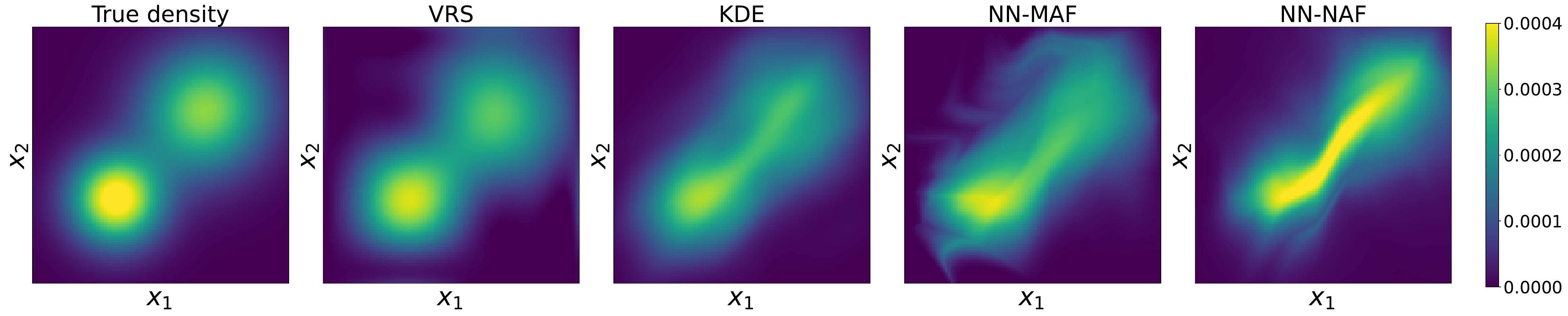}
    \caption{Density functions from the  two-dimensional Gaussian mixture model in {\bf  Simulation $\mathbf{I}$}. From left to right: the ground truth density, and estimations from VRS, KDE, MAF, and NAF. The values in the colorbar on the right represent function values.}
    \label{fig:2dgaussian-simple}
\end{figure}
\
\\
$\bullet$ {\bf  Simulation $\mathbf{III}$: Gaussian mixture in 30 dimensions.} We consider  a higher dimensional density model  in this simulation. Specifically, we generate $10^5$ data points  from the 30-dimensional Gaussian mixture density function
\begin{align*}  
\begin{pmatrix}  x_1 \\x_2 \\x_3\end{pmatrix}
&\overset{\text{ i.i.d. }}{\sim} \frac{1}{2} \mathcal{N}
\left(
\begin{pmatrix}-0.5 \\-0.5 \\-0.5\end{pmatrix},
\begin{pmatrix}
0.1^2 & 0.06^2 & 0 \\
0.06^2 & 0.1^2 & 0 \\
0 & 0 & 0.1^2
\end{pmatrix}
\right)
+ \frac{1}{2} \mathcal{N}
\left(
\begin{pmatrix}0.5 \\0.5 \\0.5\end{pmatrix},
0.1^2 I_{3 \times 3}
\right),
\\
  x_4, x_5 &\overset{\text{ i.i.d. }}{\sim} \mathcal{N}(0, 0.04),
\\
  x_6, \dots, x_{30} &\overset{\text{ i.i.d. }}{\sim} \frac{1}{2} \mathcal{N}(-0.4, 0.3^2) + \frac{1}{2} \mathcal{N}(0.4, 0.3^2).
\end{align*}
  The experiments are repeated for 50 times. Since computing high-dimensional $\lt$ errors is {\it NP-hard}, we only report the averaged KL divergence for performance evaluation.  Table~\ref{table:30dgaussian} showcases that VRS outperforms kernel  and neural network estimators with a remarkable margin.
\begin{table}[H]
    \centering
    \begin{tabular}{|c|c|c|c|c|}
    \hline
         & VRS & KDE & MAF & NAF\\
         \hline
        KL divergence & 0.0195(0.0056) & 4.3823(0.0047) & 0.9260(0.0523) & 0.1613(0.0823)\\
         \hline
    \end{tabular}
    \caption{KL divergence for four methods of the 30-dimensional density model in {\bf  Simulation $\mathbf{II}$}. The experiments are repeated  for 50 times. The average errors are reported and standard deviations are shown in the bracket. }
    \label{table:30dgaussian}
\end{table}
To further visualize  the performance of the four different methods, 
 we provide visualization of  a few estimated marginal densities and compare them with the ground truth marginal density.
Figure~\ref{fig:30dgaussian} depicts the comparison of the two-dimensional   marginal densities corresponding to   $(x_1,x_2)$, $(x_4,x_8)$, and $(x_{10},x_{20})$, respectively.  
Direct comparison  in 
Figure~\ref{fig:30dgaussian} demonstrates  VRS  provides a relatively better fit for the true density.  

\begin{figure}[h!]
    \centering
    \includegraphics[width=1.0\linewidth]{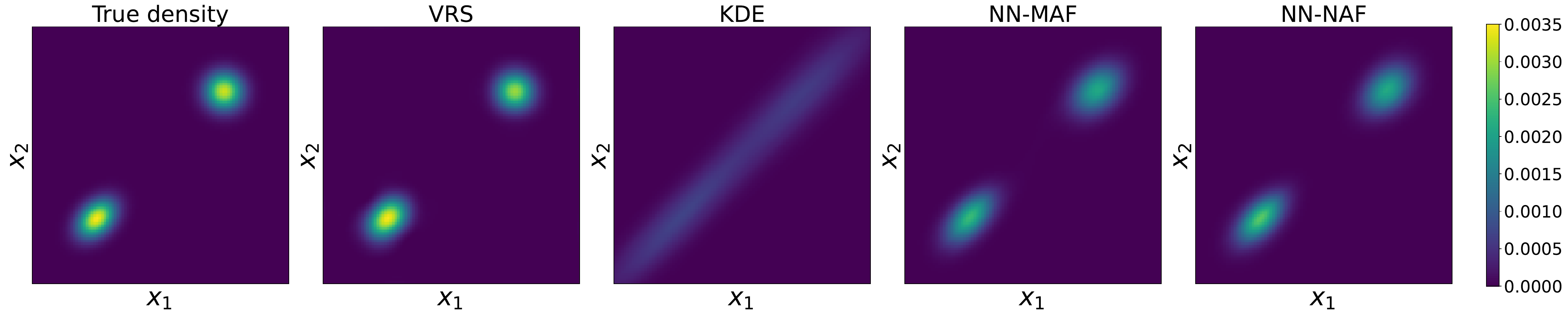}
    \includegraphics[width=1.0\linewidth]{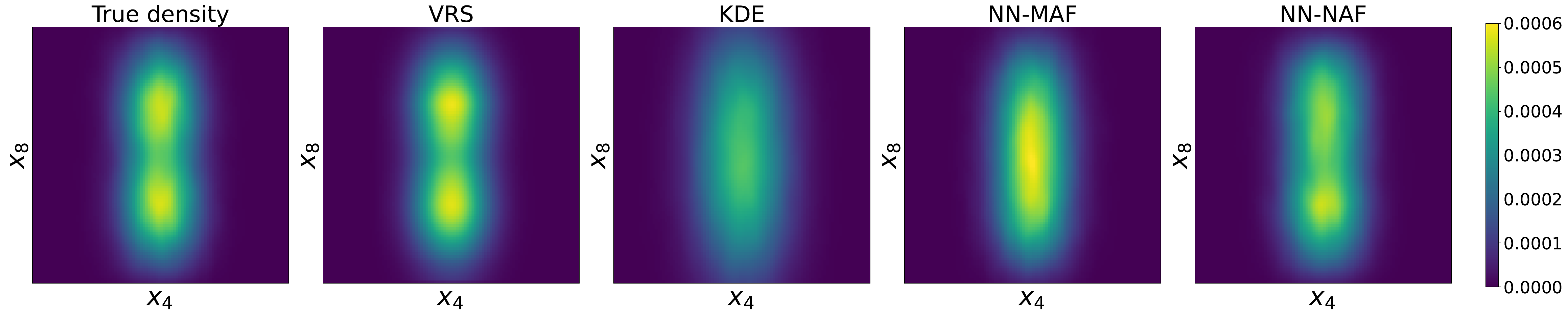}
    \includegraphics[width=1.0\linewidth]{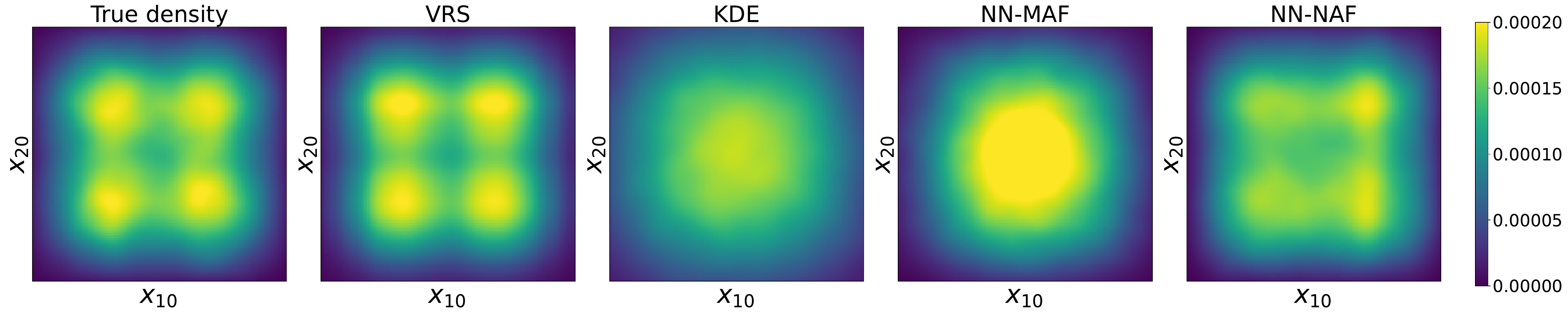} 
    \caption{Marginal densities from the 30-dimensional Gaussian mixture model in {\bf Simulation II}. From left to right: the ground truth density, and estimations from VRS, KDE, MAF, and NAF. From top to bottom: two-dimensional marginal densities corresponding to $(x_1, x_2)$, $(x_4, x_8)$, and $(x_{10}, x_{20})$. The values in the colorbar on the right represent    function values.}
    \label{fig:30dgaussian}
\end{figure}
\    
\\
$\bullet$ {\bf  Simulation $\mathbf{IV}$: The Ginzburg-Landau model.} Ginzburg-Landau theory is widely used to model the microscopic behavior of superconductors. The Ginzburg-Landau density has the following expression 
\begin{align} \label{eq:GL density} 
A^*(x_1,\ldots,x_d)\propto \exp\left(-\beta \bigg\{\sum_{j=0}^d \frac{\lambda}{2}(\frac{x_j-x_{j+1}}{h})^2
+\sum_{j=1}^d\frac{1}{4\lambda}(x_j^2-1)^2\bigg\}\right)  , 
\end{align}
where  $x_0=x_{d+1}=0$. We sample  data from the Ginzburg-Landau  density 
with coefficient $\beta=1/8, \lambda=0.02, h=1/(d+1)$ using the Metropolis-Hastings sampling algorithm.  This type of density   concentrates on two centers $(+1,+1,\cdots,+1)$ and $(-1,-1,\cdots,-1)$, and all the coordinates $(x_1, \ldots, x_d)$ are   correlated  in a non-trivial way  due to the interaction term $\exp \big(-\beta\sum_{j=0}^d \frac{\lambda}{2}(\frac{x_j-x_{j+1}}{h})^2\big)$ in the density function.   We consider two sets of experiments for the Ginzburg-Landau density model. In the first set of experiments,  we fix $d=10$ and change the sample size  $N$ from $1\times 10^5$ to $5\times 10^5$. In the second set of experiments, we keep the sample size $N$ at $1\times 10^5$ and vary $d$ from $2$ to $10$. We summarize the averaged relative $\lt$-error  for each method in \Cref{fig:DE_GL}. 
Furthermore in Figure~\ref{fig:10dGL-visualize}, we visualize several two-dimensional marginal densities estimated by the  four different methods with sample size $1\times 10^5$. 
Direct comparison showcases that our VRS method  recovers these marginal densities with   decent accuracy.\par 

\begin{figure}[H]
    \centering
    \includegraphics[width=0.48\linewidth]{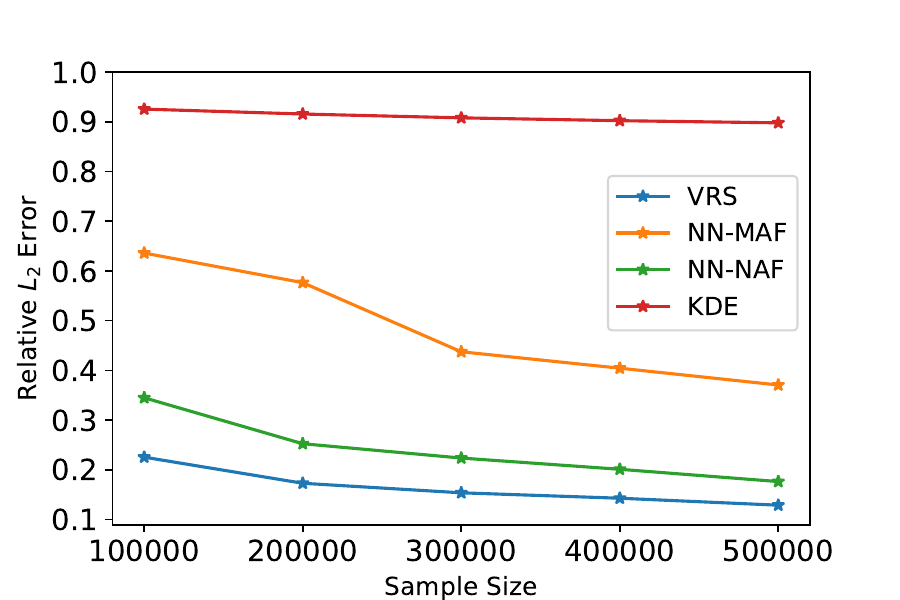}
    \includegraphics[width=0.48\linewidth]{fig/GL_L2_diffd.pdf}
    \caption{The plot on the left corresponds to {\bf Simulation $\mathbf{III}$} with $d=10$ and  $N$ varying from  $1\times 10^5$ to $5\times 10^5$ ; the plot on the right  corresponds to {\bf Simulation $\mathbf{III}$} with  $N$ being $1\times 10^5$ and $d$ varying from $2$ to $10$.   }
    \label{fig:DE_GL}
\end{figure} 

\begin{figure}
    \centering
    \includegraphics[width=1.0\linewidth]{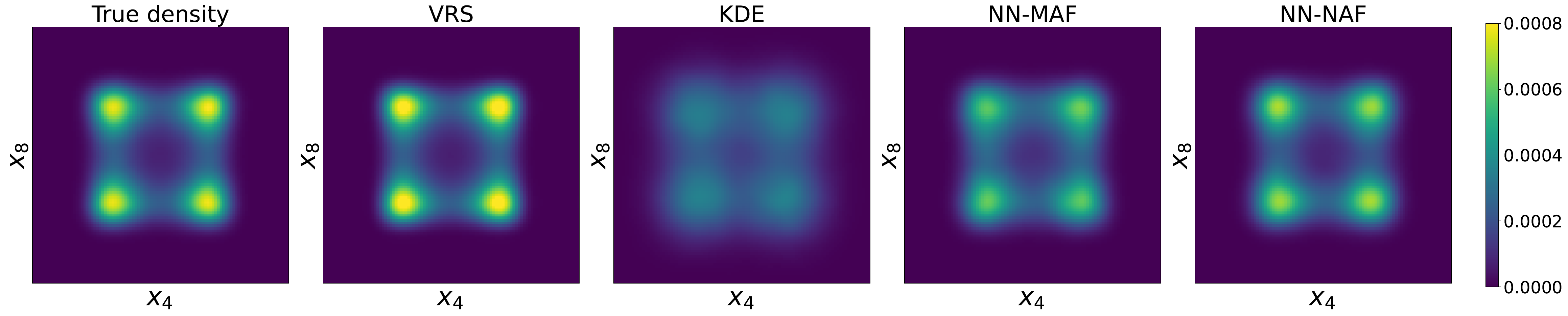}
    \includegraphics[width=1.0\linewidth]{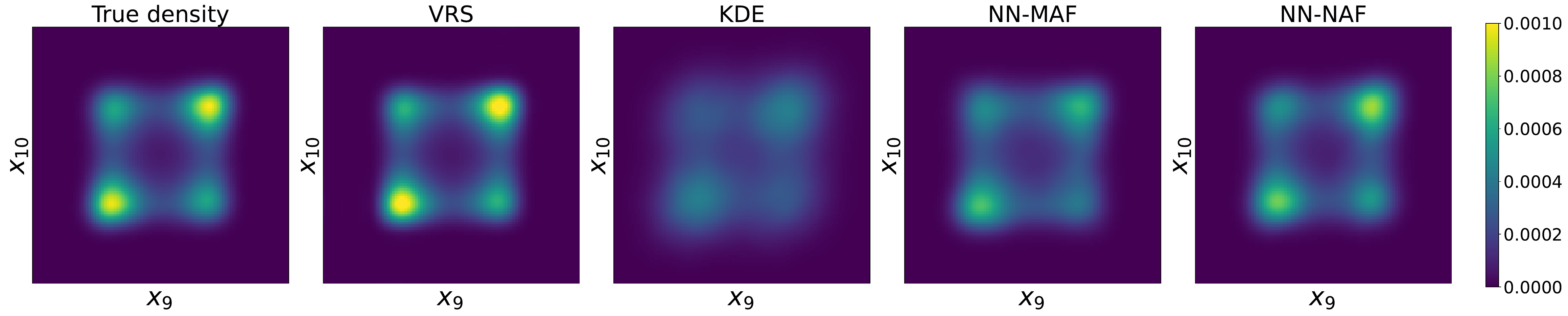}    
    \caption{Marginal densities from the 10-dimensional Ginzburg-Landau model in {\bf Simulation III}. and estimations from VRS, KDE, MAF, and NAF. From top to bottom: two-dimensional marginal densities corresponding to $(x_4, x_8)$ and $(x_9, x_{10})$. The values in the colorbar on the right represent the density function values.}
    \label{fig:10dGL-visualize}
\end{figure}
 
\  
\\
$\bullet$ {\bf  Real data example.} We analyze  the density estimation  for the  \href{https://archive.ics.uci.edu/dataset/186/wine+quality}{Portugal wine quality} dataset from the UCI Machine Learning Repository. This dataset contains $6497$ samples of red and white wines, along with $8$ continuous variables: volatile acidity, citric acid, residual sugar, chlorides, free sulfur dioxide, density, sulphates, and alcohol. To provide a comprehensive comparison between different methods, we estimate the joint density of the first  $d$ variables in this dataset, allowing 
$d$ to vary from 2 to 8. For instance, $d=2$ corresponds to the joint density of volatile acidity and citric acid. Since the true density is unknown, we randomly split the dataset into a 90\% training set and a 10\% test set, and   evaluate the performance of various approaches using the averaged log-likelihood of the test  data.  The averaged log-likelihood is defined as follows: let
$\widetilde p$ be the density estimator based on the training data. The averaged log-likelihood of the test  data $\{ Z_i \}_{i=1}^{ N _{\text{test}} } $ is 
$$ \frac{1}{N _{\text{test}} }\sum_{i=1}^{N_{\text{test}}  } \log\{  \widetilde{p}(Z_i )     \} . $$
The numerical performance of VRS, NN, and KDE   is summarized in \Cref{fig:DE_wine}.   Notably, VRS achieves the highest  averaged log-likelihood values, indicating  its superior numerical performance.

\begin{figure}
    \centering
    \includegraphics[width=0.5\linewidth]{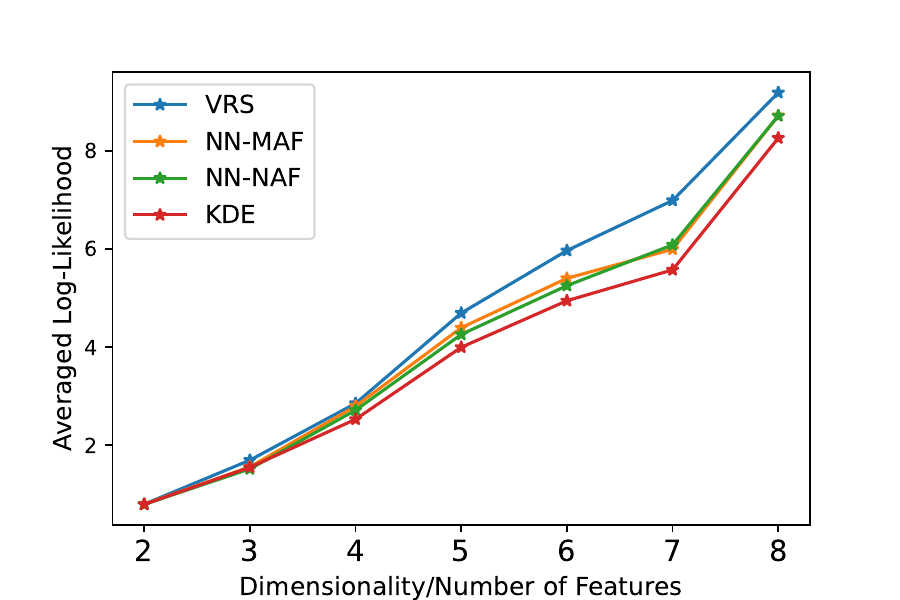}
    \caption{Density estimation for the Portugal wine quality dataset with VRS, KDE, and neural network estimators.  }
    \label{fig:DE_wine}
\end{figure}

\section{Examples of exactly  and approximately low-rank functions}
\label{section:examples of low rank}

In this section, we present three examples commonly encountered in nonparametric statistical literature that are approximately low rank.

\begin{example}[Additive models in regression] \label{example:additive}
In multivariate  nonparametric statistical literature, it is commonly assumed that the underlying unknown    nonparametric function 
  $f^*:[0,1]^d \to \mathbb R$   process an additive structure, meaning    that there exists  a collection of univariate functions $\{f_j^*(z_j):[0,1]\to \mathbb R \}_{j=1}^d$  such that 
  $$f^*(z_1,\ldots z_d) = f_1^*(z_1) + \cdots+ f_d^*(z_d)  \text{ for all } (z_1, \ldots, z_d) \in [0,1]^d$$
Let  $ y=(z_2,\ldots, z_d),$ we show that the rank of  $f^*(z_1, y)$ is at most 2.  By \eqref{eq:definition of range}, $\range_1(f^*) =\Span\{1, f^*_1(z_1)\}$.
The dimensionality of $ \range_{1}(f^*)$ is at most $2$, and  consequently  the rank  of $f^*(z_1, y) \in \lt([0,1])\otimes \lt([0,1]^{d-1})    $ is at most $2$. 

\end{example}

\begin{example}[Mean-field models in density estimation]\label{example:mean field}
Mean-field theory is a popular framework   in computational physics and Bayesian probability   as it  studies the behavior of high-dimensional  stochastic  models.   
The main idea of the mean-field theory is to replace all interactions to any one body with an  effective interaction  in a  physical system. Specifically, the mean-field model assumes that 
the   density function $p^*(z_1,\ldots, z_d):[0,1]^d \to \mathbb R^+ $   can be well-approximated by 
$    p_1^*(z_1) \cdots p_d^*(z_d)   $,  where for $j=1,\ldots, d$, $p^*_j(z_j): [0,1]\to \mathbb R^+$ are univariate marginal density functions. The  readers are referred  to  \cite{blei2017variational} for further discussion. For simplicity, suppose 
  $$p ^*(z_1,\ldots, z_d)   = p_{  1}  ^*(z_1) \cdots p_{ d}  ^*(z_d)  .$$  Then   according to \eqref{eq:range definition}, $\range_{ 1}(p^*) =\Span\{ p_{ 1}^*(z_1) \}   $.
The dimensionality of $ \range_{ 1}(p^*)$ is at most $\alpha$, and therefore   the rank  of $p^*    $ in the variable $z_1$ is   1. 

\end{example}

\begin{example}[Multivariate Taylor expansion]\label{example:taylor}
Suppose  $ G :[0,1]^d \to \mathbb R    $ is an  $\alpha$-times continuously   differentiable function. Then    Taylor's theorem in the multivariate setting states that  for $ z=(z_1,\ldots, z_d) \in \mathbb R^d$ and $t=(t_1,\ldots, t_d) \in \mathbb R^d$, $G(z) \approx  T_t(z) $, where 
\begin{align} \label{eq:taylor multivariate}T_t (z) =  G(t) + \sum_{k=1}^\alpha \frac{1}{k! } \mathcal D^k G (t, z-t),
\end{align} 
and 
$
\mathcal{D}^k  G(x,h)= \sum_{i_1,\cdots,i_k=1}^d    \partial_{i_1} \cdots \partial_{i_k}  G (x)   h_{i_1}\cdots h_{i_k}.
$
For example,	
$\mathcal D G(x, h)= \sum_{i=1}^d    \partial_i   G  (x)   h_i$,   $
\mathcal D^2 G (x, h)= \sum_{i=1}^d \sum_{j=1}^d  \partial_i  \partial_j  G (x)   h_ih_j,$ and so on. To simplify our discussion, let $t=0 \in \mathbb R^d$. Then \eqref{eq:taylor multivariate} becomes 
 $$T_0 (z) = G (0) + \sum_{i=1}^d    \partial_i    G  (0)  z_i +  \frac{1}{2!}\sum_{i=1}^d \sum_{j=1}^d  \partial_i \partial_j  G  (0)  z_iz_j + \ldots +
\frac{1}{\alpha!}\sum_{i_1, \ldots, i_\alpha=1}^d    \partial_{i_1} \cdots \partial_{i_\alpha}  G (0)   z_{i_1}\cdots z_{i_\alpha} . $$
Let 
  $  y=(z_2,\ldots, z_d).$   Then    by \eqref{eq:range definition},    $\range_1(T_0 ) =\Span\{ 1, z_1, z_1^2,\ldots, z_1^{\alpha }  \} $.
The dimensionality of $ \range_1(T_0)$ is at most $\alpha+1$, and therefore $  G   $  can be well-approximated by  finite rank functions.     
\end{example}

\section{Conclusion}\label{sec:conclusion}

In this paper, we develop a comprehensive framework  Variance-Reduced Sketching (VRS)  for nonparametric density estimation problems  in higher dimensions. Our approach leverages the concept of sketching from numerical linear algebra to address the curse of dimensionality in function spaces. Our method treats  multivariable functions as infinite-dimensional matrices or tensors and the selection of sketching is specifically tailored to the regularity of the estimated function. This design takes the variance of random samples in nonparametric problems into consideration, intended to reduce the curse of dimensionality in density estimation problems. Extensive simulated  experiments and real data examples  demonstrate that our sketching-based method substantially outperforms both neural network estimators and classical kernel density methods in terms of numerical performance.

%% file: appendix.tex
\section{Tensors and multivariable functions }
\subsection{Finite-dimensional tensors}
Let $B  $ be a tensor of size $  m_1\times m_2 \times \cdots \times  m_d $. Let 
$  B (\mi _1,\ldots, \mi_k ; \mi_{k+1},\ldots, \mi _d) $  be   the $k$-th unfolding matrix of $B$   such that 
\begin{align}\label{eq:unfold} B (\mi_1,\ldots, \mi_k ; \mi_{k+1},\ldots, \mi_d)= B (\mi_1,\ldots, \mi_d). 
\end{align}
Here $B (\mi_1,\ldots, \mi_k ; \mi_{k+1},\ldots, \mi_d)$ is a matrix in $\prod_{s=1}^k m_s \times \prod_{s=k+1}^d m_s $ dimensions, and its element  in  row  $(\mi_1,\ldots, \mi_k )$  and  column $( \mi_{k+1},\ldots, \mi_{d})$ equals to $B (\mi_1,\ldots, \mi_d)$ as in \eqref{eq:unfold}. The matrix $B (\mi_1,\ldots, \mi_k ; \mi_{k+1},\ldots, \mi_d)$ can be obtained from $B$ by the MATLAB function 
$$\text{reshape} \bigg( B, \ \prod_{s=1}^k m_s ,\prod_{s=k+1}^d m_s \bigg). $$
\begin{remark}
Given $\{ \mi_k\}_{k=1}^d  $ such that $1\le \mu_k \le m_k $, we can identify the integer   $(\mi_1,\ldots, \mi_k )  $ in the following recursive way.  Set $   (\mu_1,\mu_2) =  (\mu_1-1)\cdot m_2 +\mu_2$. So $1\le (\mu_1,\mu_2)  \le m_1m_2$.  Then
$  (\mu_1,\mi_2,\mi_3) = ((\mu_1,\mu_2) ,\mi_3) = ((\mu_1,\mu_2) -1)\cdot m_3  +\mi_3 $, and so on.   
When working with $B (\mi _1,\ldots, \mi_k ; \mi_{k+1},\ldots, \mi _d)$, the   $k$-th unfolding matrix of  the tensor $B$, it suffices    to  use $(\mi_1,\ldots, \mi_k )$ as the row index, and  $( \mi_{k+1},\ldots, \mi_{d})$  as the column index. Therefore  it is not necessary to keep track of the exact  value of the 
 integer $(\mi_1,\ldots, \mi_k )$.
\end{remark}
Let $Q_k\in \mathbb R^{m_k\times q_k}$ be a matrix. Then 
$  B\times_k Q_k  $ is a tensor of size $m_1\times \ldots\times m_{k-1}\times q_k\times m_{k+1}\times \ldots m_d$ such that 
$$ (B\times_kQ_k )(\mi_1,\ldots, \mi_{k-1}, \li_{k}, \mi_{k+1}, \ldots, \mi_{d}) = 
\sum_{\mi_k =1}^{q_k} B(\mi_1,\ldots, \mi_{k-1}, \mi_{k}, \mi_{k+1}, \ldots, \mi_{d}) Q_k(\mi_k, \li_k)
$$

\subsection{Tensors in function spaces}
We first introduce  a convenient notation for integration of two functions. 
For $j=1,2,3$  let $\Omega_j \subset \mathbb R^{d_j}$  be any regular domain.  
Let $A(x, y) \in  \lt(\Omega_1\times \Omega_2 ) $ and $Q(y,z) \in  \lt(\Omega_2\times \Omega_3 ) $ be two functions. Then 
\begin{align}\label{eq:integration as matrix multiplication} ( A\times _y Q ) (x, z)  = \int_{\Omega_2} A(x, y) Q(y,z) dy.  
\end{align}
Let $\mathcal O \subset \mathbb R$ be a measurable subset.
For arbitrary  univariate functions $ \{ f_k(z_k)\}_{k=1}^d $, denote 
$f_1\otimes \cdots \otimes  f_d  \in \lt(\mathcal O^d)  $  as 
\begin{align} \label{eq:tensor of univariate functions}
     f_1\otimes \cdots \otimes  f_d  (z_1,\ldots, z_d) =f_1 (z_1)   \cdots    f_d(z_d).
\end{align}  
Let $ \{ \mathcal S_j\}_{j=1}^d \subset \lt(\mathcal O) $ be   arbitrary  subspaces, where functions in $\mathcal S_j $ correspond to  the variable $z_j$. Define
$$ \mathcal S_{1}\otimes \cdots \otimes  \mathcal S_{d} = \Span\{g_1(z_1) \cdots g_d(z_d) : g_1(z_1) \in   \mathcal S_{1}, \ldots, g_d(z_d) \in   \mathcal S_{d}\}. $$
So $ \mathcal S_{1}\otimes \cdots \otimes  \mathcal S_{d} \subset \lt(\mathcal O^d)$.
Let $ \{ \phi_{\mi}   \}_{\mi=1}^ \infty$ be a collection of complete basis in $\lt(\mathcal O)$, and
   $m   \in \mathbb Z^+ $. For $j \in \{ 1,\ldots,d\}$,  define  
\begin{align}  \mathcal M_j =   \Span \{ \phi_{\mi} (z_j)  \}_{\mi=1}^ m  . 
\end{align}
So $\mathcal M _j\subset \lt(\mathcal O) $  and  $ \mathcal M_{1} \otimes \cdots \otimes \mathcal M_d \subset \lt (\mathcal O^d )$. Denote 
$$ \mathcal P_{\mathcal M_j} (z_j, z_j') = \sum_{\mi=1}^m \phi_\mi(z_j) \phi_\mi(z_j') . $$
For any function $A \in \lt(\mathcal O^d )$,   we define 
$A  \times_1 \mathcal P _{\mathcal M_1} (z_1, \ldots, z_d)$ using  \eqref{eq:integration as matrix multiplication} as 
$$ A  \times_1 \mathcal P _{\mathcal M_1} (z_1, \ldots, z_d) = \int_{\mathcal O} A(z_1, \ldots, z_{j-1}, z_j',z_{j+1},\ldots, z_d ) \mathcal P_{\mathcal M_j} (z_j', z_j) dz'_j.$$
Therefore 
$A  \times_1 \mathcal P _{\mathcal M_1} \times \ldots \times_d  \mathcal P_{\mathcal M_d}    $  can be defined inductively.  A direct calculation shows that   
$$  A {\times_1 \mathcal P _{\mathcal M_1} \times \ldots \times_d  \mathcal P_{\mathcal M_d}}(z_1, \ldots, z_d) = \sum_{\mi_1=1}^m \cdots \sum_{\mi_d=1}^m  \langle A, \phi_{\mi_1} \otimes \cdots \otimes \phi_{\mi_d}\rangle  \phi_{\mi_1}(x_1) \cdots   \phi_{\mi_d} (x_d) . $$
Therefore   $  A {\times_1 \mathcal P _{\mathcal M_1} \times \ldots \times_d  \mathcal P_{\mathcal M_d}} $ is a function in the space $\mathcal M_{1} \otimes \cdots \otimes \mathcal M_d  $.
Given data $\{ \BZ_i\}_{i=1}^N \subset \mathcal O^d$,  denote the corresponding  empirical measure as 
$$ \widehat A    = \frac{1}{N} \sum_{i=1}^N \delta _{\BZ_i}  ,$$
where $ \delta _{\BZ_i }$ is the indicator function   at the  point $\BZ_i$.
We can also project the empirical measure $\widehat A$ onto $\mathcal M_{1} \otimes \cdots \otimes \mathcal M_d  $ as 
\begin{align*}
    \widehat A  {\times_1 \mathcal P _{\mathcal M_1} \times \ldots \times_d  \mathcal P_{\mathcal M_d}}  (z_1, \ldots, z_d) = &\sum_{\mi_1=1}^m \cdots \sum_{\mi_d=1}^m  \langle  \widehat A, \phi_{\mi_1} \otimes \cdots \otimes \phi_{\mi_d}\rangle  \phi_{\mi_1}(z_1) \cdots   \phi_{\mi_d} (z_d) 
   \\
   = &\sum_{\mi_1=1}^m \cdots \sum_{\mi_d=1}^m \bigg \{ \frac{1}{N } \sum_{i=1}^N    \phi_{\mi_1}  \otimes \cdots \otimes \phi_{\mi_d} (\BZ_i)  \bigg\}   \phi_{\mi_1}(z_1) \cdots   \phi_{\mi_d} (z_d)  
\end{align*}
\begin{lemma}
Suppose $\{ \BZ_i\}_{i=1}^N \subset \mathbb R^{d}$ are sampled independently from the density $A^*$. Then 
$$\E ( \widehat A {\times_1 \mathcal P _{\mathcal M_1} \times \ldots \times_d  \mathcal P_{\mathcal M_d}} )  = A^*  {\times_1 \mathcal P _{\mathcal M_1} \times \ldots \times_d  \mathcal P_{\mathcal M_d}} . $$
\end{lemma}
\begin{proof} It suffices to observe that 
$ \E ( \phi_{\mi_1}  \otimes \cdots \otimes \phi_{\mi_d} (\BZ_i) ) = \langle A^*, \phi_{\mi_1} \otimes \cdots \otimes \phi_{\mi_d}\rangle.$ 
\end{proof}

\subsection{The coefficient tensor}

\begin{definition}[Coefficient tensors] \label{definition:coefficient tensor}
Suppose  $ f \in \mathcal M_{1} \otimes \cdots \otimes \mathcal M_d    \subset  \lt (\mathcal O^d )$. 
Let  $\mathcal C(f) $  be   a tensor in  $\mathbb R^{m^d}$    satisfying
\begin{align}\label{eq:coefficient tensor definiton} \mathcal C(f) (\mi _1,\ldots, \mi _d)  =  \langle f, \phi_{\mi_1} \otimes \cdots  \phi_{\mi_d} \rangle, 
\end{align} 
where $1\le \mi_1,\ldots, \mi_d\le m $. Here, $\mathcal C(f)$ is the coefficient tensor of $f$. Note that the coefficient tensor $\mathcal C(f) $ is defined  with respect to   the basis functions $ \{ \phi_{\mi}   \}_{\mi=1}^ m$.   
\end{definition} 
The following lemma shows that the coefficient tensor is distance preserving, and consequently the space of $\mathcal M^{\otimes d}$ and $\mathbb R^{m^d}$ have the same topology.

\begin{lemma} \label{lemma:coefficient tensor distance preserving}
Let 
$f \in \mathcal M_{1} \otimes \cdots \otimes \mathcal M_d     $, and let $ \mathcal C(f) $ be defined in \eqref{eq:coefficient tensor definiton}. Then 
$$ \| f\|_{\lt(\mathcal O^d)} = \| \mathcal C(f) \|_{F}.$$
\end{lemma}
\begin{proof}
    Let $\{ \phi_{\mi}\}_{\mi=1}^\infty  $ be the complete basis functions of $\lt(\mathcal O) $. Then $\{ \phi_{\mi_1}\otimes \cdots \otimes \phi_{\mi_d}\}_{\mi_1,\ldots, \mi_d=1}^\infty $ is a   complete basis for $\lt(\mathcal O^d) $. Therefore 
    \begin{align*}
       \| f\|_{\lt(\mathcal O^d)} ^2 =  \sum_{\mi_1=1}^\infty  \cdots \sum_{\mi_d=1}^ \infty \langle    f, \phi_{\mi_1} \otimes \cdots \otimes \phi_{\mi_d}\rangle ^2 
        =  \sum_{\mi_1=1}^m \cdots \sum_{\mi_d=1}^ m  \langle    f, \phi_{\mi_1} \otimes \cdots \otimes \phi_{\mi_d}\rangle ^2  = \| \mathcal C(f) \|_{F} ^2,
    \end{align*}
    where the second  equality follows from the assumption that $f\in \mathcal M_{1} \otimes \cdots \otimes \mathcal M_d   $.
\end{proof}

The following lemma demonstrates the compatibility of the coefficient matrices.
\begin{lemma} \label{lemma:product compactible coefficient matrix}
For $j=1,2,3$  let $\Omega_j \subset \mathbb R^{d_j}$  be any regular domain. Let $x\in \Omega_1, y\in \Omega_2, $ and $z \in \Omega_3$  be three variables in arbitrary dimensions. Let $\mathcal M = \Span \{ \phi_\mi(x) \}_{\mi=1}^m$,   $\mathcal L = \Span\{ \psi_\li(y)\}_{\li=1}^l $, and  $ \mathcal N =\Span \{\xi_k(z) \}_{k=1}^n$ be three subspaces of functions, where  $\{ \phi_\mi(x) \}_{\mi=1}^m $, $\{ \psi_\li(y)\}_{\li=1}^l $, and 
$  \{\xi_k(z) \}_{k=1}^n$ are othornomal functions. 
Let $A(x, y) \in \mathcal M \otimes \mathcal L$ and $Q(y,z) \in \mathcal L \otimes \mathcal N $ be two functions. Then 
 $  A\times _y Q  (x, z) = \int_{\Omega_2} A(x, y) Q(y,z) dy  $ satisfies 
 $$\mathcal C (A \times_y Q) = \mathcal C(A) \mathcal C(Q) ,   $$
 where $ \mathcal C(A) \in \mathbb R^{m\times l }$ and $ \mathcal C(Q) \in \mathbb R^{l\times n }$.
\end{lemma}
\begin{proof} The proof follows from a  standard coordinate matching calculation. By definition, the $\mi, k$ entry of the matrix $\mathcal C (A \times_y Q)  \in \mathbb R^{m\times n} $ is 
\begin{align*}
\int_{\Omega_1} \int_{\Omega_3} \int_{\Omega_2} A(x, y) Q(y,z) dy  \phi_\mi (x) \xi _k (z) dz dx.
\end{align*}
Note that for any function $f(y) \in \mathcal L , g(y)\in \mathcal L $, 
$$ \int _{\Omega_2} f(y) g(y) dy =  \sum_{\li=1}^\infty \int _{\Omega_2}f(y) \psi_\li (y) dy 
 \int_{\Omega_2}  g(y) \psi_\li (y) dy  = \sum_{\eta=1}^ l \int_{\Omega_2} f(y) \psi_\li (y) dy 
 \int_{\Omega_2}  g(y) \psi_\li (y) dy  . $$
 Since    $A(x, y) \in   \mathcal L $ for fixed $x$,
 and   $Q(y,z) \in \mathcal L$ for fixed $z$, it follows that 
 \begin{align*}
& \int_{\Omega_1} \int_{\Omega_3} \int_{\Omega_2} A(x, y) Q(y,z) dy  \phi_\mi (x) \xi _k (z) dz dx  \\
=  &\int_{\Omega_1 } \int_{\Omega_3 } \sum_{\eta=1}^ l \bigg\{ \int_{\Omega_ 2}A(x, y) \psi_\li (y) dy \bigg\} \bigg\{ \int_{\Omega_2} Q(y, z) \psi_\li (y) dy \bigg\}  \phi_\mi (x) \xi _k (z) dxdz 
\\
=& \sum_{\eta=1}^ l  \bigg\{  \int_{\Omega_1 } \int _{\Omega_2} A(x, y) \psi_\li (y) \phi_\mi (x) dy  dx \bigg\}  \bigg\{\int_{\Omega_2  } \int_{\Omega_3 } Q(y, z) \psi_\li (y) \xi _k (z)  dy   dz \bigg\}    
\end{align*}
which is the $(\mi, k)$-entry of the matrix $ \mathcal C(A) \mathcal C(Q)$.
 \end{proof}

\subsection{Unfolding the coefficient tensor }
\label{subsection:unfolding coefficient tensor}
The coefficient tensor of $A {\times_1 \mathcal P _{\mathcal M_1} \times \ldots \times_d  \mathcal P_{\mathcal M_d}} $ satisfies 
\begin{align}\label{eq:coefficient tensor definition} \mathcal C( A {\times_1 \mathcal P _{\mathcal M_1} \times \ldots \times_d  \mathcal P_{\mathcal M_d}}) (\mi_1,\ldots, \mi_d)  = \langle A, \phi_{\mi_1} \otimes \cdots \otimes \phi_{\mi_d}\rangle  .
\end{align}
For $1\le k < d $,  let $x=(z_1,\ldots,z_k) $ and $y=(z_{k+1},\ldots, z_d) $.  
Denote 
$$ \mathcal M^{\otimes k}  = \mathcal M_1 \otimes \cdots \otimes \mathcal M_{k} \quad \text{and} \quad \mathcal M^{\otimes d-k} = \mathcal M_{k+1} \otimes \cdots \otimes \mathcal M_{d} .  $$
Then $ A \times_x \mathcal P_{\mathcal M^{\otimes k}}  \times_y \mathcal P_{\mathcal M^{\otimes d-k}}     $ is a function in $\mathcal M^{\otimes k} \otimes\mathcal M^{\otimes d-k} $. 
The coefficient matrix 
 $  \mathcal C (A \times_x \mathcal P_{\mathcal M^{\otimes k}}  \times_y \mathcal P_{\mathcal M^{\otimes d-k}})    \in \mathbb R^{m^k \times m^{d-k}} $ satisfies 
\begin{align} \nonumber & \mathcal C ( A \times_x \mathcal P_{\mathcal M^{\otimes k}}  \times_y \mathcal P_{\mathcal M^{\otimes d-k}}) (\mi_1,\ldots, \mi_k; \mi_{k+1}, \ldots,\mi_{d}) 
\\ = & \langle A, ( \phi_{\mi_1} \otimes  \cdots \otimes \phi_{\mi_k}) \otimes  ( \phi_{\mi_{k+1}} \otimes  \cdots \otimes \phi_{\mi_d})\rangle .
 \label{eq:unfold coefficient tensor} 
 \\
= & \langle A, \phi_{\mi_1} \otimes \cdots \otimes \phi_{\mi_d}\rangle  = \nonumber  \mathcal C( A {\times_1 \mathcal P _{\mathcal M_1} \times \ldots \times_d  \mathcal P_{\mathcal M_d}}) (\mi_1,\ldots, \mi_d)    .
\end{align}
In other words,    $\mathcal C (A {\times_x \mathcal P_{\mathcal M^{\otimes k}}  \times_y \mathcal P_{\mathcal M^{\otimes d-k}}}) $ is the $k$-th unfolding matrix of $\mathcal C( A\timesMD)  $.

\subsection{Best low rank approximations}
\label{appendix:low-rank function}
Throughout the section, we denote 
$x\in \Omega_1 \subset \mathbb R^{d_1}$ and $ y\in \Omega_2 \subset \mathbb R^{d_2}$. Thus $x$ and $y$ can be variables in arbitary dimensions. 
\begin{theorem}
\label{theorem:svd Hilbert}
[Singular value decomposition in function space]
Let $A   (x, y) :\Omega_1 \times \Omega_2 \to \mathbb R $ be any function such that $\|A \|_{\lt(\Omega_1 \times \Omega_2)}   <\infty $.   There exists   a collection of     singular values   $  \sigma_1  (A )  \ge \sigma_2(A) \ge \cdots \ge 0 $, and two collections of orthonormal basis functions $\{ \Phi_\ri(x)  \}_{\ri=1}^{\infty }  \subset \lt(\Omega_1)  $ and $\{ \Psi_\ri   (y)  \}_{\ri=1}^{\infty } \subset \lt(\Omega_2 ) $  such that 
\begin{align}\label{eq:svd Hilbert} A  (x, y) = \sum_{\ri=1}^{\infty} \sigma_{\ri} (A )  \Phi_\ri  (x) \Psi_\ri(y).  
\end{align} 
\end{theorem} 
\begin{proof} See  Section 6 of \cite{brezis2011functional}.
\end{proof} 
  Therefore  $A (x, y)$ can be viewed as an infinite-dimensional matrix.   Suppose in addition   that $A(x,y)$ has exactly $\beta  $ many strictly positive singular values, where $\beta \in \mathbb Z^+\cup \{ \infty\}$. Then the rank of $A(x,y)$ is $\beta$.

\begin{theorem} \label{theorem:best low rank function}
Let $\{ \Phi_\ri(x)  \}_{\ri=1}^{\infty }   $ and $\{ \Psi_\ri   (y)  \}_{\ri=1}^{\infty }   $ be defined as in \eqref{eq:svd Hilbert}.  For $r\in \mathbb Z^+$, denote 
$$[A]_r (x,y)  =  \sum_{\ri=1}^{r} \sigma_{\ri} (A )  \Phi_\ri  (x) \Psi_\ri(y).    $$ Then $[A]_r$ is the best low rank approximation of the function $A$ in the sense that 
\begin{align*} \|A (x, y) - [A ] _r(x, y) \|_{\op} = &  \min_{   \text{ rank}(B) \le r  } \|  A (x, y) -B(x, y) \|_{\op } =\sigma_{\R+1} (A(x, y)) ,
\\
\|A(x, y) - [A ] _r(x, y) \|_{\lt(\Omega_1\times \Omega_2) } = &  \min_{   \text{ rank}(B) \le r  } \|  A (x, y) -B(x, y)\|_{\lt(\Omega_1\times \Omega_2) } = \sqrt{ \sum_{k=\R+1}^\infty \sigma_{k}^2 (A(x, y))}.
\end{align*}
\end{theorem}
\begin{proof} See  Section 6 of \cite{brezis2011functional}.
\end{proof} 
From \Cref{theorem:best low rank function}, it follows that
    if $B   (x, y)  $ is a function of rank at most $\R$, then 
    \begin{align}
       & \|A (x, y) - [A ] _r (x, y) \|_{\op}  \le \| A (x, y) - B (x, y) \| _{\op} 
        \\
         \label{remark:best low rank}
    & \|A (x, y) - [A ] _r (x, y) \|_{\lt(\Omega_1\times \Omega_2)}  \le \| A (x, y)  - B (x, y)  \| _{\lt(\Omega_1\times \Omega_2)}.
    \end{align}   
\begin{definition}
    Let $A (x, y) :\Omega_1 \times \Omega_2 \to \mathbb R $. The range  of $A (x, y)$ in the $x$-variable is defined as 
    \begin{align} \label{eq:definition of range 1}{\range}_x(A ) = \bigg \{ f(x): f(x)=  \int A (x, y) g(y)\mathrm{d}y \text{ for any }  g(y)\in \lt(\Omega_2) \bigg\}. 
 \end{align} 
\end{definition}
Suppose $ A  (x, y) = \sum_{\ri=1}^{\beta} \sigma_{\ri} (A )  \Phi_\ri  (x) \Psi_\ri(y)   $ as in \eqref{eq:svd Hilbert}.  Then   an equivalent definition of  ${\range}_x(A )$ is that 
\begin{align}\label{eq:definition of range}{\range}_x(A ) = \Span\{\Phi_\ri (x) \}_{\ri=1}^{\alpha} =\bigg\{ \sum_{\ri=1}^{\alpha } a_\ri \Phi_\ri (x) :\{ a_\ri\}_{\ri=1}^\alpha \subset \mathbb R\bigg\}.
\end{align} 
Consequently, the rank of $A(x, y)$ is the same as the dimensionality of $\range_x(A )$.

 \begin{corollary}
 \label{corollary:singular value preserving}
Let $x\in \Omega_1, y\in \Omega_2 $   be two variables in arbitrary dimensions,  and $A   (x, y) :\Omega_1 \times \Omega_2 \to \mathbb R $ be any function such that $\|A \|_{\lt(\Omega_1 \times \Omega_2)}   <\infty $. Let $\{ \phi_\mi\}_{\mi=1}^\infty  $ and $ \{ \psi_\li\}_{\li=1}^\infty  $  be two complete basis functions of $\lt(\Omega_1)$ and $\lt(\Omega_2)$ respectively. Let $ \C(A)  \in \mathbb R^{\infty \times \infty}$ be the coefficient matrix of $A$ in the sense that 
$$ \C(A)(\mi, \li) = \langle A, \phi_\mi\otimes \psi_{\li} \rangle. $$
{\bf (a)}  $\sigma_\ri( A ) =\sigma_\ri (\mathcal C(A))$ for all $\ri\in \mathbb Z^+$. Here $\sigma_\ri( A )  $ is the $\ri$-th singular value of $A$ when $A$ is   viewed as a function of two variables $x$ and $y$. 
\\
\\
{\bf (b)}   $ \mathcal C([A]_\R) = [\mathcal C( A ) ] _\R$. Here $[A]_\R$ is the best rank-$\R$ approximation of the function $A$; and $[\mathcal C(A)]_\R$ is  the best   rank-$\R$ approximation of the matrix $\mathcal C(A) $.
 \end{corollary} 
\begin{proof} \Cref{corollary:singular value preserving} is a direct  consequence of \Cref{theorem:best low rank function}. See  Section 6 of \cite{brezis2011functional} for   detailed proofs.
\end{proof}  
We will use   \Cref{corollary:singular value preserving} mainly in the setting where  $A$ is a finite rank function. In this case $\C(A)$ can be identify as a finite-dimensional matrix.

\subsection{Norms of functions}
Suppose $B (x, y): \Omega_1\times \Omega_2 $ is a two-variable function with $\Omega_1\subset \mathbb R^{d_1}$ and $\Omega_2\subset \mathbb R^{d_2}$. Suppose in addition that $\| B\|_{\lt(\Omega_1\times \Omega_2)}< \infty$. 
Then the operator norm of $B(x,y)$ can be defined as 
$$ \| B(x, y)\|_{\op} =  \sup_{ u \in \lt(\Omega_1) , v\in \lt(\Omega_2) }
 \langle B, u\otimes v\rangle .$$
It is a well known linear algebra result that 
$$ \| B(x, y)\|_{\op}  \le \| B\| _{\lt(\Omega_1\times \Omega_2)}  .$$
Whenever the context is clear, for convenience we will write 
$   \| B \|_{\op} $ 
to indicate the operator norm of $B(x, y)$.

\section{Proofs Related to \Cref{lemma:general projection deviation main}}
\label{section:proof of sketching}
 
\begin{proof}[Proof of \Cref{lemma:general projection deviation main}]

By \Cref{lemma:sobolev matrix Bernstein},
\begin{align}\label{eq:operator norm projected}\| ( \widehat A - A^* )\timesML \|_{\op} =\OP \bigg(  \sqrt{ \frac{ (\m^{d_1}+ \l^{d_2})\log(N) }{N}  }\bigg)  .
\end{align} 
Supposed this good event holds. 
Then
\begin{align} \nonumber 
&\|   \widehat A \timesML  - A^*\timesS  \|_{\op} 
\\\nonumber 
\le&  \| A^*\timesML  - A^*\timesS \|_{\op} + \| \widehat A \timesML - A^*\timesML \|_{\op}
\\ \nonumber
\le&  \| A^*\times_x \mathcal P_{\mathcal M }  - A^* \|_{\op } \| \mathcal P_{\mathcal S } \|_{\op} + \| \widehat A \timesML - A^*\timesML \|_{\op}
\\  \label{eq:sketched operator norm bound}
 \le&  \OP \bigg(  \m^{-\alpha } +  \sqrt{ \frac{ ( \m^{d_1}+\l^{d_2} ) \log(N) }{N}} \bigg) ,
\end{align}
where the last inequality follows from \Cref{theorem: approximation theory}, $ \| \mathcal P_{\mathcal S } \|_{\op} \le 1 $ and \eqref{eq:operator norm projected}.  
 
By definition, $\mathcal P^*_x$ is the projection operator onto $\range_x(A^*)$ By \Cref{lemma:general sketching consistency},  $\mathcal P^*_x$ is also the projection operator onto $\range_x(A^ * \times_y \mathcal P_\mathcal S)$. Therefore by 
 \Cref{theorem:Wedin Hilbert space}, the Wedin's theorem in Hilbert space, it follows that  
$$ \| \widehat {\mathcal P} _x - \mathcal P^*_x \|_{\op}  \le  
 \frac{   \sqrt 2    \|    \widehat A\timesML - A^*\timesS  \|_{\op}  }{ \sigma_\R(A^*\times_y \mathcal P_\mathcal S )   - \sigma_{\R+1} (A^*\times_y \mathcal P_\mathcal S )-\|    \widehat A\timesML - A^*\timesS  \|_{\op}   }    .$$
By \eqref{eq:sketched operator norm bound}, we have that
$\|   \widehat A \timesML  - A^*\timesS  \|_{\op}  =\OP\bigg(  \m^{-\alpha } +  \sqrt{ \frac{( \m^{d_1}+\l^{d_2}) \log(N)}{N}} \bigg)$. 
 In addition, since $A^* $ has rank $\R$, it follows that $A^*\times_y \mathcal P_S$ has at most rank $\R$. So $\sigma_{\R+1} (A^*\times_y \mathcal P_\mathcal S ) = 0$. Consequently,
\begin{align*}
\sigma_\R(A^*\times_y \mathcal P_\mathcal S )   - \sigma_{\R+1} (A^*\times_y \mathcal P_\mathcal S )-\|    \widehat A\timesML - A^*\timesS  \|_{\op}  
\\
\ge 
\sigma_\R(A^*)  /2-\|    \widehat A\timesML - A^*\timesS  \|_{\op}\ge \sigma_\R(A^*)  /4
\end{align*}
where  the first inequality follows from   \Cref{lemma:general sketching consistency}, and the last inequality follows from \eqref{eq:eigenvalue low bound} in \Cref{lemma:general projection deviation main} and \eqref{eq:sketched operator norm bound}. Therefore
\begin{align} 
\| \widehat {\mathcal P} _x - \mathcal P^*_x \|_{\op}^2  =\OP  \bigg\{  \frac{1}{\sigma_r^2}  \bigg(  \m^{-2\alpha } +   \frac{ (  \m^{d_1}+    \l^{d_2} ) \log(N) }{N}     \bigg)  \bigg\}   . 
\end{align}
Since $\widehat {\mathcal P}_x$ and $\mathcal P^*_x$ are both rank $\R$, it follows that $\widehat {\mathcal P} _x - \mathcal P^*_x $ is at most rank $2\R$. Therefore $ \| \widehat {\mathcal P} _x - \mathcal P^*_x \|_{\lt} \le 
\sqrt {2r}\| \widehat {\mathcal P} _x - \mathcal P^*_x \|_{\op}$. This  immediately leads to the desired result. 
\end{proof}

\begin{lemma} \label{lemma:general sketching consistency} Suppose all the conditions in \Cref{lemma:general projection deviation main}  hold.  Then
$$  {\range} _x( A^ * \times_y \mathcal P_\mathcal S)  = { \range } _x( A^*), $$
and that  $   \sigma_\R (A^* \timesL )  > \sigma_\R   (A^*)/2$.
\end{lemma}
 
 \begin{proof}[Proof of \Cref{lemma:general sketching consistency}]
By  \Cref{lemma:SVD pertubation for operators} in \Cref{appendix:Perturbation bounds} and \Cref{theorem: approximation theory},  the singular values
  $\{ \sigma_\ri (A^* \times_y \mathcal P_\mathcal S  ) \}_{\ri=1}^\infty $ of the  function  $A^* \times_y \mathcal P_\mathcal S (x, y) $  satisfies  
$$|  \sigma_\ri  (A^*)  -  \sigma_\ri (A^* \timesS  )    |\le 
\| A ^*- A^* \times_y \mathcal P_\mathcal S \| _{\lt(\Omega_1 \times \Omega_2) } =\bigO(\l^{-\alpha}  )  \text{ for all $1\le  \ri <\infty$}.  $$
As a result if
$ \sigma_\R  (A^*)  >C_L \l^{-\alpha}   $ for sufficiently large constant $C_L$, then 
\begin{align}\label{eq:lower bound of spectral value after sketching} \sigma_1(A^* \timesL ) \ge  \ldots  \ge  \sigma_\R (A^* \timesL )  \ge  \sigma_\R (A^*)  -\| A ^*- A^* \times_y \mathcal P_\mathcal L \| _{\lt(\Omega_1 \times \Omega_2) }   > \sigma_\R   (A^*)/2 . 
\end{align} 
Therefore, the leading  $\R$ singular values of $ A^* \timesL  $ is positive. So the rank of $A^*\timesL$ is at least $r$, and therefore the dimensionality of $\range_x(A^*\timesL)$ is at least $r$.

Since by construction,  $\range_x( A^* \timesL  ) \subset  \range_x(A^*)  $.  Since in \Cref{lemma:general projection deviation main} we assume that the dimensionality of $\range_x(A^*)$ is   $r$, it follows that     $ \range_x( A^* \timesL ) = \range_x(A^*)$.   
\end{proof}

\section{Proofs Related to \Cref{theorem:density bound main}}
 \label{section:proof of main result}
\begin{assumption} \label{assume: baised in projected space in operator norm}
Let $\{ \phi_{\mi}   \}_{\mi=1}^\infty  $ be a collection of basis functions  in $\lt(\mathcal O)$ generated from either   the reproducing kernel Hilbert spaces or the Legendre polynomials system.
\end{assumption} 
  For $j  \in \{ 1,\ldots,d\}$, let $\m \in \mathbb N  ,  \l_j \in  \mathbb N$, $z_j \in \mathbb R $  and denote 
 \begin{align} \label{eq:projection basis for tensor estimation}
 \mathcal M_j &=\text{span} \bigg\{\phi_{\mi }(z_j )   \bigg \}_{{\mi=1} }^{\m } \quad \text{and} \quad 
\\\label{eq:projection basis for tensor estimation 2}
\mathcal S_j  &= \text{span}\bigg\{  \phi _{\li_1 }(z_1)\cdots  \phi _{\li _{j-1} }(z_{j-1}) \phi _{\li _{j+1} }(z_{j+1}) \cdots  \phi _{\li_{d}  }(z_{d})  \bigg\}_{\li _1, \ldots,\li _{j-1}, \li_{j+1}, \ldots, \li_{d}=1}^{\l_j}.
\end{align} 

\begin{assumption} \label{assumption:regularity of density}
    Suppose that  the data $\{ \BZ_i\}_{i=1}^N\subset \mathcal O  ^{d}$ are  independently sampled from the density $A^* : \mathcal O   ^{d} \to \mathbb R ^+ $. Suppose in addition that   $\|A^*\|_{W^\alpha_2 (\mathcal O ^d)} < \infty  $ with $\alpha\ge 1$  and that  $\|A^*\|_\infty <\infty$.  
\end{assumption}
For $j=1,\ldots, d$, denote  
$y= (z_1,\ldots, z_{j-1}, z_{j+1}, \ldots, z_d)$, and 
\begin{align}\label{eq:singular values after reshape} \sigma_{j,\ri }^* = \sigma_{\ri} ( A^*(z_j, y ))   ,
\end{align}
where
 $ \sigma_{\ri} ( A^*(z_j, y ))  $  denote the $\ri$-th singular value of the two-variable function $A^* (z_j, y)$.

\subsection{Bias from the projection}\label{section:example of bias}

\begin{theorem} \label{theorem: approximation theory}
For $j\in \{ 1,\ldots, d\}$, let $\mathcal M_j$ and $\mathcal S_j$  be defined as in  \eqref{eq:projection basis for tensor estimation} and \eqref{eq:projection basis for tensor estimation 2} respectively.   Suppose that \Cref{assume: baised in projected space in operator norm} and \Cref{assumption:regularity of density} hold. Denote $x=z_j$ and $y=( z_1,\ldots, z_{j-1}, z_{j+1},\ldots, z_d ).  $ Then
\begin{align}
\label{eq:general bias 2}
  &\|A^* - A^* \times_y \mathcal P _{\mathcal S_j  } \|^2_{\lt(\mathcal O^d) }     =   \bigO (  \l^{-2\alpha }  )  
,\quad \text{and} 
\\
\label{eq:general bias 1}
&\|A ^* - A^* \times_x \mathcal P _{\mathcal M_j}  \times_y \mathcal P _{ \mathcal S_j } \|^2_{\lt( \mathcal O^d ) }    = \bigO  ( \m^{-2\alpha}+  \l^{-2\alpha }) 
\end{align}
\end{theorem}
\begin{proof}
The proof of  \Cref{theorem: approximation theory} is detailed  in \Cref{section: justify approximation theory}.
\end{proof}

\begin{lemma} \label{lemma:singular gap of population coefficient matrix} Let  $x=z_j$ and  $y= (z_1,\ldots,z_{j-1}, z_{j+1}, \ldots, z_{d} )$.   Then there exists an absolute constant $C$ such that 
$$ \sigma_\R ( \mathcal C ( A^* \times _x  \mathcal P_{\mathcal M_j }  \times_y \mathcal P_{\mathcal M^{\otimes d-1}} )  ) \ge \sigma^*_{j,\R}  -Cm^{-\alpha}  .$$ 
 Here, $  \sigma_{j,\ri }^*$  is defined in \eqref{eq:singular values after reshape} .
 \end{lemma}
\begin{proof} By symmetry, it suffices to consider 
$x=z_1$ and $y= ( z_{2}, \ldots, z_{d} )$. For brevity, denote
$$\mathcal M =\mathcal M_1 \quad \text{and} \quad \mathcal M^{\otimes d-1} = \mathcal M_2 \otimes \cdots \otimes \mathcal M_d.$$
Note that $  \mathcal C ( A^* \times _x \mathcal P_{\mathcal M }  \times_y \mathcal P_{\mathcal M^{\otimes d-1}} ) \in \mathbb R^{m\times m^{d-1}}$.  By \Cref{corollary:singular value preserving}, $ \sigma_\R ( \mathcal C ( A^* \times _x \mathcal P_{\mathcal M }  \times_y \mathcal P_{\mathcal M^{\otimes d-1}} )  ) =\sigma_\R ( A^* \times _x \mathcal P_{\mathcal M }  \times_y \mathcal P_{\mathcal M^{\otimes d-1}} ).$ 
Therefore  
\begin{align*} | \sigma_\R ( A^* \times _x \mathcal P_{\mathcal M }  \times_y \mathcal P_{\mathcal M^{\otimes d-1}} ) - \sigma_r (A^*(x, y))  |  \le   \|  A^* \times_x \mathcal P _{\mathcal M}  \times_y \mathcal P _{ \mathcal M^{\otimes d-1}}- A^* \|^2_{\op}
\\
\le   \|  A^* \times_x \mathcal P _{\mathcal M}  \times_y \mathcal P _{ \mathcal M^{\otimes d-1} }- A^* \|^2_{\lt(\mathcal O^d )} = \bigO(m^{-2\alpha }),
\end{align*}
where the  first inequality follows from \Cref{lemma:SVD pertubation for operators}, and the  equality follows from  \eqref{eq:general bias 1} with $\mathcal S_j =\mathcal M^{\otimes d-1}$. The desired result follows from the  definition in \eqref{eq:singular values after reshape} that 
$ \sigma_r (A^*(x, y)) = \sigma_{1,\R}^*$.
\end{proof}

\subsection{Consistency of VRS }

For analysis convenient, we can rewrite our main algorithm into  \Cref{algorithm:analysis version of general svd version 2}.
\begin{algorithm}[h!]
\begin{algorithmic} 
	\INPUT Data $\{\BZ_i\}_{i=1}^N$, parameters $\{ {\R}_j \}_{j=1}^d\subset  \mathbb Z^+$, subspaces $\mathcal M_j$ as in \eqref{eq:projection basis for tensor estimation} and subspaces $ \{ \mathcal S_j\} _{j=1}^d $ as in \eqref{eq:projection basis for tensor estimation 2}. 
 \State Set empirical measures  $\widehat   A  =\frac{1}{N/2 }\sum_{i=1}^{N/2} \delta_{\BZ_i}   $, and   $\widehat   A'  =\frac{1}{N/2 }\sum_{i=N/2 +1}^{N } \delta_{\BZ_i}   $ 	 
 	 \State [{\bf  Subspaces estimation via Sketching}]
 	\For {$ j  \in \{ 1,  \ldots,   d  \}$} 
    \State Set   $y=(z_1,\ldots, z_{j-1}, z_{j+1},\ldots, z_d)$.
 Compute 
     \begin{align}\label{eq:algorithm 2 first svd}  \mathcal C (\widehat A\times_j\mathcal P_{\mathcal M_j}\times _y \mathcal P_{\mathcal S_j})  \in \mathbb R^{m\times \ell_j^{d-1}}.
     \end{align}
     \State Use SVD to compute  $\widehat U_j \in \mathbb O^{m\times r_j}$, where columns of $ \widehat U_j$ are the leading $r_j$ left  singular vectors of $ \mathcal C (\widehat A\times_j\mathcal P_{\mathcal M_j}\times _y \mathcal P_{\mathcal S_j}) $. Set $    \mathcal P  _{ j} ^U = \widehat U_j \widehat U_j^\top  \in \mathbb R^{m\times m} $.
 	\EndFor
  \vskip 0.5cm
  \State [{\bf  Sketching using  estimated subspaces}]
   \State Set
   $  \widehat {\mathcal C} = \mathcal C( \widehat A\times _1 \mathcal P_{\mathcal M_1} \times \cdots  \times _d \mathcal P_{\mathcal M_d}) \in \mathbb R^{m^d} $.
    \For {$ j  \in \{ 1,  \ldots,   d  \}$}  
 
     \State  Compute 
     $$ \widehat {\mathcal B}_j=  \widehat {\mathcal C}\times_1   {\mathcal P}_1 ^U \times \cdots \times_{j-1}   {\mathcal P}_{j-1}^U  \times_{j+1}   {\mathcal P}_{j+1} ^U \times \cdots \times_d     {\mathcal P}_d^U \in \mathbb R^{m^d}.$$
     \State Use SVD to compute  $\widehat V_j \in \mathbb O^{m\times r_j}$, where  the  columns of $\widehat V_j$ are the leading $r_j$ left  singular vectors of the matrix $\widehat {\mathcal B}_j (\mi_j; \mi_1, \ldots, \mi_{j-1},\mi_{j+1},\ldots,\mi_d) \in \mathbb R^{m\times m^{d-1}}$. Set $    \mathcal P  _{ j} ^V = \widehat V_j \widehat V_j^\top  \in \mathbb R^{m\times m} $.
     \EndFor
     \State  Set \begin{align} \label{eq:algorithm 2 second coefficient matrix}\widehat {\mathcal C}' = \mathcal C( \widehat A '\times _1 \mathcal P_{\mathcal M_1} \times \cdots  \times _d \mathcal P_{\mathcal M_d}) \in \mathbb R^{m^d}. \end{align}  Compute the tensor 
     $$\widetilde  {\mathcal C} = \widehat {\mathcal C}'\times_1   {\mathcal P}_1 ^V \times \cdots   \times_d     {\mathcal P}_d^V \in \mathbb R^{m^d}  . $$
   \State  Compute the estimated   density function 
    $$   \widetilde  A (z_1,\ldots,z_d) =  \sum_{\mi_1=1}^{{m}_1}\cdots \sum_{\mi_d=1}^{{m}_d} \widetilde {\mathcal  C}_{\mi _1,\ldots,\mi _d}  \phi _{1,\ri_1}(z_1)\cdots  \phi _{d,\ri_d}(z_d) .$$  
	\OUTPUT   The  estimated  density function $\widetilde A(z_1,\ldots,z_d)$

 \caption{Multivariable Density  Estimation via Variance-Reduced Sketching}
\label{algorithm:analysis version of general svd version 2}
\end{algorithmic}
\end{algorithm}   

\begin{remark} We show that    \Cref{algorithm:analysis version of general svd version 1}   and \Cref{algorithm:analysis version of general svd version 2} are numerically identical. In other words, \Cref{algorithm:analysis version of general svd version 1} and \Cref{algorithm:analysis version of general svd version 2}  are    different interpretations of the same algorithm. 
\\
\\
Note that   $\mathcal P^{\Phi}_j (z_j, z_j')   $ corresponds to the 
        leading $\R_j$ singular functions in the variable $z_j$ of     $     \widehat A\times_{z_j} \mathcal P_{\mathcal M_j}\times _y \mathcal P_{\mathcal S_j} (z_j, y)  . $   In addition,  $\mathcal P  _{ j} ^U   $ corresponds to the  leading $\R_j$ left  singular vectors of $ \mathcal C (\widehat A\times_{z_j}\mathcal P_{\mathcal M_j}\times _y \mathcal P_{\mathcal S_j}) $.  Therefore  $\mathcal P  _{ j} ^U   $ is the coefficient matrix of $\mathcal P^{\Phi}_j (z_j, z_j')$ with respect to the basis $\{\phi_{\mi}\}_{\mi=1}^{m}$. In other words 
       \begin{align}\label{eq:equivalence of two algorithms 1} \C(\mathcal P^{\Phi}_j  ) = \mathcal P  _{ j} ^U. 
       \end{align}
 \\
 Observe that 
  the range of $ \mathcal P_{\mathcal M_j} $  contains $ \mathcal P^{\Phi}_{j} $, so $$\mathcal P _{  \M_j}   \mathcal P^{\Phi}_{j}    = \mathcal P^{\Phi}_{j} .  $$
Note that $\mathcal P^{\Psi}_j (z_j , z_j')   $
      corresponds to  the leading $\R_j$ singular functions in the variable $z_j$ of     \begin{align}
         \nonumber &  \widehat A\times_{ j}\mathcal P_{\mathcal M_j}\times _1   \mathcal P_1^{\Phi} \times \cdots \times_{j-1}  \mathcal P_{j-1} ^{\Phi}\times_{j+1}  \mathcal P_{j+1} ^{\Phi} \cdots \times_d  \mathcal P_{d} ^{\Phi} 
         \\ \nonumber 
         = &   \widehat A\times_{ j}\mathcal P_{\mathcal M_j}\times _1  \mathcal P _{  \M_1}  \mathcal P_1^{\Phi} \times \cdots  \times_{j-1}  \mathcal P _{  \M_{j-1}}  \mathcal P_{j-1} ^{\Phi}\times_{j+1} \mathcal P _{  \M_{j+1}}  \mathcal P_{j+1} ^{\Phi} \cdots \times_d  \mathcal P _{  \M_{d}}  \mathcal P_{d} ^{\Phi} 
         \\ \label{eq:equivalence of two algorithms 2}
         = & \big(  \widehat A\times_{ 1}\mathcal P_{\mathcal M_1}\times  \cdots \times_{ d}\mathcal P_{\mathcal M_d}     \big)  \times _1   \mathcal P_1^{\Phi} \times \cdots \times_{j-1}  \mathcal P_{j-1} ^{\Phi}\times_{j+1}  \mathcal P_{j+1} ^{\Phi} \cdots \times_d  \mathcal P_{d} ^{\Phi}   .
     \end{align}   
     As a result,  by \eqref{eq:equivalence of two algorithms 1} and \eqref{eq:equivalence of two algorithms 2},
    \begin{align}
         \nonumber &  \mathcal C( \widehat A\times_{ j}\mathcal P_{\mathcal M_j}\times _1   \mathcal P_1^{\Phi} \times \cdots \times_{j-1}  \mathcal P_{j-1} ^{\Phi}\times_{j+1}  \mathcal P_{j+1} ^{\Phi} \cdots \times_d  \mathcal P_{d} ^{\Phi} )
         \\
         = & \nonumber
         \mathcal C \big(  \widehat A\times_{ 1}\mathcal P_{\mathcal M_1}\times  \cdots \times_{ d}\mathcal P_{\mathcal M_d}     \big)  
         \times _1   \mathcal C  ( \mathcal P_1^{\Phi} )  \times \cdots \times_{j-1}    \mathcal C  (  \mathcal P_{j-1} ^{\Phi} ) \times_{j+1}    \mathcal C  ( 
        \mathcal P_{j+1} ^{\Phi} )  \cdots \times_d   \mathcal C  (   \mathcal P_{d} ^{\Phi}  )
         \\ \nonumber 
         = & \mathcal C  (  \widehat A\times_{ 1}\mathcal P_{\mathcal M_1}\times  \cdots \times_{ d}\mathcal P_{\mathcal M_d}      )  \times _1     \mathcal P  _{ 1} ^U  \times \cdots \times_{j-1}  \mathcal P  _{ j-1} ^U \times_{j+1} \mathcal P  _{ j+1} ^U \cdots \times_d  \mathcal P  _{ d} ^U  .
     \end{align}     
     So $\widehat {\mathcal B}_j$ is the coefficient matrix of $\widehat B_j$, where 
     $\widehat B_j$ is defined in \Cref{algorithm:analysis version of general svd version 1}
     and  $\widehat {\mathcal B}_j$ is defined in  \Cref{algorithm:analysis version of general svd version 2}.
      Consequently,    $\mathcal P  _{ j} ^V   $ is the coefficient matrix of $\mathcal P^{\Psi}_j (z_j, z_j')$ with respect to the basis $\{\phi_{\mi}\}_{\mi=1}^{m}$. In other words 
        $$ \C(\mathcal P^{\Psi}_j  ) = \mathcal P  _{ j} ^V. $$
        Therefore $\widetilde \C$ in \Cref{algorithm:analysis version of general svd version 2} is the coefficient tensor of $\widetilde A$ in \Cref{algorithm:analysis version of general svd version 1}, and thus \Cref{algorithm:analysis version of general svd version 1}   and \Cref{algorithm:analysis version of general svd version 2} are numerically identical
\end{remark}

\begin{definition}[Range and rank of multivariable functions] \label{definition:function range and rank}
    Let $B(z_1, \ldots, z_d) \in \lt(\mathcal O^d)$. For $j\in \{ 1,\ldots,d\} $, let $y=(z_1,\ldots, z_{j-1}, z_{j+1},\ldots, z_d)$. Define
    $$ {\range}_j(B) = \bigg \{ \int_{O^{d-1}} B(z_j, y )g(y) dy: \text{  $g(y)$ is any function in $\lt(\mathcal O^{d-1})$}   \bigg\}. $$
    Therefore ${\range}_j(B) $ is a linear subspace in $\lt(\mathcal O )$.
    In addition, define $\rank_j(B)  $   the dimensionality of $\range_j(B) $.  
\end{definition}
Denote 
\begin{align}\label{eq:definition of low rank bias} \xiapp = \min_{B\in \lt(\mathcal O^d): {\rank}_j(B) \le \R_j}\{ \| A^* - B \|_{\lt(\mathcal O^d)}  \}.
\end{align}
In other words, $\xiapp  $ is the error of the best  rank-$(r_1,\ldots, r_d)$ approximation of the density $A^* $ in $\lt(\mathcal O^d)$.

 \begin{proof}[Proof of \Cref{theorem:density bound main}]
    Let $\widetilde  {\mathcal C} = \widehat {\mathcal C}'\times_1   {\mathcal P}_1 ^V \times \cdots   \times_d     {\mathcal P}_d^V \in \mathbb R^{m^d}   $ be as in \Cref{algorithm:analysis version of general svd version 2}. It follows that 
    $$\C(\widetilde A) = \widetilde  \C, $$
    where $\C(\widetilde A)$ denotes the coefficient tensor of $\widetilde A$ corresponding to the basis functions $\{\phi_\mi\}_{\mi=1}^\infty$ as in \Cref{definition:coefficient tensor}. Then 
    \begin{align} \nonumber 
         \|  \widetilde A  - A  ^* \|  _{\lt(\mathcal O  ^d)}  
         \le &\|\widetilde A -A^* {\times_1\mathcal P_{\mathcal M_1} \times \cdots\times_d\mathcal P_{\mathcal M_d}  }    \|_{\lt(\mathcal O  ^d)}  + \| A^* {\times_1\mathcal P_{\mathcal M_1} \times \cdots\times_d\mathcal P_{\mathcal M_d}  }   -A^*\| _{\lt(\mathcal O  ^d)} .
    \end{align}
    Note that with high probability, 
    \begin{align*}
        \|\widetilde A -A^* {\times_1\mathcal P_{\mathcal M_1} \times \cdots\times_d\mathcal P_{\mathcal M_d}  }    \|_{\lt(\mathcal O  ^d)} =  \| \widetilde \C - \C (A^* {\times_1\mathcal P_{\mathcal M_1} \times \cdots\times_d\mathcal P_{\mathcal M_d}  }) \|_F 
       \\
       \le   C _1\bigg( \sqrt{ \frac{m\R_1\R_2\cdots \R_d}{N}}  +  \xiapp + m^{-\alpha}\bigg),
    \end{align*}
    where the equality follows from \Cref{lemma:coefficient tensor distance preserving}, and the inequality follows from
    \Cref{lemma:variance of tensor density}. 
    In addition  by \Cref{theorem: approximation theory}, there exists an absolute constant $C_2$ such that 
    $$ \| A^* {\times_1\mathcal P_{\mathcal M_1} \times \cdots\times_d\mathcal P_{\mathcal M_d}  }   -A^*\| _{\lt(\mathcal O  ^d)}  \le C_2 m^{-\alpha} .$$
    Therefore, with high probability, 
    \begin{align} \nonumber 
         \|  \widetilde A  - A  ^* \|  _{\lt(\mathcal O  ^d)}  
         \le C _3\bigg( \sqrt{ \frac{m\R_1\R_2\cdots \R_d}{N}}  +  \xiapp + m^{-\alpha}\bigg).
         \end{align}
         The desired result follows from the fact that 
         $ \m \asymp  N^{1/(2\alpha+1)} .$
 \end{proof}

 \begin{lemma} \label{lemma:variance of tensor density}
     Suppose all  the conditions as in \Cref{theorem:density bound main} hold.  Let 
     \begin{align*}  \widehat {\mathcal C}' = \mathcal C( \widehat A '\times _1 \mathcal P_{\mathcal M_1} \times \cdots  \times _d \mathcal P_{\mathcal M_d}) \in \mathbb R^{m^d}  \end{align*}   as in \eqref{eq:algorithm 2 second coefficient matrix}, and  $\widetilde  {\mathcal C} = \widehat {\mathcal C}'\times_1   {\mathcal P}_1 ^V \times \cdots   \times_d     {\mathcal P}_d^V \in \mathbb R^{m^d}   $ as in \Cref{algorithm:analysis version of general svd version 2}. Then with high probability, 
     $$ \|  \widetilde \C - \mathcal C( A^* {\times_1\mathcal P_{\mathcal M_1} \times \cdots\times_d\mathcal P_{\mathcal M_d}  } )\|_F \le C \bigg( \sqrt{ \frac{m\R_1\R_2\cdots \R_d}{N}}  +  \xiapp + m^{-\alpha}\bigg) , $$
     Here $C$ is some absolute constant independent of $N, m$, and $\{\r_j\}_{j=1}^d$, and $ \mathcal C(A^*)$ is the coefficient tensor of $A^*$ corresponding to the basis $\{ \phi_k\}_{k=1}^m$. The definition of coefficient tensor can be found  in \Cref{definition:coefficient tensor}.
 \end{lemma}
\begin{proof}
    For convenience, denote $\mathcal C^* = \mathcal C(A^*{\times_1\mathcal P_{\mathcal M_1} \times \cdots\times_d\mathcal P_{\mathcal M_d}  }) \in \mathbb R^{m^d}$.
    Note that 
    \begin{align} \nonumber 
       & \| \widetilde \C -  \mathcal C^* \|_F =  \| \widehat {\mathcal C}'\times_1   {\mathcal P}_1 ^V \times \cdots   \times_d     {\mathcal P}_d^V -\mathcal C^* \|_F 
        \\ \label{eq:coefficient tensor deviation bound one}
        \le & \| \mathcal C^* - \C^*\times_1   {\mathcal P}_1 ^V \times \cdots   \times_d     {\mathcal P}_d^V \| _F +
        \| (\C^*-\widehat \C ') \times_1   {\mathcal P}_1 ^V \times \cdots   \times_d     {\mathcal P}_d^V  \| _F .
    \end{align}
    \
    \\
We bound $\| \mathcal C^* - \C^*\times_1   {\mathcal P}_1 ^V \times \cdots   \times_d     {\mathcal P}_d^V \| _F$ in {\bf Step 1}  to {\bf Step 4}, and bound  $  \| (\C^*-\widehat \C') \times_1   {\mathcal P}_1 ^V \times \cdots   \times_d     {\mathcal P}_d^V  \| _F $ in {\bf Step 5}.
\\
\\
{\bf Step 1.}
Let columns of  $ \widehat V_j \in \mathbb O^{m\times r_j} $   correspond  to  the leading $\R_j$ left  singular vectors of the matrix $\widehat {\mathcal B}_j (\mi_j; \mi_1, \ldots, \mi_{j-1},\mi_{j+1},\ldots,\mi_d) \in \mathbb R^{m\times m^{d-1}}$, and $ \widehat V_{j \perp} \in \mathbb O^{m\times (m-r_j)}$ correspond  to    the rest  $m-\R_j$  singular vectors of the same matrix. It follows that   
$$\widehat V_{j\perp}    \widehat V _{j\perp} ^\top  =  I_m -  \widehat V_{j }    \widehat V _{j } ^\top  =  I_m -  \mathcal P^V_j      .$$
Therefore 
\begin{align} \nonumber 
     & \| \mathcal C^* - \C^*\times_1   {\mathcal P}_1 ^V \times \cdots   \times_d     {\mathcal P}_d^V \| _F 
     \\ \nonumber 
     = & \|  \C^* \times_1 \widehat V_{1\perp}    \widehat V _{1\perp} ^ \top  + \C^* \times_1   {\mathcal P}_1 ^V   \times_2  \widehat V_{2\perp}    \widehat V _{2\perp} ^\top   + \cdots + \C^* \times_1   {\mathcal P}_1 ^V   \times\cdots \times_{d-1}  {\mathcal P}_{d-1} ^V   \times_d  \widehat V_{d\perp}    \widehat V _{d\perp} ^\top \|_F 
     \\ \nonumber 
     \le & \|  \C^* \times_1 \widehat V_{1\perp}    \widehat V _{1\perp}^\top \|_F    + \| \C^* \times_1   {\mathcal P}_1 ^V   \times_2  \widehat V_{2\perp}    \widehat V _{2\perp}  ^\top \| _F  + \cdots + \|  \C^* {\times_1   {\mathcal P}_1 ^V   \times\cdots \times _{d-1}  {\mathcal P}_{d-1} ^V    \times_d  \widehat V_{d\perp}    \widehat V _{d\perp}^\top  } \|_F  
     \\ \nonumber 
     \le & \|  \C^* \times_1 \widehat V_{1\perp}    \widehat V _{1\perp}^\top  \|_F    + \| \C^*    \times_2  \widehat V_{2\perp}    \widehat V _{2\perp} ^\top   \| _F  \|     {\mathcal P}_1 ^V\|_{\op} + \cdots + \|  \C^*  \times_d  \widehat V_{d\perp}    \widehat V _{d\perp} ^\top   \|_F   \|   {\mathcal P}_1 ^V   \|_{\op } \cdots  \|{\mathcal P}_{d-1} ^V  \| _{\op}  
     \\ \label{eq:tensor estimation term one bound one}
     \le & \sum_{j=1}^d \|  \C^* \times_j \widehat V_{j\perp}    \widehat V _{j\perp}^\top \|_F  ,
\end{align}
    where the last inequality follows from the fact that projection matrices have operator norm 1.
    \\
    \\
    {\bf Step 2.} We proceed to bound $ \|  \C^* \times_1 \widehat V_{1\perp}    \widehat V _{1\perp} ^\top  \|_F$, as $ \|  \C^* \times_1 \widehat V_{j\perp}    \widehat V _{j\perp} ^\top \|_F$  can be bounded by the exact same calculations. 
    \\
    For $j\in \{1,\ldots,d \}$, let  $  U^* _j \in \mathbb R^{m\times \R_j}$ denote the  matrix  with columns being the leading $\R_j$ singular vectors of $\C^*(\mi_j;\mi_1,\ldots, \mi_{j-1}, \mi_{j+1},\ldots, \mi_d)  \in \mathbb R^{m\times m^{d-1}}$. 
    Let $\widehat U_j \in \mathbb R^{m\times \R_j} $ be the the matrix  with  columns being the leading $\R_j$ singular vectors of  $\mathcal C (\widehat A\times_j \mathcal P_{\mathcal M_j}\times _y \mathcal P_{\mathcal S_j})   $. Here $\mathcal C (\widehat A\times_j \mathcal P_{\mathcal M_j}\times _y \mathcal P_{\mathcal S_j})   $ is defined in 
    \eqref{eq:algorithm 2 first svd}.
    Then 
    \begin{align} \nonumber 
        &\|  \C^* \times_1 \widehat V_{1\perp}    \widehat V _{1\perp}^\top  \|_F 
        \\
        \le & \label{eq:tensor main result term 12}
        \|  \C^* \times_1 \widehat V_{1\perp} 
        \widehat V _{1\perp} ^\top  \times_2   U^*_2  U^{*\top} _2    \|_F
        \\
        + & \label{eq:tensor main result term 12 complement} \|  \C^* \times_1 \widehat V_{1\perp}   
        \widehat V _{1\perp} ^\top \times_2 ( I_m- U^*_2  U^{*\top} _2 )  \|_F.
    \end{align}
    Note that 
   \begin{align} \label{eq:tensor main result term 12 complement bound}
       \eqref{eq:tensor main result term 12 complement}
        \le \|  \C^*  {\times_2 ( I_m- U^*_2  U^{*\top} _2 ) } \|_F 
        \|\widehat V_{1\perp}   
        \widehat V _{1\perp} ^\top \| _{\op}
          \le 
        \xiapp +Cm^{-\alpha},
   \end{align}
   where the last inequality follows from  \Cref{lemma:mirky projection bound 1} and the fact that $\|\widehat V_{1\perp} 
        \widehat V _{1\perp} ^\top \| _{\op}=1$. 
\\
In addition, by applying 
  \Cref{eq:subspace consistency to matrix consistency} to the matrix $\C^* \times_1 \widehat V_{1\perp} 
        \widehat V _{1\perp}  ^\top \times_2   U^*_2  U^{*\top} _2   (\mi_2;\mi_1,\mi_3,\ldots, \mi_d)$, it follows that 
\begin{align*}
   \eqref{eq:tensor main result term 12}
        \le \sigma_{\min}^{-1}( \widehat U_2 ^\top  U^{*  }_2  ) \| \C^* \times_1 \widehat V_{1\perp} 
        \widehat V _{1\perp}^\top  \times_2   U^*_2  U^{*\top} _2  \widehat U _2  \widehat U^{\top} _2   \| _F
\end{align*}  
By  \Cref{lemma:consistency of subspace in step 1}, we have that 
  $ \sigma_{\min}^{ -1}( \widehat U_2 ^\top  U^{*  }_2   ) =\bigO_\p(1).$
Consequently, there exists a constant $C_1$ such that  with high probability,
\begin{align} \nonumber 
     \eqref{eq:tensor main result term 12} \le& C_1\| \C^* \times_1 \widehat V_{1\perp} 
        \widehat V _{1\perp} ^\top \times_2   U^*_2  U^{*\top} _2  \widehat U _2  \widehat U^{\top} _2   \| _F 
      \\ \nonumber 
      \le & C_1  \|\C^* \times_1 \widehat V_{1\perp} 
        \widehat V _{1\perp} ^\top \times_2    \widehat U _2  \widehat U^{\top} _2   \|_F + C_1\| \C^* \times_1 \widehat V_{1\perp} 
        \widehat V _{1\perp} ^\top  \times_2   (I_m-U^*_2  U^{*\top} _2)  \widehat U _2  \widehat U^{\top} _2  \|_F 
        \\ \nonumber 
        \le & C_1 \|\C^* \times_1 \widehat V_{1\perp} 
        \widehat V _{1\perp}^\top  \times_2    \widehat U _2  \widehat U^{\top} _2   \|_F  + C_1\| \C^* \times_1 \widehat V_{1\perp} 
        \widehat V _{1\perp} ^\top  \times_2   (I_m-U^*_2  U^{*\top} _2)   \|_F \|\widehat U _2  \widehat U^{\top} _2 \|_{\op}
        \\  \label{eq:tensor main result term 12 bound}
        \le &  C_1 \|\C^* \times_1 \widehat V_{1\perp} 
        \widehat V _{1\perp} ^\top \times_2    \widehat U _2  \widehat U^{\top} _2   \|_F  +C_1'( \xiapp + m^{-\alpha}),
\end{align}
   where the last inequality follows from \eqref{eq:tensor main result term 12 complement bound} and $\| \widehat U _2  \widehat U^{\top} _2\|_{\op} =1$.
   As a result
   \begin{align}\label{eq:eq:tensor main result term 1 bound}
       \|  \C^* \times_1 \widehat V_{1\perp}    \widehat V _{1\perp} ^\top \|_F  \le C_1 \|\C^* \times_1 \widehat V_{1\perp} 
        \widehat V _{1\perp} ^\top  \times_2    \widehat U _2  \widehat U^{\top} _2   \|_F  +C_1'( \xiapp + m^{-\alpha}).
   \end{align}
 {\bf Step 3.} We bound $ \|\C^* \times_1 \widehat V_{1\perp} 
        \widehat V _{1\perp} ^\top   \times_2    \widehat U _2  \widehat U^{\top} _2   \|_F $ in this step. Observe that
  \begin{align} \nonumber 
        &\|  \C^* \times_1 \widehat V_{1\perp}    \widehat V _{1\perp} ^\top \times_2    \widehat U _2  \widehat U^{\top} _2   \|_F 
        \\
        \le & \label{eq:tensor main result term 22}
        \|  \C^* \times_1 \widehat V_{1\perp} 
        \widehat V _{1\perp} ^\top  \times_2    \widehat U _2  \widehat U^{\top} _2   \times_3   U^*_3  U^{*\top} _3    \|_F
        \\
        + & \label{eq:tensor main result term 22 complement} \|  \C^* \times_1 \widehat V_{1\perp}   
        \widehat V _{1\perp} ^\top \times_2    \widehat U _2  \widehat U^{\top} _2   \times_3 ( I_m- U^*_3  U^{*\top} _3 )  \|_F.
    \end{align}
Note that 
   \begin{align} \label{eq:tensor main result term 22 complement bound}
       \eqref{eq:tensor main result term 22 complement}
        \le \|  \C^*  {\times_3 ( I_m- U^*_3  U^{*\top} _3 ) } \|_F \|\widehat V_{1\perp}   
        \widehat V _{1\perp} ^\top  \| _{\op} \|    \widehat U _2  \widehat U^{\top} _2\| _{\op}
          \le 
        \xiapp +Cm^{-\alpha},
   \end{align}
   where the last inequality follows from \Cref{lemma:mirky projection bound 1}.
   \\
In addition, by applying 
  \Cref{eq:subspace consistency to matrix consistency} to  $\C^* \times_1 \widehat V_{1\perp} 
        \widehat V _{1\perp}^\top  \times_2    \widehat U _2  \widehat U^{\top} _2 \times_3 U^*_3  U^{*\top} _3   (\mi_3;\mi_1,\mi_2,\ldots, \mi_d)$, it follows that
\begin{align*}
   \eqref{eq:tensor main result term 22}
        \le \sigma_{\min}^{-1}( \widehat U_3 ^\top  U^{*  }_3  ) \| \C^* \times_1 \widehat V_{1\perp} 
        \widehat V _{1\perp}  \times_2     \widehat U _2  \widehat U^{\top} _2 \times_3 U^*_3  U^{*\top} _3 \widehat U _3  \widehat U^{\top} _3 \| _F
\end{align*}  
 By  \Cref{lemma:consistency of subspace in step 1}, we have that 
  $ \sigma_{\min}^{ -1}(   \widehat U_3 ^\top  U^{*  }_3  ) =\bigO_\p(1).$
Consequently, similar to \eqref{eq:tensor main result term 12 bound}, there exists a constant $C_2$ such that  with high probability,
\begin{align} 
     \eqref{eq:tensor main result term 22}    
        \le &  C_2 \|\C^* \times_1 \widehat V_{1\perp} 
        \widehat V _{1\perp} ^\top  \times_2    \widehat U _2  \widehat U^{\top} _2 \times_3    \widehat U _3  \widehat U^{\top} _3   \|_F  +C_2( \xiapp + m^{-\alpha}).
        \label{eq:tensor main result term 22 bound}
\end{align}
   Combining \eqref{eq:tensor main result term 22 complement bound}, \eqref{eq:tensor main result term 22 bound}, and   \eqref{eq:eq:tensor main result term 1 bound}, it follows that 
   \begin{align}\label{eq:eq:tensor main result term 2 bound}
       \|  \C^* \times_1 \widehat V_{1\perp}    \widehat V _{1\perp} ^\top \|_F  \le C_2' \|\C^* \times_1 \widehat V_{1\perp} 
        \widehat V _{1\perp} ^\top  \times_2    \widehat U _2  \widehat U^{\top} _2  \times_3    \widehat U _3  \widehat U^{\top} _3   \|_F  +C_2'( \xiapp + m^{-\alpha}).
   \end{align}
{\bf Step 4.} Following the same argument that leads to \eqref{eq:eq:tensor main result term 2 bound}, 
there exists a constant $C '$ such that
 \begin{align}  \label{eq:projection term one bound two}
       \|  \C^* \times_1 \widehat V_{1\perp}    \widehat V _{1\perp}^ \top  \|_F  \le C  ' \|\C^* \times_1 \widehat V_{1\perp} 
        \widehat V _{1\perp} ^\top  \times_2    \widehat U _2  \widehat U^{\top} _2  \times \cdots \times_d    \widehat U _d  \widehat U^{\top} _d   \|_F  +C '( \xiapp + m^{-\alpha}).
   \end{align}
It suffices to use \Cref{eq:projection error for approximately low rank matrix} to control $\|\C^* \times_1 \widehat V_{1\perp} 
        \widehat V _{1\perp}^\top  \times_2    \widehat U _2  \widehat U^{\top} _2  \times \cdots \times_d    \widehat U _d  \widehat U^{\top} _d   \|_F$. 
        Note that $\widehat V_1$ corresponds to the  leading $\R_1$ left singular vectors 
 of    the matrix $\widehat {\mathcal B}_1 (\mi_1; \mi_2, \ldots,  \mi_d) \in \mathbb R^{m\times m^{d-1}}$, where
     $  \widehat {\mathcal B}_1=  \widehat {\mathcal C}\times_2   \widehat U _2  \widehat U^{\top} _2  \times \cdots \times_d    \widehat U _d  \widehat U^{\top} _d   \in \mathbb R^{m^d}.$ 
     So 
     \begin{align} \nonumber 
         &\|\C^* \times_1 \widehat V_{1\perp} 
        \widehat V _{1\perp} ^\top  \times_2    \widehat U _2  \widehat U^{\top} _2  \times \cdots \times_d    \widehat U _d  \widehat U^{\top} _d   \|_F 
        \\ \nonumber 
        = &   \|\C^* \times_1 \widehat V_{1\perp} 
        \widehat V _{1\perp} ^\top  \times_2    \widehat U _2  \widehat U^{\top} _2  \times \cdots \times_d    \widehat U _d  \widehat U^{\top} _d   (\mi_1; \mi_2, \ldots,\mi_d) \|_F  
        \\ \label{eq:projection error coordinate 1 term 1}
        \le &\|(\C^* -\widehat \C)  \times_2    \widehat U _2  \widehat U^{\top} _2  \times \cdots \times_d    \widehat U _d  \widehat U^{\top} _d   (\mi_1; \mi_2, \ldots,\mi_d) \|_F  
        \\ \label{eq:projection error coordinate 1 term 2}
         + &\| \C^*   {   \times_2    \widehat U _2  \widehat U^{\top} _2  \times \cdots \times_d    \widehat U _d  \widehat U^{\top} _d  } (\mi_1; \mi_2, \ldots,\mi_d) - [ \C^*    { \times_2    \widehat U _2  \widehat U^{\top} _2  \times \cdots \times_d    \widehat U _d  \widehat U^{\top} _d   (\mi_1; \mi_2, \ldots,\mi_d)]_{\R_1} }\|_F, 
     \end{align}
    where the inequality follows from \Cref{eq:projection error for approximately low rank matrix}.   By      $  \|  \widehat U _j  \widehat U^{\top} _j  -  U ^*_j    U^{*\top} _j \| = {\bigO}_{\p}(m^{-1}) $ for all $j\in \{ 1,\ldots, d\}$.  Therefore by \Cref{lemma: random projection deviation bound},
    \begin{align}  \nonumber 
        \eqref{eq:projection error coordinate 1 term 1} = &\|(\C^* -\widehat \C)  \times_2    \widehat U _2  \widehat U^{\top} _2  \times \cdots \times_d    \widehat U _d  \widehat U^{\top} _d     \|_F   =
        \\ \nonumber  
        \le & \|(\C^* -\widehat \C)  \times_2    \widehat U _2    \times \cdots \times_d    \widehat U _d       \|_F  \|  \widehat U _2^\top \| _{\op} \cdots \|  \widehat U _d ^\top  \|_{\op}
        \\ \label{eq:projection error coordinate 1 term 1 bound}        = & \|(\C^* -\widehat \C)  \times_2    \widehat U _2    \times \cdots \times_d    \widehat U _d       \|_F = 
        {\bigO}_\p \bigg( \sqrt{  \frac{m\R_2\cdots \R_d}{N} }\bigg),
    \end{align}
    where the last equality follows from \Cref{lemma:consistency of subspace in step 1}.
In addition  by  \Cref{lemma:Mirsky tensor version 2},
\begin{align}\label{eq:projection error coordinate 1 term 2 bound}
        \eqref{eq:projection error coordinate 1 term 2} \le 
         \xiapp +C_3 m^{-\alpha} . 
    \end{align}
Observe that \eqref{eq:projection error coordinate 1 term 1 bound} and \eqref{eq:projection error coordinate 1 term 2 bound}
    together imply that with high probability,
  \begin{align}  \label{eq:projection term one bound three} \|\C^* \times_1 \widehat V_{1\perp} 
        \widehat V _{1\perp} ^ \top   \times_2    \widehat U _2  \widehat U^{\top} _2  \times \cdots \times_d    \widehat U _d  \widehat U^{\top} _d   \|_F  \le C_4\bigg( \sqrt{\frac{m\R_2\cdots \R_d}{N}}  +  \xiapp + m^{-\alpha}\bigg). 
        \end{align}
\eqref{eq:projection term one bound two} and \eqref{eq:projection term one bound three} together  imply that with high probability,
 \begin{align*} \|  \C^* \times_1 \widehat V_{1\perp}      \widehat V _{1\perp} ^\top \|_F  \le C_5\bigg( \sqrt{ \frac{m\R_2\cdots \R_d}{N}  } +  \xiapp + m^{-\alpha}\bigg). 
 \end{align*}
 It follows from \eqref{eq:tensor estimation term one bound one} that 
 \begin{align} \label{eq:tensor probability bound}
     \| \mathcal C^* - \C^*{\times_1   {\mathcal P}_1 ^V \times \cdots   \times_d     {\mathcal P}_d^V }\| _F  
     \le \sum_{j=1}^d \|  \C^* \times_j \widehat V_{j\perp}      \widehat V _{j\perp} ^\top \|_F  
\le C_6\bigg( \sqrt{ \frac{m\R_2\cdots \R_d}{N}}  +  \xiapp + m^{-\alpha}\bigg).  \end{align}
{\bf Step 5.} Note that  
the projection matrices $ {\mathcal P}_j ^V =\widehat V_j \widehat V_j^\top$  are computed based on the empirical measure $\widehat A$. Therefore $ \{ {\mathcal P}_j ^V\} _{j=1}$  are independent of $\widehat A'$ and $\widehat \C' = \C(\widehat A' \times _1 \mathcal P_{\mathcal M_1} \times \cdots \times _d \mathcal P_{\mathcal M_d} ).$ So by  \Cref{remark:non-random projection frobenius norm expectation} 
$$ \E(   \| (\C^*-\widehat \C ') \times_1   {\mathcal P}_1 ^V \times \cdots   \times_d     {\mathcal P}_d^V  \| _F^2)  = \bigO\bigg( \frac{\prod_{j=1}^d \R_j}{N} \bigg)  .$$
{\bf Step 6.} By \eqref{eq:coefficient tensor deviation bound one}, \eqref{eq:tensor probability bound} and {\bf Step 5}, it follows that with high probability, 
\begin{align} \nonumber 
         \| \widetilde \C -  \mathcal C^* \|_F  
        \le & \| \mathcal C^* - \C^*\times_1   {\mathcal P}_1 ^V \times \cdots   \times_d     {\mathcal P}_d^V \| _F +
        \| (\C^*-\widehat \C ') \times_1   {\mathcal P}_1 ^V \times \cdots   \times_d     {\mathcal P}_d^V  \| _F 
        \\ 
        \le & C_7\bigg( \sqrt{ \frac{m\R_1\R_2\cdots \R_d}{N}}  +  \xiapp + m^{-\alpha}\bigg)  .
    \end{align}

\end{proof}

\begin{lemma} \label{lemma:consistency of subspace in step 1}
     Suppose all  the conditions as in \Cref{theorem:density bound main} hold.    Let $U_j^*$ and $\widehat U_j$ be defined as in \Cref{lemma:consistency of sketching}.
Then $ 
    \| \widehat U^\top_{j} U^*_{j \perp } \| _{\op}  =
    {\bigO}_\p \big( m^{-1}   \big)$  and     
   $ \sigma_{\min}^{-1} ( U^{* \top}_j \widehat U_j) ={\bigO}_\p (1).$
 
     \end{lemma}
\begin{proof}
    By   \Cref{corollary:angle using wedin}, we have that 
\begin{align*}
      \| \widehat U^\top_{j} U^*_{j \perp } \| _{\op} = \sqrt { 1 - \sigma_{\min} (\widehat U_j^\top    U^* _j
 ) }  = {\bigO}_\p \bigg(  \frac{  \sigma^*_{j,\R_j+1 }}{\sigma^*_{j,\R_j}}  +    \sqrt { \frac{\big( m + l^{d-1}_j \big) \log(N) }{ N\sigma^{*2}_{j,\R_j}  } }    \bigg) .  
\end{align*}   
Since \eqref{eq:main result snr} holds for sufficiently large constant, with the choices of $m$ and $\ell_j$ in  \Cref{theorem:density bound main},   it holds   that
$ \| \widehat U^\top_{j} U^*_{j \perp } \| _{\op}  =
    {\bigO}_\p \big( m^{-1}   \big)$. Since
with high probability 
 $\sqrt{ 1- \sigma_{\min}^{ 2}( \widehat U_j ^\top  U^{*  }_j   )}  <1,
  $  it follows that 
  $ \sigma_{\min}^{ -1}(  \widehat U_j ^\top  U^{*  }_j  ) =\bigO_\p(1).$  
\end{proof}

\subsection{Consistency of Sketching} 

Suppose  $\{\phi_k \}_{k=1}^{\infty } $ be a collection of basis function  $ \lt(\mathcal O)$  defined in   \Cref{assume: baised in projected space in operator norm}. For $j\in \{ 1,\ldots, d\}$, let $\mathcal M_j$ and $\mathcal S_j$  be defined as in  \eqref{eq:projection basis for tensor estimation} and \eqref{eq:projection basis for tensor estimation 2} respectively. 
Denote 
$$ x=z_j \quad \text{and} \quad y=(z_1,\ldots,z_{j-1},\ldots, z_{j+1},\ldots, z_d).$$

\begin{lemma} \label{lemma:consistency of sketching} 
For every $j\in \{1,\ldots,d\}$, suppose   $\sigma_{j,\R_j }  ^* \ge C_1 \max \bigg\{ l^{-\alpha},\sigma_{j,\R_j+1} ^*,\sqrt { \frac{ ( m + l^{d-1}) \log(N) }{ N} }  \bigg\}$ for some absolute constant $C_1$.  Let $U^*_j  \in \mathbb O ^{m\times \R_j}$ with columns being the $\R_j$-leading  left singular vectors of the matrix $ \mathcal C ^* (\mi_j; \mi_1,\ldots, \mi_{j-1},\mi_{j+1},\ldots, \mi_d) \in \mathbb R^{m\times m^{d-1}}$, where  $\mathcal C^*=   A^*    {\times_1 \mathcal P _{\mathcal M_1} \times \ldots \times_d  \mathcal P_{\mathcal M_d}}      \in \mathbb  R^{m^d} $.
Let $\widehat U _j  \in \mathbb O ^{m\times \R_j}$ with columns being the $\R_j$-leading  left singular vectors of the matrix $\mathcal C(\widehat A \times_x \mathcal P _{\mathcal M_j} \times_y \mathcal P_{ \mathcal S_j } )\in \mathbb R^{m\times l_j^{d-1}} $. Then  
\begin{align*}
     \|  U^*_j  {U^*_j}^\top  -  \widehat U _j \widehat U^\top_j \| _{\op}      = {\bigO}_\p \bigg(  \frac{  \sigma^*_{j,\R_j+1 }}{\sigma^*_{j,\R_j}}  +    \sqrt { \frac{\big( m + l^{d-1}_j \big) \log(N) }{ N\sigma^{*2}_{j,\R_j}  } }    \bigg) .
\end{align*}   
\end{lemma}

\begin{proof} Without loss of generality, assume $j=1$.
Throughout the proof, for brevity we  denote  $U^*  =U^*_j$,  $     \widehat U  =\widehat U _j$, $\mathcal M= \mathcal M_1$, $\mathcal S=\mathcal S_1$, $r=r_1$, $\ell=\ell_1$ and 
$$ \mathcal M^{\otimes d-1} = \mathcal M_2 \otimes \cdots \otimes \mathcal M_d.$$
It follows from  \eqref{eq:unfold coefficient tensor} that $\mathcal C ^* (\mi_1; \mi_2,\ldots, \mi_d) $ is the coefficient matrix of $ \mathcal C (A^* {\times_x \mathcal P_{\mathcal M} \times _y \mathcal P_{\mathcal M ^{\otimes d-1}} }  ).$ Therefore the columns of  $U^*  $  are    the $\R$-leading  left singular vectors of the matrix $ \mathcal C (A^* {\times_x \mathcal P_{\mathcal M} \times _y \mathcal P_{\mathcal M ^{\otimes d-1}} }  )$.
\\
\\
{\bf Step 1.} Let $[A^* {\times_x \mathcal P_{\mathcal M} \times _y \mathcal P_{\mathcal M ^{\otimes d-1}} }]_\R$ be the best rank-$\R$ approximation of $A^* {\times_x \mathcal P_{\mathcal M} \times _y \mathcal P_{\mathcal M ^{\otimes d-1}} }$.  See \Cref{theorem:best low rank function} for the definition of best rank $\R$ approximation for functions. In this step, we show that 
\begin{align}\label{eq:minimal singular gap of projection}\sigma_{\R}  ([ A^* {\times_x \mathcal P_{\mathcal M} \times _y \mathcal P_{\mathcal M ^{\otimes d-1}} }]_\R \times _y \mathcal P_{  \mathcal S} )  \ge \frac{1}{2} \sigma_{1,\R} ^*. 
\end{align}
To see this, observe  that 
\begin{align*}
&\|   \  [ A^* {\times_x \mathcal P_{\mathcal M} \times _y \mathcal P_{\mathcal M ^{\otimes d-1}} }]_\R \times _y \mathcal P_{  \mathcal S }   - A^* \times_x \mathcal P_{\mathcal M} \times _y \mathcal P_{\mathcal M ^{\otimes d-1}}     \|_{\op}
\\
\le & \|   \  [ A^* {\times_x \mathcal P_{\mathcal M} \times _y \mathcal P_{\mathcal M ^{\otimes d-1}} }]_\R \times _y \mathcal P_{  \mathcal S}   -\    A^* {\times_x \mathcal P_{\mathcal M} \times _y \mathcal P_{\mathcal M ^{\otimes d-1}} }  \times _y \mathcal P_{  \mathcal S}     \|_{\op} 
\\ +
& \|    \    A^* {\times_x \mathcal P_{\mathcal M} \times _y \mathcal P_{\mathcal M ^{\otimes d-1}} }  \times _y \mathcal P_{  \mathcal S }  -  A^* \times_x \mathcal P_{\mathcal M} \times _y \mathcal P_{\mathcal M ^{\otimes d-1}}      \|_{\op} .
\end{align*}
Note that 
\begin{align}\nonumber
 & \|   \  [ A^* {\times_x \mathcal P_{\mathcal M} \times _y \mathcal P_{\mathcal M ^{\otimes d-1}} }]_\R \times _y \mathcal P_{  \mathcal S}   -\    A^* {\times_x \mathcal P_{\mathcal M} \times _y \mathcal P_{\mathcal M ^{\otimes d-1}} }  \times _y \mathcal P_{  \mathcal S}     \|_{\op} 
\\  \nonumber
= & \|   \big(   [ A^* {\times_x \mathcal P_{\mathcal M} \times _y \mathcal P_{\mathcal M ^{\otimes d-1}} }]_\R - A^* {\times_x \mathcal P_{\mathcal M} \times _y \mathcal P_{\mathcal M ^{\otimes d-1}} }  \big)   \times _y \mathcal P_{  \mathcal S}     \|_{\op}  
\\ \nonumber
\le & \|     [ A^* {\times_x \mathcal P_{\mathcal M} \times _y \mathcal P_{\mathcal M ^{\otimes d-1}} }]_\R - A^* {\times_x \mathcal P_{\mathcal M} \times _y \mathcal P_{\mathcal M ^{\otimes d-1}} }    \| _{\op }  \|  \mathcal P_{  \mathcal S }     \|_{\op}  
\\ \nonumber
\le & \|     [ A]_\R^* {\times_x \mathcal P_{\mathcal M} \times _y \mathcal P_{\mathcal M ^{\otimes d-1}} }   - A^* {\times_x \mathcal P_{\mathcal M} \times _y \mathcal P_{\mathcal M ^{\otimes d-1}} }    \| _{\op }  \|  \mathcal P_{ \mathcal S }     \|_{\op}   
\\ \nonumber
=
 & \|    \big(   [ A]_\R^* - A^* \big)  {\times_x \mathcal P_{\mathcal M} \times _y \mathcal P_{\mathcal M ^{\otimes d-1}} }      \| _{\op }  \|  \mathcal P_{  \mathcal S }     \|_{\op}     
 \\ \label{eq:sketching best low rank error in operator norm}
 \le & \| [ A]_\R^* - A^* \| _{\op} \| \mathcal P_{\mathcal M} \|_{\op} \| \mathcal P_{\mathcal M ^{\otimes d-1}}  \|_{\op} \|  \mathcal P_{  \mathcal S }     \|_{\op}  \le \| [ A]_\R^* - A^* \| _{\op}     = \sigma^*_{1, \R+1}  ,
\end{align}
where the second inequality follows from \eqref{remark:best low rank} and the fact that $  [ A]_\R^* {\times_x \mathcal P_{\mathcal M} \times _y \mathcal P_{\mathcal M ^{\otimes d-1}} }$ is a function with rank at most $r$;  the last inequality follows from  the fact that projection functions has operator norm bounded by 1; and the last equality follows from \Cref{theorem:best low rank function}.
\\
\\
In addition 
\begin{align*}
  &\|        A^* {\times_x \mathcal P_{\mathcal M} \times _y \mathcal P_{\mathcal M ^{\otimes d-1}} }  \times _y \mathcal P_{  \mathcal S }  -  A^* \times_x \mathcal P_{\mathcal M} \times _y \mathcal P_{\mathcal M ^{\otimes d-1}}      \|_{\op} 
  \\
  = & \|        A^* {\times_x \mathcal P_\mathcal M \times_y \mathcal P_{\mathcal S}}   -  A^* \times_x \mathcal P_{\mathcal M} \times _y \mathcal P_{\mathcal M ^{\otimes d-1}}      \|_{\op} 
  \\
  \le & \|        A^* {\times_x \mathcal P_\mathcal M \times_y \mathcal P_{\mathcal S} }  -A^*\| _{\op} + \|  A^*   -  A^* \times_x \mathcal P_{\mathcal M} \times _y \mathcal P_{\mathcal M ^{\otimes d-1}}      \|_{\op} 
  \\
  = & \bigO(m^{-\alpha} + l^{-\alpha}) + \bigO(m^{-\alpha} )  = \bigO(  l^{-\alpha}),
\end{align*}
where the first equality follows from $ \mathcal S \subset \mathcal M^{\otimes d-1}$; the second  equality follows from \Cref{theorem: approximation theory}; and the last equality holds because $l\le m$.
Therefore 
\begin{align}\label{eq:sketching consistency deviation of sketch}
 \|   \  [ A^* {\times_x \mathcal P_{\mathcal M} \times _y \mathcal P_{\mathcal M ^{\otimes d-1}} }]_\R \times _y \mathcal P_{  \mathcal S}   - A^* \times_x \mathcal P_{\mathcal M} \times _y \mathcal P_{\mathcal M ^{\otimes d-1}}     \|_{\op}\le \sigma^*_{1, \R+1} (A^*) + C _2 l^{-\alpha} 
\end{align}
for some sufficiently large constant $C_2$. It follows that 
\begin{align*}
 &\sigma_\R ( [ A^* {\times_x \mathcal P_{\mathcal M} \times _y \mathcal P_{\mathcal M ^{\otimes d-1}} }]_\R ) 
 \\
 \ge &\sigma_\R (A^* {\times_x \mathcal P_{\mathcal M} \times _y \mathcal P_{\mathcal M ^{\otimes d-1}} }   )   -\|   \  [ A^* {\times_x \mathcal P_{\mathcal M} \times _y \mathcal P_{\mathcal M ^{\otimes d-1}} }]_\R \times _y \mathcal P_{  \mathcal S}   - A^* {\times_x \mathcal P_{\mathcal M} \times _y \mathcal P_{\mathcal M ^{\otimes d-1}} }    \|_{\op}
 \\ \ge &\sigma^*_{1,\R} - C_3m^{-\alpha} -\sigma^*_{1, \R+1} (A^*)  - C _2 l^{-\alpha}  \ge \frac{1}{2}\sigma^*_{1,\R},
\end{align*}
where the first inequality follows from   \Cref{lemma:SVD pertubation for operators}; the second inequality follows from \Cref{lemma:singular gap of population coefficient matrix} and \eqref{eq:sketching consistency deviation of sketch}; and the last inequality follows from  the assumption that $\sigma_{1,\R _1}  ^* \ge C_1 \max\{ l^{-\alpha},\sigma_{1,\R_1+1} ^* \}$.
\\
\\
{\bf Step 2.} We show that the  the column space of $\mathcal C([ A^* {\times_x \mathcal P_{\mathcal M} \times _y \mathcal P_{\mathcal M ^{\otimes d-1}} }]_\R  \times _y \mathcal P _{\mathcal S }) $  equals to the column space of $[\mathcal C(  A^* {\times_x \mathcal P_{\mathcal M} \times _y \mathcal P_{\mathcal M ^{\otimes d-1}} } )]_\R$.
\\
\\
To see this, note that 
\begin{align*} \mathcal C([ A^* {\times_x \mathcal P_{\mathcal M} \times _y \mathcal P_{\mathcal M ^{\otimes d-1}} }]_\R  \times _y \mathcal P _{\mathcal S})  
= & \mathcal C([ A^* {\times_x \mathcal P_{\mathcal M} \times _y \mathcal P_{\mathcal M ^{\otimes d-1}} }]_\R ) \mathcal C( \mathcal P _{\mathcal S })   
\\ =
& [\mathcal C(  A^* {\times_x \mathcal P_{\mathcal M} \times _y \mathcal P_{\mathcal M ^{\otimes d-1}} } )]_\R \mathcal C( \mathcal P _{\mathcal S })  , 
\end{align*}
where the first  equality follows from \Cref{lemma:product compactible coefficient matrix}; and the second equality follows from \Cref{corollary:singular value preserving} {\bf (b)}.  
So the column space of $\mathcal C([ A^* {\times_x \mathcal P_{\mathcal M} \times _y \mathcal P_{\mathcal M ^{\otimes d-1}} }]_\R  \times _y \mathcal P _{\mathcal S }) $  is contained in the column space of $[\mathcal C(  A^* {\times_x \mathcal P_{\mathcal M} \times _y \mathcal P_{\mathcal M ^{\otimes d-1}} } )]_\R$, or simply 
\begin{align} \label{eq:column space after sketching} \text{Col} \big\{  \mathcal C([ A^* {\times_x \mathcal P_{\mathcal M} \times _y \mathcal P_{\mathcal M ^{\otimes d-1}} }]_\R  \times _y \mathcal P _{\mathcal S }) \big\}  \subset  \text{Col} \big\{ [\mathcal C(  A^* {\times_x \mathcal P_{\mathcal M} \times _y \mathcal P_{\mathcal M ^{\otimes d-1}} } )]_\R  \big\}.
\end{align}
\
\\
Note that the matrix $[\mathcal C(  A^* {\times_x \mathcal P_{\mathcal M} \times _y \mathcal P_{\mathcal M ^{\otimes d-1}} } )]_\R $ is at most rank $\R$.  By \Cref{lemma:singular gap of population coefficient matrix}, $$  \sigma_\R ( [ \mathcal C(  A^* {\times_x \mathcal P_{\mathcal M} \times _y \mathcal P_{\mathcal M ^{\otimes d-1}} } ) ]_\R)=\sigma_\R (\mathcal C(  A^* {\times_x \mathcal P_{\mathcal M} \times _y \mathcal P_{\mathcal M ^{\otimes d-1}} } ))   >0. $$
So the matrix $[\mathcal C(  A^* {\times_x \mathcal P_{\mathcal M} \times _y \mathcal P_{\mathcal M ^{\otimes d-1}} } )]_\R$ is exactly rank $\R$.

\
\\
In addition, $\mathcal C([ A^* {\times_x \mathcal P_{\mathcal M} \times _y \mathcal P_{\mathcal M ^{\otimes d-1}} }]_\R  \times _y \mathcal P _{\mathcal S}) $   is at most rank $\R$, because $[ A^* {\times_x \mathcal P_{\mathcal M} \times _y \mathcal P_{\mathcal M ^{\otimes d-1}} }]_\R$ is a function of rank at most $\R$. By {\bf Step 1}, $\sigma_\R ( [ A^* {\times_x \mathcal P_{\mathcal M} \times _y \mathcal P_{\mathcal M ^{\otimes d-1}}  ]_\R  \times _y \mathcal P _{\mathcal S}}  ) >0 $. \\
So $\mathcal C([ A^* {\times_x \mathcal P_{\mathcal M} \times _y \mathcal P_{\mathcal M ^{\otimes d-1}} }]_\R  \times _y \mathcal P _{\mathcal S }) $   is exactly rank $\R$. 
\\
\\
Since both $[\mathcal C(  A^* {\times_x \mathcal P_{\mathcal M} \times _y \mathcal P_{\mathcal M ^{\otimes d-1}} } )]_\R $ and $\mathcal C([ A^* {\times_x \mathcal P_{\mathcal M} \times _y \mathcal P_{\mathcal M ^{\otimes d-1}} }]_\R  \times _y \mathcal P _{\mathcal S}) $ are exactly rank $\R$ matrices, \eqref{eq:column space after sketching} implies that 
\begin{align*}   \text{Col} \big\{  \mathcal C([ A^* {\times_x \mathcal P_{\mathcal M} \times _y \mathcal P_{\mathcal M ^{\otimes d-1}} }]_\R  \times _y \mathcal P _{\mathcal S }) \big\}  = \text{Col} \big\{ [\mathcal C(  A^* {\times_x \mathcal P_{\mathcal M} \times _y \mathcal P_{\mathcal M ^{\otimes d-1}} } )]_\R  \big\}.
\end{align*}
\ 
\\
{\bf Step 3.}    The columns of $U^* $ are  the  leading $\r$ left singular vectors of $\mathcal C(  A^* {\times_x \mathcal P_{\mathcal M} \times _y \mathcal P_{\mathcal M ^{\otimes d-1}} }  )$. 
Therefore the columns of $U^* $ spans the column space of  $ [ \mathcal C(  A^* {\times_x \mathcal P_{\mathcal M} \times _y \mathcal P_{\mathcal M ^{\otimes d-1}} }  ) ]_\R$.
By { \bf Step 2}, the columns of $U^* $ also  spans the column space of $ \text{Col} \big\{  \mathcal C([ A^* {\times_x \mathcal P_{\mathcal M} \times _y \mathcal P_{\mathcal M ^{\otimes d-1}} }]_\R  \times _y \mathcal P _{\mathcal S }) \big\} $. Note that columns of $U^*$ are not necessarily the singular vectors of  $\mathcal C([ A^* {\times_x \mathcal P_{\mathcal M} \times _y \mathcal P_{\mathcal M ^{\otimes d-1}} }]_\R  \times _y \mathcal P _{\mathcal S })$, but 
$   U^* {U^*}^\top $ is the projection matrix onto the column space of $\mathcal C([ A^* {\times_x \mathcal P_{\mathcal M} \times _y \mathcal P_{\mathcal M ^{\otimes d-1}} }]_\R  \times _y \mathcal P _{\mathcal S})$.

\
\\
It follows from \Cref{corollary:Wedin} that 
\begin{align*}  &  \|  U^*  U^{*\top}  -  \widehat U  \widehat U^\top    \|_{\op } \le  \sqrt 2  \frac{ \eqref{eq:sketching wedin term 1}  }{ \eqref{eq:sketching wedin term 2} -\eqref{eq:sketching wedin term 1} }, \quad \text{where}
\\
&\eqref{eq:sketching wedin term 1} =  \| \mathcal C( [ A^* \times_x \mathcal P_{\mathcal M} \times_y \mathcal P _{\mathcal M ^{\otimes d-1}}]_\R  \times _y \mathcal P _{\mathcal S } ) - \mathcal C( \widehat A \times_x \mathcal P _{\mathcal M} \times_y \mathcal P_{ \mathcal S} )     \|_{\op} 
 \\
    & \eqref{eq:sketching wedin term 2}  =  \sigma_\R (\mathcal C( [ A^* \times_x \mathcal P_{\mathcal M} \times_y \mathcal P _{\mathcal M ^{\otimes d-1}}]_\R  \times _y \mathcal P _{\mathcal S} )  )   . \end{align*}
\refstepcounter{equation}\label{eq:sketching wedin term 1}
\refstepcounter{equation}\label{eq:sketching wedin term 2}
\
\\
  Note that with high probability,
 \begin{align*} \eqref{eq:sketching wedin term 1}     =  & \|   [ A^* \times_x \mathcal P_{\mathcal M} \times_y \mathcal P _{\mathcal M ^{\otimes d-1}}]_\R  \times _y \mathcal P _{\mathcal S}   -  \widehat A \times_x \mathcal P _{\mathcal M} \times_y \mathcal P_{ \mathcal S}       \|_{\op} 
\\ \le & \|   [ A^* \times_x \mathcal P_{\mathcal M} \times_y \mathcal P _{\mathcal M ^{\otimes d-1}}]_\R  \times _y \mathcal P _{\mathcal S}   -   A ^*  \times_x \mathcal P _{\mathcal M} \times_y \mathcal P_{ \mathcal S}       \|_{\op}  
\\ + &
\| A ^*  \times_x \mathcal P _{\mathcal M} \times_y \mathcal P_{ \mathcal S }       -\widehat A    \times_x \mathcal P _{\mathcal M} \times_y \mathcal P_{ \mathcal S }      \|_{\op}
\\ 
 \le &  \sigma^*_{1,\r+1 }  + C_2 \bigg ( \sqrt { \frac{ 
 \big( m + l^{d-1} \big) \log(N)}{ N} } \bigg),
 \end{align*}
 where the last inequality follows from \eqref{eq:sketching best low rank error in operator norm} and \Cref{lemma:sobolev matrix Bernstein}.
\\
\\
 In addition, 
\begin{align}
\eqref{eq:sketching wedin term 2}   = \sigma_\R (\mathcal C( [ A^* \times_x \mathcal P_{\mathcal M} \times_y \mathcal P _{\mathcal M ^{\otimes d-1}}]_\R  \times _y \mathcal P _{\mathcal S} )  ) \ge \frac{1}{2} \sigma^*_{1,\r} . 
\end{align}
So  $$\eqref{eq:sketching wedin term 2}  -\eqref{eq:sketching wedin term 1}  \ge  \frac{1}{2}\sigma^*_{1,\r}  - \sigma^*_{1,\r+1 }  - C_2 \bigg ( \sqrt { \frac{ ( m + l^{d-1}) \log(N) }{ N} } \bigg) \ge  \frac{1}{4} \sigma^*_{1,\r} ,$$
where the last inequality holds because  $\sigma_{1,\R }  ^* \ge C_1 \max \bigg\{ l^{-\alpha},\sigma_{1,\R+1} ^* ,\sqrt { \frac{ ( m + l^{d-1}) \log(N) }{ N} } \bigg\}$.
$$\|  U^*  {U^*}^\top  -  \widehat U  \widehat U^\top  \|_{\op}   \le \sqrt 2  \frac{ \eqref{eq:sketching wedin term 1}  }{ \eqref{eq:sketching wedin term 2}  - \eqref{eq:sketching wedin term 1} } \le C_3\frac{\sigma^*_{1,\r+1 }  +    \sqrt { \frac{( m + l^{d-1} ) \log(N) }{ N} }   }{\sigma^*_{1,\r} } .$$

\end{proof}
 \begin{corollary} \label{corollary:angle using wedin} Suppose all  the  conditions   in \Cref{lemma:consistency of sketching} hold.
   Let $U_{j\perp} ^* \in \mathbb O^{m\times  (m-r_j)}$  be the orthogonal complement of $U_{j } ^* $.
Then 
 \begin{align}\label{eq:davis khan 2}   \| \widehat U^\top_{j} U^*_{j \perp } \| _{\op} = \sqrt { 1 - \sigma_{\min} (\widehat U_j^\top    U^* _j
 ) }  = {\bigO}_\p \bigg(  \frac{  \sigma^*_{j,\R_j+1 }}{\sigma^*_{j,\R_j}}  +    \sqrt { \frac{\big( m + l^{d-1}_j \big) \log(N) }{ N\sigma^{*2}_{j,\R_j}  } }    \bigg) . 
 \end{align}
 Here  $\sigma_{\min} (\widehat U_j^\top    U^* _j)  $  is  the minimal singular value of the matrix $\widehat U_j^\top    U^* _j  \in \mathbb R^{r\times r }$. 
 \begin{proof}
 By Lemma 1 of \cite{cai2018rate},
$$\| \widehat U^\top_{j} U^*_{j \perp } \| _{\op} = \sqrt { 1 - \sigma_{\min} (\widehat U_j^\top    U^* _j
 ) } \le 2 \|  U^*_j  {U^*_j}^\top  -  \widehat U _j \widehat U^\top_j \| _{\op}  .$$ 
 Therefore \eqref{eq:davis khan 2}
 is a direct consequence  of the  \Cref{lemma:consistency of sketching}.
\end{proof} 
 \end{corollary}

\section{Deviation bounds}

Throughout this section,  let    $\{ \BZ_i\}_{i=1}^N  \subset \mathcal  O^{d} $  be the   i.i.d.\, data  sampled from  the density $A^*$ and  $\widehat A =\frac{1}{N}\sum_{i=1}^N \delta_{\BZ_i}$ be  the corresponding empirical measure, and 
$$\C^* = \C(A^*{ \times_1\mathcal P_{\M_1} \times \cdots     \times_d\mathcal P_{\M_d} })\quad \text{and} \quad \widehat \C  = \C( \widehat A { \times_1\mathcal P_{\M_1} \times \cdots     \times_d\mathcal P_{\M_d} })$$
\begin{lemma} \label{lemma:sobolev matrix Bernstein}    Let $p  \in \mathbb Z^+$ be such that $p< d$.
Suppose  $\{\phi_k \}_{k=1}^{\infty } $ is  a collection of $ \lt(\mathcal O)$ basis such that $\|\phi_k \|_\infty \le C_{\phi} $ for some absolute constant $C_{\phi} $.
For positive integers $\m$ and $\l$, denote  
\begin{align} \label{eq:projection basis 1}
\mathcal M =\Span \bigg\{\phi_{\mi_1}(z_1)  \cdots  \phi_{\mi_{p}}(z_{p}) \bigg \}_{\mi_1, \ldots, \mi_{p}=1}^\m \  \text{and} \ \   
\mathcal N = \Span\bigg\{  \phi_{\li_1 }(z_{p+1})\cdots  \phi_{\li_{{d-p}}}(z_{d })\bigg\}_{\li_1, \ldots, \li_{{d-p} }=1}^\l.
\end{align} 
Suppose in addition that $ N \ge C_1  \max\{ m^{p}, l^{{d-p}}\}    $ for some absolute constant $C_1$. Then it holds that
$$\|  (\widehat A - A ^*) \times_x \mathcal P_\mathcal M \times_y \mathcal P_{\mathcal N}  \|_{\op}  ={\bigO}_{\p} \bigg( \sqrt{ \frac{ \{ m^{p} +l^{d-p}\} \log(N)  }{N}}    \bigg). $$

\end{lemma} 
 \begin{proof}
 Denote  
$$x=(z_1,\ldots, z_{p}) \quad \text{and} \quad  y=(z_{p+1},\ldots,  z_{d}).$$
For positive integers $\m$ and $\l$, by ordering the indexes $(\mi_1,\ldots, \mi_{p})$ and  $(\li_1,\ldots, \li_{{d-p}})$   in \eqref{eq:projection basis 1}, we can also write
 \begin{align} 
 \mathcal M =\text{span}\{\Phi_\mi(x) \}_{\mi=1}^{\m^{p}}  \quad \text{and} \quad 
 \mathcal N =  \text{span} \{  \Psi_{\li }(y )  \}_{\li =1}^{\l^{{d-p}}}.
 \end{align}  
 Here with the multi-index $\mu= (\mu_1,\ldots, \mu_p)$, $$\Phi_\mu(x) = \phi_{\mi_1} (z_1) \ldots \phi_{\mi_p}(z_p) . $$
 So $\|\Phi_\mi \|_\infty \le C_\phi^{p} $. Similarly $\|\Psi_\li \|_\infty \le C_\phi^{d-p} $. Therefore 
 \begin{align}\label{eq:bernstein l_infty bound}\langle A^*, \Phi_\mi \otimes  \Psi_\li  \rangle \le C_\phi ^{d}\|A ^* \|_\infty 
 \end{align}
{\bf Step 1.} Observe that $   \widehat  A   \times_x \mathcal P_\mathcal M \times_y \mathcal P_{\mathcal N}  $
and $   A ^*  \times_x \mathcal P_\mathcal M \times_y \mathcal P_{\mathcal N} $ are both zero outside the subspace $\mathcal M\otimes \mathcal N $. Let $\mathcal C( \widehat  A   \times_x \mathcal P_\mathcal M \times_y \mathcal P_{\mathcal N}   ) $ and $ \mathcal C( A ^*  \times_x \mathcal P_\mathcal M \times_y \mathcal P_{\mathcal N} )$ be the coefficient matrices in $\mathbb R^{\m^{p}\times \l^{{d-p}}}$.  Note that  
\begin{align} \nonumber  
 \|  (\widehat A  - A  ^*) {\times_x \mathcal P_\mathcal M \times_y \mathcal P_{\mathcal N} }  \|_{\op}  &=   
 \| \mathcal C( \widehat  A  { \times_x \mathcal P_\mathcal M \times_y \mathcal P_{\mathcal N} } )  -\mathcal C( A^*  { \times_x \mathcal P_\mathcal M \times_y \mathcal P_{\mathcal N} ) } \|_{\op} 
 \\ \label{eq:matrix bernstein in hilbert space} &=    \| \frac{1}{N } \sum_{i=1}^N B_i -\E(B_i)\|_{\op}, 
\end{align}
where the first equality follows from  \Cref{lemma:coefficient tensor distance preserving}; and  in the second equality, $B _ i \in \mathbb R^{  m^{p}\times l^{d-p}}$  is such that 
$  ( B _ i)_{\mi,\li} = \Phi_\mi \otimes \Psi_{\li }(\BZ _i)    .$  Note that 
$ \big( \mathcal C( \widehat  A  { \times_x \mathcal P_\mathcal M \times_y \mathcal P_{\mathcal N} } ) \big)_{\mi, \li} =\frac{1}{N} \sum_{i=1}^N  \Phi_\mu \otimes \Phi_\eta (\BZ_i) $. So 
$$ \frac{1}{N } \sum_{i=1}^N B_i = \mathcal C( \widehat  A  { \times_x \mathcal P_\mathcal M \times_y \mathcal P_{\mathcal N} } ) . $$
Therefore $\E( ( B _ i)_{\mi,\li} ) =  \langle A^* ,  \Phi_\mi \otimes \Psi_{\li } \rangle =\big( \mathcal C(   A^*  { \times_x \mathcal P_\mathcal M \times_y \mathcal P_{\mathcal N} } ) \big)_{\mi, \li}  $.
\\
\\
{\bf Step 2.}  In this step,  we apply   \Cref{theorem:matrix bernstein}   to bound \eqref{eq:matrix bernstein in hilbert space}.
To bound $\| B_i -\E(B_i)\|_{\op}$, let $v= (v_1,\ldots, v_{\m^{p} })  \in \mathbb R^{\m^{p} }$ and  $w= (w_1,\ldots, w_{\l^{{d-p}}})  \in \mathbb R^{\l^{{d-p}} }$. Then  $\| B_i\|_{\op} =  \sup_{|v|=1, |w|=1} v^\top B_i w$. Observe that
 \begin{align*}
  v^\top B_i w  = &
 \sum_{\mi=1}^{m^p} \sum_{\li=1}^{l^{d-p}} v_\mi (B_i)_{\mi, \li}  w_{\li}
 = \sum_{\mi=1}^{m^p} \sum_{\li=1}^{l^{d-p}}  v_\mi   \Phi_\mi \otimes \Psi_{\li }(\BZ _i)  w_{\li} 
 \\
 \le  &\sum_{\mi=1}^{m^p} \sum_{\li=1}^{l^{d-p}} |v_\mi |   C_\phi^{d}   |w_\li|  \le   C_\phi^{d} \sqrt { \sum_{\mi=1}^{m^p}  |v_\mi |  ^2 }
 \sqrt { \sum_{\li=1}^{l^{d-p}} |w_\li| ^2 }
 =  C_\phi^{d}  \sqrt {m^p l^{d-p}},
 \end{align*}
 where the first inequality follows from the fact that 
 $\|\Phi_\mi \|_\infty \le C_\phi^{p} $ and  $\|\Psi_\li \|_\infty \le C_\phi^{d-p} $. Similarly calculations shows that by  \eqref{eq:bernstein l_infty bound}, 
 $  \| \E (B_i) \|_{\op}  \le  C_\phi^{d}  \|A^*\|_\infty  \sqrt {m^p l^{d-p}}.$ 
 So 
 $$\| B_i - \E (B_i) \|_{\op}  \le C _2   \sqrt {m^p l^{d-p}} . $$
 \
 \\
 {\bf Step 3.}   Note that 
 $   \E( \big\{ B_i -\E(B_i) \big\} \big\{ B_i -\E(B_i) \big\} ^\top   )  = \E (B_i B_i^\top ) -\E (B_i) \{ \E( B_i)\} ^\top , $ and $\E (B_i) \{ \E( B_i)\} ^\top$ is positive semi-definite.  So  by  8.3.2 on page 125 of \cite{tropp2015introduction},
 $$ \| \E( \big\{ B_i -\E(B_i) \big\} \big\{ B_i -\E(B_i) \big\} ^\top   ) \|_{\op } \le  \| \E (B_i B_i^\top ) \|_{\op }   .$$
 Let $v= (v_1,\ldots, v_{\m^{p} })  \in \mathbb R^{\m^{p} }$ be such that $|v|=1$. 
 Then $ \|  \E( B_i B_i^\top ) \|_{\op} = \sup_{|v|=1}  v^\top  \E( B_i B_i^\top )v $. Observe that 
 \begin{align*} v^\top  \E( B_i B_i^\top )v  = &\sum_{\mi, \mi'=1}^{m^p}    \sum_{\li=1} ^{l^{d-p}} v_\mi  \E \big \{ \Phi_\mi \otimes \Psi_\li (\BZ_i ) \cdot \Phi_{\mi'}  \otimes \Psi_\li (\BZ_i )  \big\} v_{\mi'}
 \\
=  &  \sum_{\mi, \mi'=1}^{m^p}    \sum_{\li=1} ^{l^{d-p}} \int_{\mathcal O^p   } \int_{\mathcal O^{d-p }}    v_\mi   \Phi_\mi (x)  \Psi_\li  (y)  \Phi_{\mi'} (x)  \Psi_\li  (y) v_{\mi'}    A^*(x, y)  dx dy  
\\
 = &  \int_{\mathcal O^{d-p }} \int_{\mathcal O^p   }    \sum_{\mi =1 }^{m^p}  v_\mi     \Phi_\mi (x)  \sum_{\mi' =1 }^{m^p}  v_{\mi'}    \Phi_{\mi'} (x)    \sum_{\li=1} ^{l^{d-p}}  \Psi_\li ^2   (y)    A^*(x, y) dx   dy   
 \\ 
 \le & \|A^* \|_\infty  \int_{\mathcal O^{d-p }} \int_{\mathcal O^p   }   \bigg \{  \sum_{\mi =1 }^{m^p}    v_\mi   \Phi_\mi (x)    \bigg \}     ^2  \bigg\{   \sum_{\li=1} ^{l^{d-p}}  \Psi_\li ^2   (y)      \bigg\} dx   dy  .
 \end{align*}
Since  $
 \{  \Psi_{\li }(y )  \}_{\li =1}^{l^{{d-p}}} $ is a collection of  orthonormal basis functions, we have 
 $$ \int _{\mathcal O^{d-p}} \sum_{\li=1} ^{l^{d-p}}  \Psi_\li ^2   (y) dy  =\sum_{\li=1} ^{l^{d-p}} \int _{\mathcal O^{d-p}}   \Psi_\li ^2   (y) dy  = l^{d-p}. $$
Since  $
 \{  \Phi_{\li }(y )  \}_{\mi =1}^{m^{d-p}} $ is a collection of  orthonormal basis functions,  and $|v| =1$, 
$$  \int_{\mathcal O^ p   }   \bigg \{  \sum_{\mi =1 }^{m^p}    v_\mi   \Phi_\mi (x)    \bigg \}     ^2   dx =1. $$
So 
$$\| \E( \big\{ B_i -\E(B_i) \big\} \big\{ B_i -\E(B_i) \big\} ^\top   ) \|_{\op }  \le \|  \E( B_i B_i^\top ) \|_{\op} \le \|A^* \|_\infty l^{d-p}.$$
\
\\
{\bf Step 4.}  Using similar calculations as in {\bf Step 3}, it follows that 
$$\| \E( \big\{ B_i -\E(B_i) \big\}^\top   \big\{ B_i -\E(B_i) \big\}  ) \|_{\op }  \le \|  \E( B_i B_i^\top ) \|_{\op} \le \|A^*\|_\infty m^{p}.$$
\
\\
{\bf Step 5.} Therefore by  \Cref{theorem:matrix bernstein}, it holds that 
\begin{align*} &\|  (\widehat A  - A  ^*) {\times_x \mathcal P_\mathcal M \times_y \mathcal P_{\mathcal N} }  \|_{\op}  
=     \| \frac{1}{N } \sum_{i=1}^N B_i -\E(B_i)\|_{\op}
\\
=& {\bigO}_{\p} \bigg( \sqrt{ \frac{ \{ m^{p} +l^{d-p}\} \log(m^{p}+l^{d-p})  }{N}}  +  \frac{\{ m^{p} +l^{d-p}\} \log(m^{p}+l^{d-p}) }{N}\bigg) 
\\
= &{\bigO}_{\p} \bigg( \sqrt{ \frac{ \{ m^{p} +l^{d-p}\} \log(N)  }{N}}    \bigg)  , 
\end{align*}
where the last inequality follows from $N \ge C_1  \max\{ m^{p}, l^{{d-p}}\}  $, 
so that $m^p/N \le \sqrt{m^p/N  }$, and $l^{d-p} /N \le \sqrt{l^{d-p}/N  }$.

\end{proof} 

\begin{theorem}\label{theorem:matrix bernstein} [Matrix Bernstein]
Let $\{ B_i\}_{i=1}^N $ be a set of independent random matrices  with dimensions $m_1\times m_2$. Suppose that for all $i$,
$ \E(B_i) =0$ and  $\| B_i\|_{\op} \le L . $
Denote $ m=\max\{ m_1,m_2\} $.  Then for any $a>2 $, it holds that with probability greater than $1-2 n^{-a+1}$  that 
$$ \bigg\|\frac{1}{N}  \sum_{i=1}^N B_i \bigg\|_{\op} \le \sqrt { \frac{2a \nu \log(m) }{ N}  } + \frac{2aL \log(m)}{3N} . $$
Here $ \nu =  \max \big \{ \| \E (B_i B_i^\top ) \|_{\op},\| \E (B_i^\top  B_i ) \|_{\op} \big \} $.
\end{theorem}
\begin{proof}
This is the well-known  matrix Bernstein inequality.  See \cite{tropp2015introduction} for the proof.
\end{proof}

\begin{lemma} \label{eq:variance of one test function}Let $G \in \lt(\mathcal O^d)$ be any non-random functions. Then 
$$ \E ( \langle \widehat A-A^* , G\rangle ^2) \le \frac{\|G\|_{\lt(\mathcal O^d) }^2\| A^*\|_\infty}{N}  .$$
\end{lemma}
\begin{proof}
    Note that 
    $ \langle \widehat A-A^* , G\rangle  = \frac{1}{N} \sum_{i=1}^N  \{ G(\BZ_i) - \langle A^* , G\rangle \}   $, and so $ \E( \langle \widehat A-A^* , G\rangle ) =0 $. Consequently,
    \begin{align*}
        &\E ( \langle \widehat A-A^* , G\rangle ^2)   = \var \big( \frac{1}{N} \sum_{i=1}^N \{ G(\BZ_i) - \langle A^* , G\rangle  \}  \big )  =\frac{1}{N} \var  \bigg(   G(\BZ_1)     \bigg)
        \\
        & \le \frac{1}{N}\E  (  G^2(\BZ_1) ) 
        =\frac{1}{N}\int_{\mathcal O^d} G^2 (z_1, \ldots,z_d ) A^*(z_1,\ldots, z_d)dz_1 \ldots dz_{d} \le\frac{1}{N}  \|G\|_{\lt(\mathcal O^d) }^2\| A^*\|_\infty.
    \end{align*}
\end{proof}

\begin{lemma} \label{lemma:non-random projection frobenius norm expectation}  Let  $\mathcal C^* = \mathcal C (A^* { \times_1 \mathcal P _{\mathcal M_1} \times \cdots  \times_d \mathcal P _{\mathcal M_d} } )$ and $\widehat {\mathcal C}= \mathcal C (\widehat A { \times_1 \mathcal P _{\mathcal M_1} \times \cdots  \times_d \mathcal P _{\mathcal M_d} } )$  be two tensors of size $m^d $. Let $\{ Q_j\}_{j=1}^d  \subset \mathbb R^{m\times q_j}$ be a collection of non-random matrix with $\|Q_j\|_{\op} \le 1$. Then 
$$ \E (  \|   ( \mathcal C^*  - \widehat { \mathcal C}  )\times_1 Q_1\times \cdots \times_d Q_d \|_F^2 )  ={\bigO}  \bigg( \frac{  \prod_{j=1}^d   q_j    }{N} \bigg) .$$
\end{lemma}
\begin{remark}
    \label{remark:non-random projection frobenius norm expectation}   Using the same notations as in \Cref{lemma:non-random projection frobenius norm expectation}, suppose in addition that   $\{ \widehat Q_j\}_{j=1}^d  \subset \mathbb R^{m\times q_j}$ is a    collection of  random matrix with $\|\widehat Q_j\|_{\op} \le 1$. Suppose in addition that $\{ \widehat Q_j\}_{j=1}^d $ are independent of $\widehat \C$. Then 
    $$ \E (  \|   ( \mathcal C^*  - \widehat { \mathcal C}  )\times_1 \widehat Q_1\times \cdots \times_d \widehat Q_d \|_F^2 )  ={\bigO}  \bigg( \frac{  \prod_{j=1}^d   q_j    }{N} \bigg) .$$ 
    The proof immediately follows from   the argument  of \Cref{lemma:non-random projection frobenius norm expectation} and independence, and is  therefore  omitted for brevity. 
\end{remark}

\begin{proof}[Proof of \Cref{lemma:non-random projection frobenius norm expectation}]
    Let $ Q_j = U_j \Sigma_j V_j^\top$ be the SVD of $Q_j$. Then 
    \begin{align}\nonumber 
      & \|   ( \mathcal C^*  - \widehat { \mathcal C}  )\times_1 Q_1\times \cdots \times_d Q_d \|_F  = \|  ( \mathcal C^*  - \widehat { \mathcal C}  )\times_1 U_1 \Sigma_1  V_ 1^\top \times \cdots \times_d U_d  \Sigma_ d V_d ^\top \| _F 
      \\ \nonumber 
       = &\| ( \mathcal C^*  - \widehat { \mathcal C}  )\times_1 U_1 \Sigma_1  \times \cdots \times_d U_d  \Sigma_ d  \| _F \le \| ( \mathcal C^*  - \widehat { \mathcal C}  )\times_1 U_1  \times \cdots \times_d U_d    \| _F  \| \Sigma_1 \| _{\op } \cdots  \| \Sigma_d \| _{\op }
       \\ \label{eq:frobenius norm for sample tensor with non-ranodm projection}
       = &  \| ( \mathcal C^*  - \widehat { \mathcal C}  )\times_1 U_1  \times \cdots \times_d U_d    \| _F  \| Q_1 \| _{\op } \cdots  \| Q_d \| _{\op },
    \end{align}
    where the second equality follows from the fact that multiplication of orthogonal matrices are norm invariant; and the last inequality follows from $\| Q_j\|_{\op} =\| \Sigma_j\|_{\op}$. For $j=1,\ldots,d$, let $ \beta_j\in \{ 1,\ldots, q_j\}  $. 
    Denote $U_{j, \mi_j,\beta_j} $ the $(\mi_j,\beta_j)$ entry of the matrix $U_j$.
    Then the $(\beta_1,\ldots, \beta_d)$ entry of $( \mathcal C^*  - \widehat { \mathcal C}  )\times_1 U_1  \times \cdots \times_d U_d  $ is 
    \begin{align*}
    &\sum_{\mi_1,\ldots, \mi_d =1 }^m U_{1,\mi_1,\beta_1} \cdots U_{d,\mi_d,\beta_d}  \langle A^*- \widehat A   ,  \phi_{\mi_1}\otimes \cdots \otimes         \phi_{\mi_d} \rangle  
   \\
   =  &\bigg \langle A^*- \widehat A   ,  \sum_{\mi_1=1 }^m  U_{1,\mi_1,\beta_1} 
 \phi_{\mi_1}\otimes \cdots \otimes       \sum_{\mi_d =1}^m U_{d,\mi_d,\beta_d}    \phi_{\mi_d}  \bigg\rangle  .
    \end{align*}   
    Since the columns of $U_j$ are orthogonal, it follows that 
    $  \bigg \|    \sum_{\mi_j=1 }^m  U_{j,\mi_j,\beta_j} 
 \phi_{\mi_j} \bigg\|_{\lt(\mathcal O)} =1   $. So 
 $$\bigg \|    \sum_{\mi_1=1 }^m  U_{1,\mi_1,\beta_1} 
 \phi_{\mi_1} \otimes \cdots \otimes       \sum_{\mi_d =1}^m U_{d,\mi_d,\beta_d}    \phi_{\mi_d} \bigg\|_{\lt(\mathcal O^d)}  =1  .$$
 By \Cref{eq:variance of one test function},
 \begin{align*} &   \E \bigg(  \bigg \langle A^*- \widehat A   ,  \sum_{\mi_1=1 }^m  U_{1,\mi_1,\beta_1} 
 \phi_{\mi_1}\otimes \cdots \otimes       \sum_{\mi_d =1}^m U_{d,\mi_d,\beta_d}    \phi_{\mi_d}  \bigg\rangle  ^2 \bigg) \le \frac{\|A^*\|_\infty}{N}.
 \end{align*}
 Consequently, 
 \begin{align*}
    & \E (  \|   ( \mathcal C^*  - \widehat { \mathcal C}  )\times_1 U_1\times \cdots \times_d U_d \|_F^2 ) 
     \\ 
     = & \E \bigg(  \sum_{\beta_1=1}^{q_1} \cdots \sum_{\beta_d=1}^{q_d} \bigg \langle A^*- \widehat A   ,  \sum_{\mi_1=1 }^m  U_{1,\mi_1,\beta_1} 
 \phi_{\mi_1}\otimes \cdots \otimes       \sum_{\mi_d =1}^m U_{d,\mi_d,\beta_d}    \phi_{\mi_d}  \bigg\rangle  ^2 \bigg)
 =
\frac{\|A^*\|_\infty q_1\cdots q_d}{N}.
 \end{align*}
 The desired result follows from \eqref{eq:frobenius norm for sample tensor with non-ranodm projection}.
\end{proof}

\begin{lemma} \label{lemma: random projection deviation bound}
 Let  $\mathcal C^* = \mathcal C (A^*{ \times_1 \mathcal P _{\mathcal M_1} \times \cdots \times_d \mathcal P _{\mathcal M_d} } )$ and $\widehat {\mathcal C}= \mathcal C ( \widehat A { \times_1 \mathcal P _{\mathcal M_1} \times \cdots  \times_d \mathcal P _{\mathcal M_d} } )$  be two tensors of size $m^d $. Let $\{ Q_j\}_{j=1}^d  \subset \mathbb R^{m\times q_j}$ be a collection of non-random matrices  with  $\|Q_j\|_{\op} \le 1$.  
 Let $\{   W^*_{j}\}_{j=1}^d  \subset \mathbb O^{m\times r_j}$ be a collection of nonrandom matrix and $\{   W^*_{j\perp}\}_{j=1}^d  \subset \mathbb O^{m\times (m-r_j)}$ be the corresponding orthogonal complements.  Let  $\{ \widehat W_{j}\}_{j=1}^d\subset \mathbb O^{m\times r_j} $ is a collection of random matrices   such that 
\begin{align}
\label{eq:frobenius deviation bound assumption}
\|   (W_{j\perp} ^*)^\top \widehat W_j\|_{\op} ^2={ \bigO}_{\p}( m^{-1})  .
\end{align}
  Then for any $1\le k\le d $, it holds that 
\begin{align}\label{eq:frobenius deviation bound general projection}
    \|   ( \mathcal C^*  - \widehat { \mathcal C}  )\times_1 Q_1\times \cdots \times_{k-1} Q_{k-1} \times _{k} \widehat W_{k} \times \cdots \times_d  \widehat W_{d }\|_F^2 
    = {\bigO}_{\p} \bigg(  \frac{ (\prod_{j=1}^{k-1}   q_j ) ( \prod_{j=k}^d  r_j)   }{N}  \bigg). 
\end{align}     
\end{lemma}
\begin{proof}
    We proceed by induction. For $k=d$, \eqref{eq:frobenius deviation bound general projection} is shown in \Cref{lemma:non-random projection frobenius norm expectation}. Suppose \eqref{eq:frobenius deviation bound general projection}  holds for $k= p +1$.  It suffices to show \eqref{eq:frobenius deviation bound general projection}  holds for $k=p$.  Observe that 
    \begin{align}\nonumber
        &\|   ( \mathcal C^*  - \widehat { \mathcal C}  )\times_1 Q_1\times \cdots \times_{p-1} Q_{p-1} \times _{p} \widehat W_{p} \times \cdots\times  _d  \widehat W_{d }\|_F^2 
        \\ \nonumber
        = &\|   ( \mathcal C^*  - \widehat { \mathcal C}  )\times_1 Q_1\times \cdots \times_{p-1} Q_{p-1} \times _{p}  \{    W_{j } ^*   (W_{j } ^*)^\top  +   W_{j\perp} ^*   (W_{j\perp} ^*)^\top  \} \widehat W_{p} \times \cdots\times  _d  \widehat W_{d }\|_F^2 
     \\  \label{eq:frobenius deviation bound term 1}
     \le  &  2 \|   ( \mathcal C^*  - \widehat { \mathcal C}  )\times_1 Q_1\times \cdots \times_{p-1} Q_{p-1} \times_{p}   W_{j } ^*   (W_{j } ^*)^\top  \widehat W_{p}   \times_{p+1} \widehat W_{p+1}  \times \cdots \times_{d}   \widehat W_{d }\|_F ^2 
     \\ \label{eq:frobenius deviation bound term 2}
     + & 2 \|   ( \mathcal C^*  - \widehat { \mathcal C}  )\times_1 Q_1\times \cdots \times_{p-1} Q_{p-1} \times_{p}  W_{j\perp} ^*   (W_{j\perp} ^*)^\top  \widehat W_{p}    \times_{p+1} \widehat W_{p+1} \times \cdots   \times_{d}\widehat W_{d }\|_F ^2  ,
    \end{align}
    where the first equality follows from the fact that 
    $  W_{j } ^*   (W_{j } ^*)^\top  +   W_{j\perp} ^*   (W_{j\perp} ^*)^\top =I_m$, the identity matrix in $\mathbb  R^{m\times m }$. 
    Note that  
    \begin{align}
      \nonumber    \eqref{eq:frobenius deviation bound term 1}  
         \le & \|   ( \mathcal C^*  - \widehat { \mathcal C}  )\times_1 Q_1\times \cdots \times_{p-1} Q_{p-1} \times_{p}  W_{j } ^*       \times_{p+1} \widehat W_{p+1} \times \cdots   \times_{d}\widehat W_{d }\|_F ^2   \| (W_{j } ^*)^\top \|_{\op}^2 \| \widehat W_{p}\|_{\op}^2  
         \\\nonumber
         =& \|   ( \mathcal C^*  - \widehat { \mathcal C}  )\times_1 Q_1\times \cdots \times_{p-1} Q_{p-1} \times_{p}  W_{j } ^*       \times_{p+1} \widehat W_{p+1} \times \cdots   \times_{d}\widehat W_{d }\|_F ^2
        \\ \label{eq:frobenius deviation bound term 11}
        = & {\bigO}_{\p}\bigg(    \frac{ (\prod_{j=1}^{p-1}   q_j )
        r_p ( \prod_{j=p+1}^d  r_j)   }{N}  \bigg)     
        = {\bigO}_{\p} \bigg(  \frac{ (\prod_{j=1}^{p-1}   q_j ) ( \prod_{j=p}^d  r_j)    }{N}  \bigg),
    \end{align}
     where the  second  equality  follows from the fact that  \eqref{eq:frobenius deviation bound general projection} holds for   $k= p  $ and   that projection matrices has operator norm bounded by 1. 
    In addition,
     \begin{align} \nonumber 
     \eqref{eq:frobenius deviation bound term 2}  
     \le & \|   ( \mathcal C^*  - \widehat { \mathcal C}  )\times_1 Q_1\times \cdots \times_{p-1} Q_{p-1} \times_{p}  W_{j\perp} ^*      \times_{p+1} \widehat W_{p+1} \times \cdots   \times_{d}\widehat W_{d }\|_F ^2 \|(W_{j\perp} ^*)^\top  \widehat W_{p} \|_{\op} ^2 
     \\ \label{eq:frobenius deviation bound term 21}  
        = & {\bigO}_{\p}\bigg(    \frac{ (\prod_{j=1}^{p-1}   q_j )
        (m-r_p)  ( \prod_{j=p+1}^d  r_j)   }{N}  \bigg)   {\bigO}_{\p}( m^{-1})   
        = {\bigO}_{\p} \bigg(  \frac{ (\prod_{j=1}^{p-1}   q_j ) ( \prod_{j=p}^d  r_j)    }{N}  \bigg) , 
        \end{align}
        where the first equality follows from the fact that  \eqref{eq:frobenius deviation bound general projection} holds for   $k= p  $ and \eqref{eq:frobenius deviation bound assumption}; and  the last equality follows from the observation that $\frac{m-r_p}{m } \le 1 \le  r_p $.
         The case of $k=p+1$ follows from \eqref{eq:frobenius deviation bound term 11}, \eqref{eq:frobenius deviation bound term 21}.
\end{proof}

\begin{lemma} \label{lemma:relate low rank matrix and low rank tensor}
Let $\{ \sigma^*_{j,\ri}\}_{\ri={\R_j} +1}^\infty$ be defined as in \eqref{eq:singular values after reshape}, and     $\xiapp$  be defined in \eqref{eq:definition of low rank bias}. Then for all $j\in \{ 1,\ldots, d\}$, it holds that $$ \sqrt{ \sum_{\rho=\R_j+ 1}^\infty \sigma^{*2}_{j, \rho} }   \le  \xiapp .  $$
\end{lemma}
\begin{proof}
Denote 
   $  x= z_j $ and  $y =(z_1,\ldots,z_{j-1}, z_{j+1},\ldots, z_d). $ 
   By the definition of $\xiapp$, there exists a sequence  $B_k \in \lt(\mathcal O^d) $ with $k \in \mathbb Z^+ $   such that $\rank_j(B_k) \le \R_j$ for $j=\{ 1,\ldots, d\}$,  and that 
   $$ \lim_{k\to \infty } \|A^*(x, y) -B_k (x, y)\|_{\lt(\mathcal O^d)} =\xiapp.$$ See \Cref{definition:function range and rank} for the definition of $\rank_j$.    Then since $\rank_j(B_k) \le \R_j$, it follows that the two-variable function 
   $B_k(x,y) $ has rank at most $\r_j$. Therefore by \eqref{remark:best low rank}, 
   \begin{align} \label{eq:low rank matrix and tensor limit bounds}
       \|A^* (x, y) - [A^*] _{\R_j}  (x,y)  \|_{\lt(\mathcal O^d)}  \le \|A^* (x, y) - B_k(x, y) \|_{\lt(\mathcal O^d ) }.
   \end{align}  
   Since $ \|A^* (x, y) - [A^*]_{\R_j} (x,y)  \|_{\lt(\mathcal O^d)}^2  = \sum_{ \rho={\R_j} +1 }^\infty \sigma^{*2}_{j,\rho}$, \eqref{eq:low rank matrix and tensor limit bounds} implies   that 
    \begin{align*} 
        \sqrt { \sum_{ \rho={\R_j} +1}^\infty \sigma^{*2}_{j,\rho}}  \le \lim_{k\to \infty }\|A^* (x, y) - B_k(x, y) \|_{\lt(\mathcal O^d ) } =\xiapp.
   \end{align*}  
\end{proof}

\begin{lemma} \label{lemma:mirky projection bound 1}
 Let  $\mathcal C^* = \mathcal C (A^* { \times_1 \mathcal P _{\mathcal M_1} \times \cdots  \times_d \mathcal P _{\mathcal M_d} } )$.     For $j\in \{1,\ldots,d \}$, let  $  U^* _j \in \mathbb R^{m\times \R_j}$ denote the matrix  with columns being the leading $\R_j$ singular vectors of $\C^*(\mi_j;\mi_1,\ldots, \mi_{j-1}, \mi_{j+1},\ldots, \mi_d)  \in \mathbb R^{m\times m^{d-1}}$.  Then
$$\|  \C^*  {\times_j ( I_m- U^*_j  U^{*\top} _j ) } \|_F
          \le 
        \xiapp +Cm^{-\alpha},    $$ 
        where $\xiapp$ is defined in \eqref{eq:definition of low rank bias}.
\end{lemma}

\begin{proof}Without loss of generality, assume $j=1$. 
Denote $x=z_1$ and $ y=(z_2,\ldots, z_d)$, and   $\C^*_1 = \C^*  (\mi_1; \mi_2, \ldots,\mi_d)$.
Then from \Cref{subsection:unfolding coefficient tensor}, it follows that 
$ \C^*_1$ is the coefficient matrix of $ \C (A^*\times_x\mathcal P_{\mathcal M_1}\times_y \mathcal P_{\mathcal M^{\otimes d-1}}),$
where
$$ \mathcal M^{\otimes d-1} =\mathcal M_2 \otimes \cdots \otimes \mathcal M_d.$$
We have 
$  
   \|  \C^*  {\times_1 ( I_m- U^*_1  U^{*\top} _1 ) } \|_F  = \| (I_m -U^*_1  U^{*\top} _1 )\C^*_1 \|_F = \|  \C_1^* - [\C_1^* ]_{\R_1}\| _F. 
$  
Note that 
\begin{align}\label{eq:projection error when applying Mirsky}
     \|  \C_1^* - [\C_1^* ]_{\R_1} \|_F   =  \sum_{ \ri =\R_1+1}^\infty  \sigma_\ri ^2 (\C_1^* )    = 
     \sum_{\ri =\R_1+1}^\infty   \sigma_\ri  ^2(   A^*\times_x\mathcal P_{\mathcal M_1}\times_y \mathcal P_{\mathcal M^{\otimes d-1}}  )
\end{align}
where the second equality follows from \Cref{corollary:singular value preserving}.  
\\
\\
We are in the position to apply \Cref{lemma:Mirsky in Hilbert space},  the  Mirsky's inequality in Hilbert space. 
Observe that 
   \begin{align*} \nonumber 
      &  \|  \C_1^* - [\C_1^* ]_{\R_1} \|_F    =
      \sqrt{ \sum_{ \ri =\R_1+1}^\infty   \sigma_\ri  ^2(   A^*\times_x\mathcal P_{\mathcal M_1}\times_y \mathcal P_{\mathcal M^{\otimes d-1}}  )}
     \\ \nonumber 
     =&\sqrt{ \sum_{ \ri =\R_1+1}^\infty   \sigma_\ri  ^2(   A^*\times_x\mathcal P_{\mathcal M_1}\times_y \mathcal P_{\mathcal M^{\otimes d-1}}  ) } \pm \sqrt{ \sum_{ \ri =\R_1+1}^\infty   \sigma_ \ri  ^2(   A^* (x,y) ) }
     \\\nonumber  
     \le & \sqrt{ \sum_{\ri =\R_1+1}^\infty  \bigg\{  \sigma_\ri   (   A^*\times_x\mathcal P_{\mathcal M_1}\times_y \mathcal P_{\mathcal M^{\otimes d-1}}  )  -     \sigma_\ri   (   A^* (x,y)    )   \bigg\}^2 }  + \sqrt{ \sum_{\ri =\R_1+1}^\infty   \sigma_\ri  ^2(   A^*  (x,y)) },
    \end{align*} 
    where the inequality follows from  the triangle inequality. 
    We have that 
    \begin{align*}
         &\sum_{\ri =\R_1+1}^\infty  \bigg\{  \sigma_\ri  (   A^*\times_x\mathcal P_{\mathcal M_1}\times_y \mathcal P_{\mathcal M^{\otimes d-1}}  )-     \sigma_\ri  (   A^* (x,y)    )   \bigg\}^2 
         \\ \le  &  \| A^* - A^*\times_x\mathcal P_{\mathcal M_1}\times_y \mathcal P_{\mathcal M^{\otimes d-1} } \|_{\lt(\mathcal O^d)} ^2  \le C  m ^{-2\alpha}  ,
    \end{align*}
    where the first inequality follows from \Cref{lemma:Mirsky in Hilbert space}; and the second inequality follows from  
\Cref{theorem: approximation theory} with $C$ being a sufficiently large constant. 
  By viewing $A^*(x, y)$ as a function of $(x, y)$, from \eqref{eq:singular values after reshape}  the singular values of $A^* (x, y)$ is $\{\sigma^*_{1, \ri}\}_{\ri=1}^\infty $. 
  So $  \sum_{\ri =\R_1+1}^\infty   \sigma_\ri  ^2(   A^*  (x,y))   =  \sum_{\ri=\R_1+1}^\infty \sigma^{*2}_{1, \ri} . $ Therefore 
   \begin{align}  \label{eq:Mirsky tensor version 12} 
        \|  \C_1^* - [\C_1^* ]_{\R_1} \|_F   \le \sqrt{C}m^{-\alpha} + \sqrt{ \sum_{\ri =\R_1+1}^\infty   \sigma_{1,\ri} ^{*2}  }   = \sqrt{C}m^{-\alpha} + \sqrt{ \sum_{\ri=\R_1+1}^\infty \sigma^{*2}_{1, \ri}   }  \le \sqrt{C}m^{-\alpha}+ \xiapp  ,
\end{align}
     where the last inequality follows from \Cref{lemma:relate low rank matrix and low rank tensor}.

\end{proof}

\begin{lemma} \label{lemma:Mirsky tensor version 2}
Let $\{ Q_j\}_{j=2}^d  \subset \mathbb R^{m\times m}$    be an arbitrary  collection  matrices such that $\| Q_j\|_{\op} \le 1$.  
Then \begin{align}\label{eq:Mirsky tensor version 2} &\| \C^*   {   \times_2     Q_2  \times \cdots \times_d    Q_d } (\mi_1; \mi_2, \ldots,\mi_d) - [ \C^*    { \times_2     Q_2  \times \cdots \times_d     Q_d  (\mi_1; \mi_2, \ldots,\mi_d)]_{\R_1} } \|_F 
\\ \nonumber 
\le &\xiapp + Cm^{-\alpha} ,
\end{align} 
where $\xiapp$ is defined in \eqref{eq:definition of low rank bias}.
\end{lemma}
\begin{proof}
Denote $\C^*_1 = \C^*  (\mi_1; \mi_2, \ldots,\mi_d)$. Then
   
    $$ \C^*   {   \times_2     Q_2  \times \cdots \times_d    Q_d } (\mi_1; \mi_2, \ldots,\mi_d) = \C^*_1 (Q_2 \otimes \cdots \otimes Q_d)(\mi_1; \mi_2, \ldots,\mi_d) ,$$
    where $ Q_2 \otimes \cdots \otimes Q_d $ is a matrix of size $ m ^{d-1}\times m^{d-1}$.
Since     $[\C^*_1]_{\R_1}$ has rank at most $\R_1$,   
    $ [\C^*_1]_{\R_1} (Q_2 \otimes \cdots \otimes Q_d) $ has rank at most $\R_1$. So 
\begin{align*}
  \eqref{eq:Mirsky tensor version 2} = &\|\C_1^* (Q_2 \otimes \cdots \otimes Q_d) - [\C_1^* (Q_2 \otimes \cdots \otimes Q_d)]_{\R_1} \|_F 
  \le   \|\C_1^* (Q_2 \otimes \cdots \otimes Q_d) - [\C_1^* ]_{\R_1} (Q_2 \otimes \cdots \otimes Q_d)\|_F 
  \\
   = &  \| (\C_1^* - [\C_1^* ]_{\R_1}) (Q_2 \otimes \cdots \otimes Q_d)  \|_F
   \le  \|  \C_1^* - [\C_1^* ]_{\R_1} \|_F \|Q_2 \otimes \cdots \otimes Q_d   \|_{\op} 
   \\
   \le & \|  \C_1^* - [\C_1^* ]_{\R_1} \|_F  
   \le  \xiapp +Cm^{-\alpha} 
\end{align*}
where the first  inequality follows from  \eqref{remark:best low rank}; the second inequality follows from $\| Q_2 \otimes \cdots \otimes Q_d  \|_{\op} \le \| Q_1\|_{\op} \cdots \|Q_d\|_{\op} \le 1 $;  and the last inequality follows from \eqref{eq:Mirsky tensor version 12}.
\end{proof}

\section{Perturbation bounds for matrices}

Let $\mathcal C\in \mathbb R^{m\times l }$ be any matrix.  Denote $k= \min\{m,l \}$. Then the singular value decomposition of $\mathcal C$ satisfies 
$$ \mathcal C= U \Sigma V ^\top .$$
Here $ U \in \mathbb O^{m\times k} $ with columns being the left singular vectors of $\mathcal C$; $ V \in \mathbb O^{l\times k} $ with columns being the right singular vectors of $\mathcal C$; and $\Sigma \in \mathbb R^{k\times k}$ is a diagonal matrix with diagonal entries $\sigma_1(\mathcal C) \ge \sigma_2(\mathcal C) \ge \ldots \ge 0$ being the singular values of $\mathcal C$. 
\\
Let $\R\le k$ be a positive integer. Suppose that $ U_\R \in \mathbb O^{m\times \R}  , V_\R \in \mathbb O^{l\times \R} $,    $ U_{\R  \perp}\in \mathbb O^{m\times (k- \R)}$, and $  V _{\R  \perp} \in \mathbb O^{l\times (k- \R)} $ are  such that 
$$ [ U_\R,U_{\R  \perp} ] =U  \quad \text{and} \quad [ V_\R,V_{\R  \perp} ] =V .$$
Then  $ U_\R   $ has columns   being the first $ r$ leading singular vectors of $\mathcal C$. Let $\Sigma_\R \in \mathbb R^{\R\times \R}$ be the diagonal matrix whose diagonal entries are $\{ \sigma_1(\mathcal C) ,\ldots, \sigma_\R(\mathcal C)\}$, and $\Sigma_{\R\perp} \in \mathbb R^{{k-\R}\times {k- \R}}$ be the diagonal matrix whose diagonal entries are $\{ \sigma_{\R+1}(\mathcal C) ,\ldots, \sigma_k(\mathcal C)\}$. Then 
\begin{align}\label{eq:matrix SVD general form} \mathcal C =U _\R \Sigma_ \R V_\R ^\top  +U _{\R\perp}  \Sigma_ {\R\perp} V _{\R\perp}  ^\top .
\end{align}
 Denote 
 $$ [\mathcal C]_\R = U _\R \Sigma_ \R V_\R ^\top.$$ 
Then $[\mathcal C]_\R$ is the best rank $\R$ approximation of $\mathcal C$ 
 in the sense that 
\begin{align*} \|\mathcal C - [\mathcal C  ] _r\|_{\op} = &  \min_{ B  \in \mathbb R^{m\times l},  \text{ rank}(B) \le r  } \|  \mathcal C   -B\|_{\op } =\sigma_{\R+1} (\mathcal C) ,
\\
\|\mathcal C - [\mathcal C  ] _r\|_{F} = &  \min_{ B  \in \mathbb R^{m\times l},  \text{ rank}(B) \le r  } \|  \mathcal C   -B\|_{F } =\sigma_{\R+1} (\mathcal C) .
\end{align*}
\begin{theorem}
    \label{theorem:matrix singular value perturbation}
Let $\mathcal C $ and $\mathcal C ' $ be two matrices of  dimensions $n_1\times n_2$. Denote $n=\min\{ n_1,n_2\}$. Let 
$ \{ \sigma_j (\mathcal C)  \}_{j=1}^n    $  be the singular values of $\mathcal C$ in the decreasing order,  and $ \{ \sigma_j (\mathcal C')  \}_{j=1}^n    $   be the singular values of $\mathcal C'$ in the decreasing order.
\\
\\
{\bf (a)}{(Weyl's inequality)} For  all $ j \in \{ 1,\ldots, n\} $, it holds that 
$ |\sigma_j (\mathcal C )  - \sigma_j (\mathcal C ')  | \le \| \mathcal C - \mathcal C '\|_{\op } $  . 
  \\
  \\
  {\bf (b)}{(Mirsky's inequality)} It holds that 
  $$ \sum_{j=1}^ n  ( \sigma_{j} (\mathcal C) -\sigma_{j}(\mathcal C  ') )^2 \le  \|\mathcal C - \mathcal C ' \|_F ^2  =  \sum_{j=1}^ n    \sigma_{j} ^2 (\mathcal C  - \mathcal C ')   
   ,$$
  where $\sigma_{j}   ( \mathcal C  -\mathcal C  ')$ denotes the $j$-th singular value of $\mathcal C  -\mathcal C ' $.
\end{theorem}
\begin{proof}
    See \cite{horn2012matrix} for the proofs of both results. 
\end{proof}

\begin{lemma} \label{eq:projection error for approximately low rank matrix}
    Let $\mathcal A  $, $\mathcal B$ and $\mathcal C$ be three matrices of size $m\times l $ and that  
     $\mathcal C=  \mathcal A + \mathcal B$. Suppose for $r\le \min\{ m,l\}$, the SVD of $\mathcal C$  satisfies
     $$\mathcal C = U_{r} \Sigma_r  V_r +  U_{r \perp } \Sigma_{r \perp }  V_{r \perp }   $$
     as in \eqref{eq:matrix SVD general form}.
     Then 
     $$\|   U_{r\perp}  U_{r\perp} ^\top \mathcal A \|_F   \le 2    \|\mathcal B\|_{F}    + \| \mathcal A - [\mathcal A]_r\|_F    .$$
\end{lemma}
\begin{proof}
    Note that 
    \begin{align}
      \|   U_{r\perp}  U_{r\perp} ^\top \mathcal A \|_F   =  & \|   U_{r\perp}  U_{r\perp} ^\top ( \mathcal C -\mathcal B ) \|_F  
      \le \|   U_{r\perp}  U_{r\perp} ^\top \mathcal C   \|_F  +\|  U_{r\perp}  U_{r\perp} ^\top \mathcal B   \|_F \label{eq:projection bounds approximate low rank}.
    \end{align}
   Observe  that  
    \begin{align*}
      \|  U_{r\perp}  U_{r\perp} ^\top \mathcal C   \|_F  = & \|  U_{r\perp}  U_{r\perp} ^\top \Sigma_{r \perp }  V_{r \perp }    \|  _F  = \|  U_{r\perp} ^\top \Sigma_{r \perp }  V_{r \perp }    \|  _F  =   \| \mathcal   C - [  \mathcal C]_r\|_F 
      \\ 
      \le & \| \mathcal   C - [  \mathcal A]_r\|_F  \le 
        \| \mathcal  C -    \mathcal A \|_F +\| \mathcal  A - [  \mathcal A]_r\|_F  =   \| \mathcal  B \|_F +\| \mathcal  A - [  \mathcal A]_r\|_F   ,
    \end{align*}
    where the first equality follows from the fact that  $ \mathcal U_{r\perp} ^\top \mathcal U_{r } =0 $; and the second inequality follows from \Cref{lemma:matrix norm invariance}. 
    In addition,
    $$\|   U_{r\perp}  U_{r\perp} ^\top \mathcal B   \|_F \le  \|  U_{r\perp} \|_{\op }\|  U_{r\perp} ^\top   \|_{\op } \| \mathcal B\| _F  =  \|\mathcal  B\| _F, $$
    where the last equality follows from the fact that $ \| U_{r\perp} \|_{\op } =1$. The desired result directly follows from \eqref{eq:projection bounds approximate low rank}.
\end{proof}

\begin{lemma} \label{eq:subspace consistency to matrix consistency}
    Let $\mathcal C \in \mathbb R^{m\times n}$ and $V, U \in \mathbb O^{m\times r}$.
     Denote $\mathcal P_V = VV^\top $ and $\mathcal P_U = UU^\top $. Then 
    $$ \|\mathcal P_V \mathcal C\|_F \le    \sigma_{\min }^{-1}(U^\top V)  \|\mathcal P_U \mathcal P_V \C\|_F .
    $$
\end{lemma}
\begin{proof}
    Note that 
    \begin{align*}
       \|\mathcal P_U \mathcal P_V \C\|_F = & \| UU^\top  VV^\top  \mathcal C  \| _F  = \|  U ^\top VV^\top \C\|_F 
        \\
        \ge & \sigma_{\min} (U^\top V) \| V^\top  \mathcal C  \|_F = \sigma_{\min} (U^\top V) \|  V V^\top   \mathcal C  \|_F =  \sigma_{\min} (U^\top V) \|\mathcal P_V \mathcal C\|_F ,
    \end{align*}
    where   the second equality follows from \Cref{lemma:matrix norm invariance}; the inequality follows from \Cref{lemma:frobenius norm bounded by minimal singular values}; and the third   equality follows from \Cref{lemma:matrix norm invariance}.
    \end{proof}

 \begin{lemma} \label{lemma:matrix norm invariance} Suppose $U\in \mathbb O^{m\times r} $ and $M\in \mathbb R^{m\times n}$. Then 
$\| UU^\top M\|_F = \| U^\top M\|_F. $ 
\end{lemma}
\begin{proof}  Observe that 
$$\| UU^\top M\|_F ^2 = \Tr ( UU^\top M M^\top UU^\top )= \Tr(U^\top U U^\top M M^\top U  )  = \Tr(  U^\top M M^\top U  ) = \| U^\top M\|_F ^2 $$
\end{proof}

\begin{lemma} \label{lemma:frobenius norm bounded by minimal singular values}
    Let $ \mathcal A   \in \mathbb R^{m \times m  }$  and $\mathcal B \in \mathbb R^{m \times n  }$ be  two matrices. Then 
    $$ \sigma_{\min}(\mathcal A)\| \mathcal B\|_F \le  \|\mathcal A\mathcal B\|_F.$$
\end{lemma}
\begin{proof} 
    Denote   the SVD of $ A$ as 
     $\mathcal A= \sum_{j=1}^m \sigma_j(\mathcal A) v_j w_j^\top  $. Then 
     \begin{align*}
         \|\mathcal A \mathcal B\|_F^2  = \|  \sum_{j=1}^m v_j  \{ \sigma_j(A) w_j^\top \mathcal B \} \|_F^2 =\sum_{j=1}^m   \| \sigma_j(\mathcal A) w_j^\top  \mathcal B \|_2^2  \ge \sigma^2_{min}(\mathcal A)\sum_{j=1}^m   \|  w_j^\top \mathcal B \|_2^2 = \sigma_{min}^2 (\mathcal A) \|   \mathcal  B \|_F^2,
     \end{align*}
     where the second equality follows from the fact that $\{v_j\}_{j=1}^m$ is a set of complete basis vectors; and the last equality follows from the fact that $\{w_j\}_{j=1}^m$ is a set of complete basis vectors.
\end{proof}

\begin{theorem}[Wedin] \label{theorem:Wedin}
Suppose without loss of generality that $n_1\le n_2$.
Let $M =M^*+E, M^*$ be two matrices 	in $\mathbb R^{n_1\times n_2}$  whose svd are given respectively by
\begin{align*}
M^*= \sum_{i=1}^{n_1}\sigma^*_i u_i^* {v_i^*} ^\top\quad \text{and} \quad 
M=  \sum_{i=1}^{n_1}\sigma _i u_i  {v_i } ^\top
\end{align*}
where $\sigma_1^*\ge \cdots \ge \sigma_{n_1}^* $ and $\sigma_1 \ge \cdots \ge \sigma_{n_1}  $. For any $r\le n_1 $, let
\begin{align*}
 &\Sigma^* =\text{diag}([\sigma_1^*, \cdots, \sigma_r^*]) \in \mathbb R^{r\times r },\quad   U^*  =[u_1^*, \cdots, u_r ^*] \in \mathbb O^{n_1\times r } , \quad  V=[v_1^*, \cdots, v_r^* ] \in \mathbb O^{ n_2 \times r},
\\
& \Sigma  =\text{diag}([\sigma_1, \cdots, \sigma_r]) \in \mathbb R^{r\times r },\quad   U =[u_1, \cdots, u_r ] \in \mathbb \mathbb O^{n_1\times r } , \quad  V=[v_1, \cdots, v_r ] \in \mathbb \mathbb O^{n_2\times r }.
\end{align*}
Denote $\mathcal P_{U^*} =  U^*  U^{*^\top}    $
and $\mathcal P_{U} = U U^\top  $. 
If $\|E\|_{\op} < \sigma_r^*-\sigma_{r+1}^* $, then  
\begin{align}\label{eq:davis khan} \|\mathcal P_{U^*}   -\mathcal P_{U }\|_{\op}    \le \frac{ \sqrt 2  \| E\|_{\op }   }{\sigma_r^*-\sigma_{r+1}^*-\|E \|_{\op} }   .
\end{align}

\end{theorem}
\begin{proof}
This is the well-known Wedin's theorem. See  section 2 of \cite{chen2021spectral} for proofs.   
\end{proof} 

\begin{corollary}[Wedin] \label{corollary:Wedin}
Suppose all  the  conditions   in \Cref{theorem:Wedin} hold. 
Suppose in addition that both $M^*$ and $M $ are exactly rank $\R$.  Let $W^*  \in \mathbb O^{n_1\times \R}$ be such that the columns of $W^*$ spans the column space of $M^*$,
and  $W    \in \mathbb O^{n_1\times \R}$ be such that the columns of $W $ spans the column space of $M$, Then 
\begin{align*}  \| W^* W^{*\top}  - WW^\top\|_{\op}    \le \frac{ \sqrt 2  \| E\|_{\op }   }{\sigma_r^* -\|E \|_{\op} }   .
\end{align*}
\begin{proof}
    It suffices to observe that $ W^* W^{*\top}  =  U^* U^{*\top} $  and $ W  W^{ \top}  =  U  U^{ \top} $, and $ \sigma^*_{r+1}=0$.
\end{proof}

\end{corollary}

\begin{corollary}  \label{theorem:Wedin Hilbert space} Let $ \mathcal W$ and $ \mathcal W'$ 
be two Hilbert spaces. 
Let $M$ and $E$ be two finite rank operators on $ \mathcal W\otimes \mathcal W' $ and denote $M =M^*+E $. Let the SVD of $M^* $ and $M$ are given respectively by
\begin{align*}
M^*= \sum_{i=1}^{\R_1}\sigma^*_i u_i^* \otimes  v_i^* \quad \text{and} \quad 
M=  \sum_{ i =1}^{\R_2}\sigma _i u_i   \otimes  v_i   
\end{align*}
where $\sigma_1^*\ge \cdots \ge \sigma_{\R_1}^* $ and $\sigma_1 \ge \cdots \ge \sigma_{\R_2}  $.   For $   r  \le min\{\R_1,\R_2 \}$, denote 
\begin{align*}
U^*=\Span(\{ u_i^* \}_{i=1}^r  )  \quad \text{and} \quad U =\Span(\{ u_ i  \}_{i=1}^r  )  
\end{align*}
Let $\mathcal P_{U^*} $ to be projection operator  from $\mathcal W$ to $U^*$, and  $\mathcal P_{U } $ to be projection operator from $\mathcal W$ to $U $.
If $\|E\|_{\op} < \sigma_r^*-\sigma_{r+1}^* $, then  
$$  \|\mathcal P_{U^*}   -\mathcal P_{U }\|_{\op}    \le \frac{ \sqrt 2  \|E\|_{\op}  }{\sigma_r^*-\sigma_{r+1}^*-\|E \|_{\op} } .$$

\end{corollary} 
\begin{proof}
Let $$ \mathcal S= \Span (\{ u_i^* \}_{i=1}^{\R_1} ,\{ u_i  \}_{i=1}^{\R_2} ) \quad \text{and} \quad \mathcal S'= \Span (\{ v_i^* \}_{r=1}^{\R_1} ,\{ v_i  \}_{i=1}^{\R_2} ) . $$
Then $M^*$ and $M$ can be viewed as finite-dimensional matrices on $\mathcal S\otimes \mathcal S' $. Since $\mathcal P_{U^*} = \sum_{i=1}^r u^*_i \otimes u^*_i  $
and $\mathcal P_{U} = \sum_{i=1}^r u_i \otimes u_i  $, the desired result follows from \Cref{theorem:Wedin}.

\end{proof}

\section{Perturbation bounds in function spaces}\label{appendix:Perturbation bounds}

\begin{lemma} \label{min max principle}
Let $A$ and $B$ be two compact self-adjoint operators on a Hilbert space $\mathcal W$.  Denote $\lambda_k(A)$ and $\lambda_k(A+B)$ to be the k-th eigenvalue of $ A$ and $A+B$ respectively. Then 
$$|\lambda_k(A+B) - \lambda _k(A) |\le \|B\|_{\op} . $$

\end{lemma}
\begin{proof}
By the min-max principle,  for any    compact self-adjoint operators $H$
and   any $S_k\subset \mathcal W$ being a $k$-dimensional  subspace
$$ \max_{ S_k } \min_{x \in S_k, \|x\|_\mathcal W =1 } H[x, x] =\lambda_k(H) .$$
It follows that 
\begin{align*}
\lambda_k(A+B) =&\max_{ S_k } \min_{x \in S_k, \|x\|_\mathcal W=1 } (A+B)[x,x]
\\
\le &\max_{ S_k } \min_{x \in S_k, \|x\|_\mathcal W=1 }  A  [x,x] +\|B\|_{\op}\|x\|_\mathcal W^2
\\
=& \lambda_k(A) + \|B\|_{\op}.
\end{align*}
The other direction follows from symmetry.
\end{proof}

\begin{lemma} \label{lemma:SVD pertubation for operators}
Let $\mathcal W $ and $\mathcal W'$ be two separable Hilbert spaces.
Suppose $A$ and $B$ are two compact operators from  $ \mathcal W \otimes \mathcal W'\to \mathbb R $.  Then 
$$\big|   \sigma_k(A+B)  -\sigma_k(A)  \big| \le \|B\|_{\op}.$$
\begin{proof}
Let $\{\phi_i\}_{i=1}^\infty $ and $\{\phi_i'\}_{i=1}^\infty $ be the orthogonal basis of $\mathcal W $ and $\mathcal W'$. 
Let 
$$\mathcal W_j =\Span (\{\phi_i\}_{i=1}^j) \quad \text{and}\quad \mathcal W_j' =\Span (\{\phi_i'\}_{i=1}^j) . $$ Denote 
$$ A_j =A\times \mathcal{P}_{\mathcal W_j} \times \mathcal{P}_{\mathcal W_j'}  \quad \text{and} \quad (A+B)_j =(A+B)\times \mathcal{P}_{\mathcal W_j} \times \mathcal{P}_{\mathcal W_j'}  .$$
Note that $(A+B)_j=A_j+B_j$ due to linearity. 
Since $A$ and $A+B$ are    compact, 
$$\lim_{j\to \infty} \|A-A_j\|_F = 0 \quad \text{and} \quad  \lim_{j\to \infty} \|(A+B)-(A+B)_j\|_F = 0.$$
Then  $ AA^\top  $ and $ A_j A_j^\top$ are two compact self-adjoint operators on $\mathcal W$ and that 
$$\lim_{j\to \infty }\|AA^\top  - A_j A_j^\top \|_F =0 . $$
By \Cref{min max principle}, $\lim_{j\to \infty } \lambda_k(A_jA_j^\top) = \lambda_k(AA^\top). $  Observe that Since $ \sigma_k(A_j) $ and $\sigma_k(A) $ are both positive,  that  $  \lambda_k(A A^\top) = \sigma_k^2(A
)  $, and that   $  \lambda_k(A_j A_j^\top) = \sigma_k^2(A_j
)  $.  It follows  that 
 $$\lim_{j\to \infty }  \sigma_k(A_j )  =  \sigma_k(A ) . $$ 
Similarly
$  \lim_{j\to \infty }  \sigma_k((A+B)_j )  =  \sigma_k(A +B) . $  
By the   finite dimensional SVD perturbation theory  (see Theorem 3.3.16 on page 178 of \cite{horn1994topics}),  it follows that 
$$ \big|  \sigma_k((A+B)_j )   -  \sigma_k(A_j)  \big| 
  \le \|B_j\|_{\op} \le \|B\|_{\op}.$$
The desired result follows by taking the limit as $j\to \infty$. 

\end{proof}

\end{lemma}

\begin{lemma}[Mirsky in Hilbert space] \label{lemma:Mirsky in Hilbert space}
    Suppose $A$ and $B$ are two compact operators in $\mathcal W\otimes \mathcal W'$, where $\mathcal W$ and $\mathcal W'$ are two separable Hilbert spaces. Let 
$ \{ \sigma_k (A)   \}_{k=1}^ \infty    $  be the singular values of $\mathcal A$ in the decreasing order,  and $ \{ \sigma_k (\mathcal B)  \}_{k=1}^\infty    $   be the singular values of $\mathcal B$ in the decreasing order. Then 
$$\sum_{k=1}^\infty (\sigma_k(A) -\sigma_k(B))^2 \le \|A - B \|_F ^2  =  \sum_{k=1}^ \infty   \sigma_{k} ^2 (A  - B)     . $$
\end{lemma}
\begin{proof}
Let $\{\phi_i\}_{i=1}^\infty $ and $\{\phi_i'\}_{i=1}^\infty $ be the orthogonal basis of $\mathcal W $ and $\mathcal W'$. 
Let 
$$\mathcal W_j =\Span (\{\phi_i\}_{i=1}^j) \quad \text{and}\quad \mathcal W_j' =\Span (\{\phi_i'\}_{i=1}^j) . $$ Denote 
$$ A_j =A\times \mathcal{P}_{\mathcal W_j} \times \mathcal{P}_{\mathcal W_j'}  \quad \text{and} \quad  B _j = B\times \mathcal{P}_{\mathcal W_j} \times \mathcal{P}_{\mathcal W_j'}  .$$ Since both $A$ and $B$ are compact, let $ n$ be sufficiently large such that for all $j\ge n$,
$$\| A-A_j \| _F \le \epsilon\quad \text{and} \quad \| B-B_j\|_F \le \epsilon.  $$ Then 
\begin{align*}
    &\sqrt{\sum_{k=1}^\infty (\sigma_k(A) -\sigma_k(B))^2}
    = \sqrt{\sum_{k=1}^\infty (\sigma_k(A) - \sigma_k(A_j) + \sigma_k(A_j) - \sigma_k(B_j) +\sigma_k(B_j) - \sigma_k(B))^2} 
    \\
    \le & \sqrt{\sum_{k=1}^\infty (\sigma_k(A) - \sigma_k(A_j) )^2 }+ \sqrt{\sum_{k=1}^\infty  (\sigma_k(A_j) - \sigma_k(B_j) )^2} + \sqrt{\sum_{k=1}^\infty (\sigma_k(B_j) - \sigma_k(B))^2} .
\end{align*}
Here 
$$ \sqrt{\sum_{k=1}^\infty (\sigma_k(A) - \sigma_k(A_j) )^2 } = \|A-A_j\|_F \le \epsilon\quad \text{and} \quad \sqrt{\sum_{k=1}^\infty (\sigma_k(B_j) - \sigma_k(B))^2}= \|B-B_j\|_F \le \epsilon  .$$
In addition, both $A_j$ and $B_j$ can be viewed as finite-dimensional matrices of size $j\times j$. So for $k>j$,
$$\sigma_{k}(A_j) = \sigma_{k}(B_j) =0.$$ 
By finite-dimensional Mirsky in \Cref{theorem:matrix singular value perturbation},
\begin{align*}
    & \sqrt{\sum_{k=1}^\infty  (\sigma_k(A_j) - \sigma_k(B_j) )^2}  =  \sqrt{\sum_{k=1}^j  (\sigma_k(A_j) - \sigma_k(B_j) )^2} 
    \le  \sqrt{\sum_{k=1}^j  (\sigma_k(A_j-B_j)   )^2} 
    \\
    = &\|A_j-B_j\|_F  \le \|A_j-A\|_F +\|A -B\|_F + \|B-B_j\|_F  \le 2\epsilon + \|A-B\|_F. 
\end{align*} 
Therefore 
$$\sqrt{\sum_{k=1}^\infty (\sigma_k(A) -\sigma_k(B))^2} \le 4\epsilon + \|A-B\|_F.$$
   The desired result follows because $\epsilon$ is arbitrary.
\end{proof}

\section{Approximation theories for smooth functions} 
\label{section: justify approximation theory}
In this section, we justify \Cref{theorem: approximation theory}.
 \subsection{Reproducing kernel Hilbert space basis}
 \label{subsection:RKHS}
 Let $\mathcal O$ be a measurable set in $\mathbb R$.  
For $x, y \in \mathcal O$,
let $\K :\mathcal O \times \mathcal O \to \mathbb R $ be a   kernel function  such that 
\begin{align} \label{eq:RKHS kernel definition}\mathbb K(x,y) = \sum_{\ki=1}^  \infty   \sigmak_\ki \phik_\ki(x)\phik_\ki(y) ,
\end{align}
where   $ \{ \sigmak_\ki\}_{\ki=1}^\infty  \subset \mathbb R^+ \cup \{ 0 \}$ , and  $\{ \phik_\ki\}_{\ki=1}^\infty  $ is a collection of basis functions in $\lt(\mathcal O)$.  
\\
\\
The 
   reproducing kernel Hilbert space generated by $\K$   is 
  \begin{align}  \h (\mathbb K)  = \bigg \{ f \in \lt([0,1])   : \|f\|_{\h(\K)} ^2  = \sum_{\ki=1}^\infty   (\lambda _\ki^\K)^{-1}\langle f, \phik_\ki\rangle^2    <\infty  \bigg \}. 
  \label{eq:RKHS}
  \end{align}
  For any functions $f, g\in \h(\mathbb K)$, the inner product in $\h(\mathbb K)$ is given by  
  $$ \langle f ,g \rangle_{\h(\mathbb K) } = \sum_{\ki=1}^ \infty    (\lambda_\ki^\K)^{-1}\langle f, \phik_\ki\rangle  \langle g, \phik_\ki\rangle .  $$
Denote $  \Theta^\mathbb K_\ki =  (\lambda^\K_\ki )^{-1/2}  \phik_\ki.$ Then $\{  \Theta ^\mathbb K_\ki\}_{\ki=1}^\infty$ are the orthonormal basis functions in $ \h(\K)$ as we have that
$$ \langle \Theta _{\ki_1}^\mathbb K  ,\Theta _{\ki_2}^\mathbb K \rangle_{\h(\mathbb K) }  =\begin{cases} 
1, &\text{if } {\ki_1} = {\ki_2};
\\
0, &\text{if } {\ki_1} \not = {\ki_2}.
\end{cases}  $$ 
 an that $$\|f\|_{ \h(\K) } ^2  = \sum_{\ki=1}^{\infty}    (\lambda_\ki^\K)^{-1 }\langle f, \phik_\ki\rangle^2 = \sum_{\ki=1}^{\infty}     \langle f, \Theta _\ki^\mathbb K \rangle^2 .$$  
 Define the tensor product space  
 $$ \underbrace{   \h(\K) \otimes \cdots \otimes \h(\K)  }_{d \text{ copies} }  =    \h(\K^{\otimes d} )  . $$
 The induced Frobenius norm in   $\h(\K^{\otimes d} ) $  is 
\begin{align}\label{eq:equivalent sobolev norm multidimensional 2} \|A\|_{ \h(\K^{\otimes d} )    }^2  = \sum_{\ki_1, \ldots, \ki_d=1}^\infty      \langle A, \Theta ^\mathbb K _{\ki_1} \otimes \cdots \otimes   \Theta _{\ki_d}^\mathbb K  \rangle   ^2  = \sum_{\ki_1, \ldots, \ki_d=1}^\infty  ( \lambda _{\ki_1} ^\K  \ldots   \lambda _{\ki_d} ^\K)^{-1}      \langle A , \phik_{\ki_1} \otimes \cdots \otimes    \phik_{\ki_d}  \rangle   ^2 , 
\end{align} 
 where $\langle A , \phik_{\ki_1} \otimes \cdots \otimes    \phik_{\ki_d}  \rangle  $ is defined  in \eqref{eq:tensor of univariate functions}.   
The following lemma shows that the space $\h(\K^{\otimes d} )  $ is naturally motivated by multidimensional Sobolev spaces.

\begin{lemma}\label{lemma:sobolev space tensor}
Let $\mathcal O=[0,1]$. With  $ \sigmak_\ki\asymp \ki^{-2\alpha } $   and suitable choices of $\{ \phik_\ki\}_{\ki=1}^\infty $, it holds that    
$$\h(\K^{\otimes d} ) =W^{\alpha}_2([0,1]^d) .$$   
\end{lemma} 
\begin{proof}
Let $\mathcal O=[0,1]$.
When $d=1$, it is a classical Functional analysis result that with  $ \sigmak_\ki\asymp \ki^{-2\alpha } $   and suitable choices of $\{ \phik_\ki\}_{\ki=1}^\infty $,
$$  \mathcal H(\K)  =W^{\alpha}_2([0,1] ) .$$   
We refer interested readers to Chapter 12 of  \cite{wainwright2019high} for more details. 
  In general, it is well-known in functional analysis that for $\Omega_1 \subset \mathbb R^{d_1}$ and $\Omega_2 \subset \mathbb R^{d_2}$, then 
 $$ W^{\alpha}_2( \Omega_1 ) \otimes W^{\alpha}_2( \Omega_2 ) =W^{\alpha}_2( \Omega_1 \times \Omega_2) . $$
 Therefore by induction 
 \begin{align*}
\h(\K^{\otimes d} )  =  \mathcal H(\K^{\otimes d-1}) \otimes   \mathcal H(\K)  = W^{\alpha}_2([0,1]^{d-1} ) \otimes W^{\alpha}_2([0,1] ) =W^{\alpha}_2([0,1]^d).
 \end{align*}
\end{proof}

 Let  $d_1+d_2 =d$.
 In what follows, we show    \Cref{assume: baised in projected space in operator norm} holds when 
\begin{align*}  
\mathcal M =\text{span} \bigg\{\phik_{\mi_1}(z_1)  \cdots  \phik_{\mi_{d_1}}(z_{d_1}) \bigg \}_{\mi_1, \ldots, \mi_{d_1}=1}^\m \quad \text{and} \quad 
\mathcal N = \text{span}\bigg\{  \phik_{\li_1 }(z_{d_1+1})\cdots  \phik_{\li_{d_2}}(z_{d })\bigg\}_{\li_1, \ldots, \li_{d_2}=1}^\l.
\end{align*}  
\begin{lemma} \label{lemma:bias} 
Let $\K$ be a kernel in the form of \eqref{eq:RKHS kernel definition}. 
Suppose that $\lambda^\K _\ki\asymp \ki^{-2\alpha }$, and that 
$A:\mathcal O ^{d}  \to \mathbb R $ is such that $ \|A\|_{ \{ \h(\K) \}^{\otimes d}  } <\infty$. Then for any two positive integers 
 $\m, \l \in \mathbb Z^+$,  it holds that 
 \begin{align}\label{eq:bias 1}\|A-A {\times_x \mathcal P_\mathcal M \times _y \mathcal P_\mathcal N} \|^2_{\lt(\mathcal O ^{d}) }    \le C ( \m^{-2\alpha}+ \l^{-2\alpha })\|A\|_{    \h(\K^{\otimes d} )    }^2 , 
\end{align}
where $C$ is some absolute constant. 
Consequently
\begin{align}\label{eq:bias 2} \|A-A { \times _y \mathcal P_\mathcal N} \|^2_{\lt( \mathcal O ^{d} ) }    \le C    \l^{-2\alpha } \|A\|_{  \h(\K^{\otimes d} )   }^2   
\quad \text{and} 
\\\label{eq:bias 3}
\| A{ \times _y \mathcal P_\mathcal N} -A {\times_x \mathcal P_\mathcal M \times _y \mathcal P_\mathcal N}\|^2_{\lt( \mathcal O ^{d} ) }    \le  C  \m^{-2\alpha } \|A\|_{  \h(\K^{\otimes d} )   }^2 .  
\end{align}

\end{lemma}

\begin{proof} Since  $\lambda^\K _\ki \asymp \ki^{-2\alpha }$, without loss of generality,
throughout the proof we assume that 
\begin{align}\label{eq:eigenvalue assumption} \lambda^\K _\ki =  \ki^{-2\alpha },
\end{align}
as otherwise all of our analysis still holds up to an absolute constant. 
Observe  that 
\begin{align*} & \langle A{\times_x \mathcal P_\mathcal M \times _y \mathcal P_\mathcal N}\rangle  ) ,  (\phik_{\mi _1} \otimes  \cdots \otimes  \phik_{\mi_{d_1}}) \otimes  ( \phik_{\li_1 } \otimes  \cdots  \otimes \phik_{\li_{d_2}}   ) \rangle 
\\=&
\begin{cases}
\langle A ,  (\phik_{\mi _1} \otimes  \cdots \otimes  \phik_{\mi_{d_1}}) \otimes  ( \phik_{\li_1 } \otimes  \cdots  \otimes \phik_{\li_{d_2}}   ) \rangle  , &\text{if } 1\le \mi_1, \ldots, \mi_{d_1} \le \m$ and $ 1\le \li_1, \ldots,\li_{d_2} \le \l, 
\\
 0, &\text{otherwise.}
\end{cases}
\end{align*}
Then
\begin{align*}
\|A-A{\times_x \mathcal P_\mathcal M \times _y \mathcal P_\mathcal N} \|^2_{\lt( \mathcal O ^{d})  } =& \sum_{\mi_1 =\m+1}^\infty \sum_{\mi_2, \ldots, \mi_{d_1}=1}^\m       \sum_{\li_1, \ldots, \li_{d_2}=1}^\l    
\langle A ,  (\phik_{\mi _1} \otimes  \cdots \otimes  \phik_{\mi_{d_1}}) \otimes  ( \phik_{\li_1 } \otimes  \cdots  \otimes \phik_{\li_{d_2}}   ) \rangle ^2  
\\
+ & \ldots
\\
+ & \sum_{\mi_1,\ldots, \mi_{d_1}=1 }^\m  \sum_{\li_1,  \ldots, \li_{d_2-1}=1}^\l  \sum_{\li_{d_2} =\l+1}^\infty     
\langle A ,  (\phik_{\mi _1} \otimes  \cdots \otimes  \phik_{\mi_{d_1}}) \otimes  ( \phik_{\li_1 } \otimes  \cdots  \otimes \phik_{\li_{d_2}}   ) \rangle ^2  .
\end{align*}
Observe that
\begin{align*}
 & \sum_{\mi_1 =\m+1}^\infty \sum_{\mi_2, \ldots, \mi_{d_1}=1}^\m       \sum_{\li_1, \ldots, \li_{d_2}=1}^\l    
\langle A ,  (\phik_{\mi _1} \otimes  \cdots \otimes  \phik_{\mi_{d_1}}) \otimes  ( \phik_{\li_1 } \otimes  \cdots  \otimes \phik_{\li_{d_2}}   ) \rangle ^2  
\\
\le &     \sum_{\mi_1 =\m+1}^\infty \m^{-2\alpha }   \mi_{1}^{2\alpha}    \sum_{\mi_2, \ldots, \mi_{d_1}=1}^\m    \sum_{\li_1, \ldots, \li_{d_2}=1}^\l   
  \langle A ,  (\phik_{\mi _1} \otimes  \cdots \otimes  \phik_{\mi_{d_1}}) \otimes  ( \phik_{\li_1 } \otimes  \cdots  \otimes \phik_{\li_{d_2}}   ) \rangle ^2   
\\
\le &   \m^{-2\alpha }  \sum_{\mi_1 =\m+1}^\infty \sum_{\mi_2, \ldots, \mi_{d_1}=1}^\m    \sum_{\li_1, \ldots, \li_{d_2}=1}^\l  
 \mi_1 ^{2\alpha }\cdots \mi_{d_1} ^{2\alpha } \li_1^{2\alpha } \cdots \li_{d_2}^{2\alpha } \langle A ,  (\phik_{\mi _1} \otimes  \cdots \otimes  \phik_{\mi_{d_1}}) \otimes  ( \phik_{\li_1 } \otimes  \cdots  \otimes \phik_{\li_{d_2}}   ) \rangle ^2   
 \\
 \le &   \m^{-2\alpha }  \sum_{\mi_1 =1}^\infty \sum_{\mi_2, \ldots, \mi_{d_1}=1}^\infty    \sum_{\li_1, \ldots, \li_{d_2}=1}^\infty   
 \mi_1 ^{2\alpha }\cdots \mi_{d_1} ^{2\alpha } \li_1^{2\alpha } \cdots \li_{d_2}^{2\alpha } \langle A ,  (\phik_{\mi _1} \otimes  \cdots \otimes  \phik_{\mi_{d_1}}) \otimes  ( \phik_{\li_1 } \otimes  \cdots  \otimes \phik_{\li_{d_2}}   ) \rangle ^2   
 \\
 =  &    \m^{-2\alpha }\|A\|_{   \h(\K^{\otimes d-1} )     }^2 
\end{align*}
where the first inequality holds because $\mi_1\ge \m+1 \ge \m $ and the last equality follows from \eqref{eq:equivalent sobolev norm multidimensional 2} and \eqref{eq:eigenvalue assumption}.
Similarly
\begin{align*}
 & \sum_{\mi _1,\ldots, \mi_{d_1}=1 }^\m   \sum_{\li_1, \ldots, \li_{d_2-1}=1}^\l  \sum_{\li_{d_2} =\l+1}^ \infty     
\langle A ,  (\phik_{\mi _1} \otimes  \cdots \otimes  \phik_{\mi_{d_1}}) \otimes  ( \phik_{\li_1 } \otimes  \cdots  \otimes \phik_{\li_{d_2}}   ) \rangle ^2 
\\
\le &   \sum_{\mi _1,\ldots, \mi_{d_1}=1 }^\m    \sum_{\li_1,  \ldots, \li_{d_2-1}=1}^\l  \sum_{\li_{d_2} =\l+1}^ \infty  \l^{-2\alpha }  \li_{d_2}^{2\alpha }  \langle A ,  (\phik_{\mi _1} \otimes  \cdots \otimes  \phik_{\mi_{d_1}}) \otimes  ( \phik_{\li_1 } \otimes  \cdots  \otimes \phik_{\li_{d_2}}   ) \rangle ^2   
\\
\le &   \l^{-2\alpha } \sum_{\mi _1,\ldots, \mi_{d_1}=1 }^\m    \sum_{\li_1,  \ldots, \li_{d_2-1}=1}^\l  \sum_{\li_{d_2} =\l+1}^ \infty     
 \mi_1 ^{2\alpha } \cdots \mi_{d_1} ^{2\alpha } \li_1^{2\alpha} \cdots \li_{d_2}^{2\alpha }\langle A ,  (\phik_{\mi _1} \otimes  \cdots \otimes  \phik_{\mi_{d_1}}) \otimes  ( \phik_{\li_1 } \otimes  \cdots  \otimes \phik_{\li_{d_2}}   ) \rangle ^2  
 \\
    \leq  & \l^{-2\alpha }\|A\|_{  \h(\K^{\otimes d-1})     }^2 ,
\end{align*}
where the first inequality holds because $\li_{d_2}\ge \l+1 \ge \l $ and the last inequality follows from \eqref{eq:equivalent sobolev norm multidimensional 2} and \eqref{eq:eigenvalue assumption}.
Thus \eqref{eq:bias 1} follows immediately.
\\
\\
For 
 \eqref{eq:bias 2}, note that when $\m=\infty  $, $\mathcal M= \lt(\mathcal O^{d_1})$. In this case 
 $ \mathcal P_\mathcal M$ becomes the  identity  operator and 
 $$A  {\times_x \mathcal P_\mathcal M \times _y \mathcal P_\mathcal N} = A{ \times _y \mathcal P_\mathcal N}. $$ Therefore \eqref{eq:bias 2} follows from 
    \eqref{eq:bias 1} by taking $\m=\infty  $.  
\\
\\
For \eqref{eq:bias 3}, similar to \eqref{eq:bias 2}, we have that  
$$ \|A - A \times_x \mathcal{P}_\mathcal{M} \|^2_{\lt( \mathcal O ^{d} ) }    \le C   \m^{-2\alpha } \|A\|_{  \h(\K^{\otimes d}  )    }^2    .  $$
It follows that 
$$\| A { \times _y \mathcal P_\mathcal N} - A {\times_x \mathcal P_\mathcal M \times _y \mathcal P_\mathcal N} \|^2_{\lt(\mathcal O ^{d} ) }   \le \| A  - A \times_x \mathcal{P}_\mathcal{M} \|^2 _{\lt( \mathcal O ^{d} ) }    
\|\mathcal P_\mathcal N \|^2_{\op }   \le C d_1 \m^{-2\alpha } \|A\|_{   \h(\K^{\otimes d}  )   }^2      , $$
where last inequality follows from the fact that $ \|\mathcal P_\mathcal N \|_{\op } \le 1$.
\end{proof}

\subsection{Legendre polynomial basis}


 \label{subsection:polynomial}
  Legendre polynomials is a well-known  classical orthonormal  polynomial system in $\lt([-1,1])$.  We can define the Legendre polynomials in the following inductive way. Let $p_0=1 $ and suppose $\{ p_k\}_{k=1}^{n-1} $ are defined. Let $p_n:[-1,1] \to \mathbb R$ be a polynomial of degree $n$ such that 
  \\
  $\bullet$ $\| p_n\|_{\lt([-1,1])}=1$, and 
  \\
   $\bullet$ $ \int_{-1}^1 p_n(z) p_k(z) \mathrm{d}z  = 0$ for all $0\le k \le n-1 $.
   \\
As a quick example, we have that
$$p_0(z ) =1, \quad   p_1( z ) = \sqrt{\frac{3}{2}} z, \quad \text{and} \quad 
p_2( z ) = \sqrt{\frac{5}{3}}\frac{ 3 z ^2-1 }{2}.
 $$
Let $q_k( z ) = \sqrt {2 }p_k(2z -1)  $. Then $\{q_k\}_{k=0}^\infty$ are the orthonormal polynomial system in $\lt([0,1])$.
 In this subsection, we show that 
     \Cref{assume: baised in projected space in operator norm} holds when   $\{ \phi_k  \}_{k=1}^\infty  =\{ q_k \}_{k=0}^\infty $.  
     More precisely, let 
     $$ \mathbb S_n (z) =\Span \{q_k (z) \}_{k=0}^n $$
     and $\mathcal P_{ \mathbb S_n (z) }$ denote the projection operator from $\lt([0,1])  $ to $\mathbb S_n(z)  $.  Then  $ \mathbb S_n (z) $ is the subspace of polynomials of degree at most $ n$. In addition, for any $ f\in \lt([0,1])$, $ \mathcal P_{ \mathbb S_n(z) }(f)$ is the best $n$-degree polynomial approximation of $f$ in the sense that 
      \begin{align} \label{eq:polynomial in 1D}  \| \mathcal P_{\mathbb S_n(z)} (f) -f  \|_{\lt([0,1])}  = \min_{g \in \mathbb S_n} \| g -f  \|_{\lt([0,1])} . 
      \end{align}
     
We begin with a well-known polynomial approximation result.
For $\alpha\in \mathbb Z^+$, denote $ C^{\alpha}([0,1])$ to be the class of functions that are $\alpha$ times  continuously differentiable.  
\begin{theorem} \label{theorem:polynomial interpolation}
Suppose $f  \in C^{\alpha}([0,1]) $. Then for any $n\in \mathbb Z^+$, there exists a   polynomial $g_{2n}(f ) $ of degree $2n$ such that
$$\| f  -g_{2n} (f )  \|_{\lt([0,1])} ^2  \le C n^{-2 \alpha }\| f^{(\alpha)}\|_{\lt([0,1])} ^2,  $$
where $C$ is an absolute constant.

\end{theorem}
\begin{proof}
This is Theorem 1.2 of  \cite{xu2018approximation}.
\end{proof}
\
\\
Therefore by \eqref{eq:polynomial in 1D} and \Cref{theorem:polynomial interpolation}, 
 \begin{align} \label{eq:polynomial in 1D 2} \| \mathcal P_{\mathbb S_n(z)} (f) -f  \|^2_{\lt([0,1])} \le  \| f  -g_{\lfloor n/2\rfloor }( f )  \|_{\lt([0,1])} ^2  \le C 'n^{-2 \alpha }\| f ^{(\alpha)}\|_{\lt([0,1])} ^2 .
  \end{align} 
Let $z=(z_1,\ldots, z_d)\in \mathbb R^{d}$ and let  $\mathbb S_n(z_j ) $ denote the linear space spanned by polynomials of $z_j$ of degree  at most $n$.

\begin{corollary}
\label{corollary:polynomial projection approximation}
Suppose $ B(z_1,\ldots, z_d)  \in C^{\alpha}([0,1]^d)$. Then for any $1\le k \le d $,
$$ \| B-B \times_{ k }\mathcal P_{\mathbb S_n (z_{k })}\|_{\lt([0,1]^d)} \le C  n^{-\alpha } \| B \|_{W^\alpha_2([0,1]^d)}. $$

\end{corollary}
\begin{proof}
It suffices to consider $k=1$. For any fixed $(z_2, \ldots, z_d) $,
$B(\cdot,z_2, \ldots, z_d) \in C^\alpha([0,1])$. Therefore by \eqref{eq:polynomial in 1D 2},
$$ \int \bigg\{ B(z_1, z_2,\ldots, z_d) -   B \times_1\mathcal P_{\mathbb S_n(z_1)}(z_1, z_2,\ldots, z_d)  \bigg\}^2 \mathrm{d}z_1  \le Cn^{-2 \alpha} \int\bigg\{  \frac{\partial^\alpha }{\partial^\alpha  z_1 }   B(z_1, z_2,\ldots, z_d) \bigg\}^2   \mathrm{d}z_1 .$$
Therefore  
\begin{align*}
 &\| B- B \times_{1}\mathcal P_{\mathbb S_n (z_{1})}\|_{\lt([0,1]^d)}  
 \\ =&\int \cdots \int \bigg\{ B(z_1, z_2,\ldots, z_d) -
 B \times_{1}\mathcal P_{\mathbb S_n (z_{1})} (z_1, z_2,\ldots, z_d)  \bigg\}^2 \mathrm{d}z_1   \mathrm{d}z_2 \cdots \mathrm{d}z_d 
\\
 \le  &\int\cdots \int  Cn^{-2\alpha} \int\bigg\{  \frac{\partial^\alpha }{\partial^\alpha  z_1 }   B(z_1, z_2,\ldots, z_d) \bigg\}^2   \mathrm{d}z_1  \mathrm{d}z_2 \cdots \mathrm{d}z_d .
\end{align*}
The desired result follows from the observation that  
$\int\cdots     \int\big \{  \frac{\partial^\alpha }{\partial^\alpha  z_1 }   B(z_1, z_2,\ldots, z_d) \big \}^2   \mathrm{d}z_1   \cdots \mathrm{d}z_d \le \| B \|^2_{W^\alpha_2([0,1]^d)} $.
\end{proof}

In what follows, we present a polynomial approximation theory in multidimensions.

\begin{lemma} \label{lemma:polynomial interpolation in R^D}For $ j \in [1,\ldots, d] $,
let $\mathbb S_{n_j}(z_j ) $ denote the linear space spanned by polynomials of $z_j$ of degree  $n_j$ and let  $\mathcal P_{\mathbb S_{n_j}(z_j )} $ be the corresponding  projection operator. Then for any $ B \in W^\alpha_2([0,1]^d)$, 
it holds that 
\begin{align} \label{eq:polynomial interpolation in R^D}\| B-B  \times_1 \mathcal P_{\mathbb S_{n_1} (z_1)} \cdots \times_{d}\mathcal P_{\mathbb S_{n_d} (z_d)}\|_{\lt([0,1]^d)} \le C  \sum_{j =1}^d n_j^{-\alpha } \| B \|_{W^\alpha_2([0,1]^d)} . 
\end{align} 
\end{lemma}
\begin{proof} Since $C^{\alpha}([0,1]^d)$ is dense in $W^\alpha_2([0,1]^d)$, it suffices to show \eqref{eq:polynomial interpolation in R^D} for all functions in $  \in C^{\alpha}([0,1]^d)$. To this end, 
we proceed by induction.  The base case 
\begin{align*} \| B-B \times_1 \mathcal P_{\mathbb S_{n_1} (z_1)}  \|_{\lt([0,1]^d)} \le C    n_1^{-\alpha } \| B  \|_{W^\alpha_2([0,1]^d)}  
\end{align*}
is a direct consequence of \Cref{corollary:polynomial projection approximation}. Suppose   by induction, the following inequality holds for $k$, 
\begin{align} \label{eq:polynomial porjection induction hypothesis}  \| B -B  \times_1 \mathcal P_{\mathbb S_{n_1} (z_1)} \cdots \times_{k}\mathcal P_{\mathbb S_{n_k} (z_{k})}\|_{\lt([0,1]^d)} \le C  \sum_{j =1}^k n_j^{-\alpha } \| B  \|_{W^\alpha_2([0,1]^d)} . 
\end{align} 
Then 
\begin{align}\nonumber  
&\| B -B  \times_1 \mathcal P_{\mathbb S_{n_1} (z_1)} \cdots \times_{k}\mathcal P_{\mathbb S_{n_k} (z_{k})}\times_{k+1}\mathcal P_{\mathbb S_{n_{k+1}} (z_{k+1} )} \|_{\lt([0,1]^d)} 
\\ \nonumber  
\le & \| B  \times_1 \mathcal P_{\mathbb S_{n_1} (z_1)} \cdots \times_{k}\mathcal P_{\mathbb S_{n_k} (z_{k})}\times_{k+1}\mathcal P_{\mathbb S_{n_{k+1}} (z_{k+1} )} -B  \times_1 \mathcal P_{\mathbb S_{n_1} (z_1)} \cdots \times_{k}\mathcal P_{\mathbb S_{n_k} (z_{k})} \|_{\lt([0,1]^d)} 
\\\nonumber   
+&\| B - B \times_1 \mathcal P_{\mathbb S_{n_1} (z_1)} \cdots \times_{k}\mathcal P_{\mathbb S_{n_k} (z_{k})}  \|_{\lt([0,1]^d)} .
\end{align} 
The desired result follows from \eqref{eq:polynomial porjection induction hypothesis} and the observation  that 
\begin{align}\nonumber  
&  \|B \times_1 \mathcal P_{\mathbb S_{n_1} (z_1)} \cdots \times_{k}\mathcal P_{\mathbb S_{n_k} (z_{k})}\times_{k+1}\mathcal P_{\mathbb S_{n_{k+1}} (z_{k+1} )} -B  \times_1 \mathcal P_{\mathbb S_{n_1} (z_1)} \cdots \times_{k}\mathcal P_{\mathbb S_{n_k} (z_{k})} \|_{\lt([0,1]^d)} 
  \\ \nonumber  
  \le & \| B   \times_{k+1}\mathcal P_{\mathbb S_{n_{k+1}} (z_{k+1} )} - B  \|_{\lt([0,1]^d)}  \|  \mathcal P_{\mathbb S_{n_1} (z_1)} \|_{\op } \cdots \| \mathcal P_{\mathbb S_{n_k} (z_{k})} \|_{\op}
  \\\nonumber  
  \le & Cn^{-\alpha}_{k+1}\| B  \|_{W^\alpha_2([0,1]^d)},
\end{align} 
where the last inequality follows from \Cref{corollary:polynomial projection approximation}  and that
$\|\mathcal P_{\mathbb S_{n_{j}} (z_{j} )}  \|_{\op}\le 1 $ for all $j$.
\end{proof}
\
\\
This   following lemma shows that    \Cref{assume: baised in projected space in operator norm}  is a consequence of \Cref{lemma:polynomial interpolation in R^D}.
\begin{lemma} \label{lemma:legendre bias} Let $1 \le d_1 \le  d$, 
$  
\mathcal M = \mathbb S_{\m}(z_1) \otimes \cdots \otimes  \mathbb S_{\m}(z_{d_1})$,    and  
$\mathcal N = \mathbb S_{\l}(z_{d_1+1}) \otimes \cdots \otimes  \mathbb S_{\l}(z_{d}). $    
Suppose 
$\|A\|_{W^{\alpha}_2([0,1]^{d} ) }< \infty$.
Then for $   1  \le \m \le \infty $ and $ 1  \le \l \le \infty$,
\begin{align}\label{eq:legendre bias 1}& \|A-A {\times_x \mathcal P_\mathcal M \times _y \mathcal P_\mathcal N} \|^2_{ \lt  ([0,1]^{d} )  }    =\bigO (\m^{-2\alpha}+  \l^{-2\alpha }  )
\\
\label{eq:legendre bias 2} 
&\|A-A { \times _y \mathcal P_\mathcal N} \|^2_{\lt  ([0,1]^{d} ) }     =  \bigO(  \l^{-2\alpha }   )
\quad \text{and} 
\\\label{eq:legendre bias 3}
&\| A{ \times _y \mathcal P_\mathcal N} -A {\times_x \mathcal P_\mathcal M \times _y \mathcal P_\mathcal N}\|^2_{ \lt ([0,1]^{d} )  }     = \bigO( \m^{-2\alpha }  ).
\end{align}

\end{lemma}
 
\begin{proof}
For \eqref{eq:legendre bias 1}, by \Cref{lemma:polynomial interpolation in R^D}, 
\begin{align*} \|A-A {\times_x \mathcal P_\mathcal M \times _y \mathcal P_\mathcal N} \| _{\lt([0,1]^{d} ) }  \le & C  \big(  d_1(\m-\alpha-1)^{- \alpha }+d_2(\l-\alpha-1)^{- \alpha } \big)  \|A\|_{W^{\alpha}_2([0,1]^{d } ) }  
\\
=& \bigO (\m^{-\alpha}+  \l^{-\alpha }   ) ,
\end{align*}
where the equality follows from the fact that $\alpha$, $d_1$ and $d$  are constants. 
\\
\\
For 
 \eqref{eq:legendre bias 2}, note that when $\m=\infty$, $\mathcal M= \lt([0,1]^{d_1})$. In this case 
 $ \mathcal P_\mathcal M$ becomes the  identity  operator and 
 $$A  {\times_x \mathcal P_\mathcal M \times _y \mathcal P_\mathcal N} = A { \times _y \mathcal P_\mathcal N}. $$ Therefore \eqref{eq:legendre bias 2} follows from 
    \eqref{eq:legendre bias 1} by taking $\m=\infty$.  
\\
\\
For \eqref{eq:legendre bias 3}, similar to \eqref{eq:legendre bias 2}, we have that  
$$ \|A-A \times_x \mathcal P_\mathcal M \|^2_{\lt([0,1]^{d} ) }    =\bigO( \m^{-2\alpha } ) .  $$
It follows that 
$$\| A{ \times _y \mathcal P_\mathcal N} -A{\times_x \mathcal P_\mathcal M \times _y \mathcal P_\mathcal N} \|^2_{\lt([0,1]^{d} ) }   
\\
\le \| A  -A\times_x\mathcal P_\mathcal M\|^2_{\lt([0,1]^{d} ) }    
\|\mathcal P_\mathcal N \|^2_{\op } =  \bigO( \m^{-2\alpha } )    , $$
where last inequality follows from the fact that $ \|\mathcal P_\mathcal N \|_{\op } \le 1$. 
\end{proof}

\section{Additional numerical studies}
\label{sec:numerical appendix}
\subsection{Neural network density estimator}

For neural network estimators, we use two popular density estimation architectures: Masked Autoregressive Flow (MAF) (\cite{papamakarios2017masked}) and Neural Autoregressive Flows (NAF) (\cite{huang2018neural}) for comparisons. Both neural networks are trained using the  Adam optimizer (\cite{kingma2014adam}). For MAF, we use 5 transforms and each transform is a 3-hidden layer neural network with width 128. For NAF, we choose 3 transforms and each transform is a 3-hidden layer neural network with width 128.

\subsection{Kernel methods}
In our simulated experiments and real data examples, we choose  Gaussian kernel for the kernel density estimators and Nadaraya–Watson kernel regression (NWKR) estimators.  The bandwidths in all the numerical examples are chosen using cross-validations. We refer interested readers to \cite{wasserman2006all} for an introduction to nonparametric statistics.